\documentclass[11pt]{article}

\usepackage{fullpage}
\usepackage[square, numbers]{natbib}

\usepackage{epsf}
\usepackage{graphics}
\usepackage[normalem]{ulem}
\usepackage{subcaption}
\usepackage{psfrag}
\usepackage{comment}
\usepackage{color}
\usepackage{pgfplots}
\usepackage{pgfplotstable}
\usepackage{multirow}

\usepackage{mathtools}
\usepackage{amsfonts}
\usepackage{amsthm}
\usepackage{amsmath}
\usepackage{amssymb}
\usepackage{scalerel}
\usepackage{nccmath}

\usepackage{algorithm}
\usepackage[noend]{algpseudocode}

\newcommand{\PlainC}{\ensuremath{\mathcal{C}}}

\newcommand{\trace}{\ensuremath{\operatorname{tr}}}
\newcommand{\real}{\ensuremath{\mathbb{R}}}

\newcommand{\blue}[1]{\textcolor{blue}{#1}}
\newcommand{\apcomment}[1]{{\bf{{\blue{{AP --- #1}}}}}}

\newcommand{\defn}{:\,=}

\newtheorem{theorem}{Theorem}
\newtheorem{lemma}{Lemma}

\newtheorem{corollary}{Corollary}
\newtheorem{proposition}{Proposition}
\newtheorem{remark}{Remark}
\newtheorem{definition}{Definition}
\newcommand{\1}{\ensuremath{{\sf (i)}}}
\newcommand{\2}{\ensuremath{{\sf (ii)}}}

\makeatletter
\long\def\@makecaption#1#2{
        \vskip 0.8ex
        \setbox\@tempboxa\hbox{\small {\bf #1:} #2}
        \parindent 1.5em  
        \dimen0=\hsize
        \advance\dimen0 by -3em
        \ifdim \wd\@tempboxa >\dimen0
                \hbox to \hsize{
                        \parindent 0em
                        \hfil
                        \parbox{\dimen0}{\def\baselinestretch{0.96}\small
                                {\bf #1.} #2
                                }
                        \hfil}
        \else \hbox to \hsize{\hfil \box\@tempboxa \hfil}
        \fi
        }
\makeatother

\DeclareMathOperator{\Exs}{\ensuremath{\mathbb{E}}}
\newcommand{\EE}{\ensuremath{\mathbb{E}}} 

\newcommand{\x}[1][]{\ensuremath{x_{#1}}} 
\newcommand{\xstar}{x^{*}} 
\newcommand{\f}{\ensuremath{f}} 
\newcommand{\g}{\ensuremath{g}} 
\newcommand{\domg}{\ensuremath{\mathcal{X}}} 
\newcommand{\F}{\ensuremath{F}} 
\newcommand{\bracket}[1]{\ensuremath{\left[ #1 \right]}}
\newcommand{\smoothness}{\ensuremath{\phi}} 
\newcommand{\curvature}{\ensuremath{\theta}} 
\newcommand{\enorm}[1]{\ensuremath{\|#1 \|_2}}
\newcommand{\grad}{\ensuremath{\nabla}}

\newcommand{\fr}{\ensuremath{f_{r}}} 

\newcommand{\dims}{\ensuremath{d}} 
\newcommand{\Lipcon}{\ensuremath{\lambda}} 
\newcommand{\gradbound}{\ensuremath{G}} 
\newcommand{\Plconst}{\ensuremath{\mu}} 
\newcommand{\sigmafield}{\ensuremath{\mathcal{F}}}
\newcommand{\CondExs}[1]{\ensuremath{\mathbb{E}^{#1}}}
\newcommand{\accuracy}{\ensuremath{\epsilon}}
\newcommand{\y}[1][]{\ensuremath{y}}
\newcommand{\unifdir}{\ensuremath{u}}
\newcommand{\Unif}{\ensuremath{\operatorname{Unif}}}

\newcommand{\unifdirball}{\ensuremath{v}} 

\newcommand{\Ball}{\ensuremath{\mathbb{B}}}
\newcommand{\Shell}{\ensuremath{\mathbb{S}}}

\newcommand{\estgrad}[1]{\ensuremath{g^{#1}}}
\newcommand{\lqrdims}{\ensuremath{D}}
\newcommand{\lqr}{\ensuremath{{\sf lqr}}}
\newcommand{\lqrPl}{\ensuremath{\mu_{{\sf lqr}}}}

\newcommand{\braces}[1]{\ensuremath{\left\lbrace #1 \right\rbrace}}

\newcommand{\order}[1]{\ensuremath{\mathcal{O} \left( #1 \right)}}
\newcommand{\ordertil}[1]{\ensuremath{\widetilde{\mathcal{O}} \left( #1 \right) }}
\newcommand{\Prob}{\ensuremath{\mathbb{P}}}

\newcommand{\totIter}{\ensuremath{T}}

\newcommand{\abss}[1]{\ensuremath{\left| \right|}}

\newcommand{\costmat}{M}
\newcommand{\abk}{G}

\newcommand{\inprod}[2]{\ensuremath{\langle #1 , \, #2 \rangle}}

\newcommand{\vecnorm}[1]{\left\lVert#1\right\rVert_2}

\newcommand{\shelldirection}{u}

\newcommand{\domx}{\mathcal{X}}
\newcommand{\factorvar}{\xi} 

\newcommand{\euclidnorm}[1]{\left\lVert#1\right\rVert_2}

\newcommand{\matsnorm}[2]{|\!|\!| #1 | \! | \!|_{{#2}}}
\newcommand{\opnorm}[1]{\matsnorm{#1}{2}}
\newcommand{\fronorm}[1]{\matsnorm{#1}{\mbox{\tiny{F}}}}

\newcommand{\smoothingradius}{r}

\newcommand{\gradest}{\mathrm{g}} 

\newcommand{\factorvardistribution}{\mathcal{D}}
\newcommand{\Dspace}{\mathcal{D}}

\newcommand{\step}{\eta}
\newcommand{\var}{{\sf var}}

\newcommand{\Kprime}{\ensuremath{K'}}

\newcommand{\X}{X}

\newcommand{\boundedset}[1]{\ensuremath{\mathcal{G}^{#1}}}

\newcommand{\diff}[1]{\ensuremath{\Delta_{#1}}}

\newcommand{\A}{A}
\newcommand{\B}{B}
\newcommand{\Q}{Q}
\newcommand{\R}{R}

\newcommand{\ct}{c}

\newcommand{\statedim}{m} 
\newcommand{\controldim}{k} 
\newcommand{\C}[1]{\PlainC(#1)} 
\newcommand{\Cfunc}{\PlainC} 
\newcommand{\Cinit}{\PlainC_{{\sf init}, \gamma}} 
\newcommand{\Cdyn}{\PlainC_{{\sf dyn}, \gamma}} 
\newcommand{\Cstate}[2]{\mathcal{C}(#1, #2)} 
\newcommand{\K}[1][]{K_{#1}} 
\newcommand{\Ktwo}{K'} 
\newcommand{\Kstar}{K^*} 
\newcommand{\Kzero}{K_0} 
\newcommand{\state}[1][]{s_{#1}} 
\newcommand{\control}[1][]{a_{#1}} 
\newcommand{\error}[1][]{z_{#1}}
\newcommand{\initialstate}{s_0} 
\newcommand{\statedistributionlqr}{\mathcal{D}_0} 
\newcommand{\noisedistributionlqr}{\mathcal{D}_{{\sf add}}} 

\newcommand{\radiusone}[1]{\beta_{\scaleto{#1}{4pt}}} 
\newcommand{\radiustwo}[1]{\zeta_{\scaleto{#1}{4pt}}} 

\makeatletter
\newtheorem*{rep@theorem}{\rep@title}
\newcommand{\newreptheorem}[2]{%
\newenvironment{rep#1}[1]{%
 \def\rep@title{#2 \ref{##1}}%
 \begin{rep@theorem}}%
 {\end{rep@theorem}}}
\makeatother
\newreptheorem{lemma}{Lemma}

\newcommand{\smoothnessK}[1]{\phi_{\scaleto{#1}{4pt}}} 
\newcommand{\localK}[1]{\rho_{\scaleto{#1}{4pt}}} 

\newcommand{\lipconK}[1]{\lambda_{\scaleto{#1}{4pt}}} 
\newcommand{\smoothK}{\phi}
\newcommand{\globalSmooth}[1]{\ensuremath{\phi_{#1}}}

\newcommand{\identity}{I}
\newcommand{\noisebound}{C_{\statedim}}

\newcommand{\Kpolyzero}{c_{\K[0]}}
\newcommand{\Kpolyone}{c_{\K[1]}}
\newcommand{\Kpolytwo}{c_{\K[2]}}
\newcommand{\Kpolythree}{c_{\K[3]}}
\newcommand{\Kpolyfour}{c_{\K[4]}}
\newcommand{\Kpolyfive}{c_{\K[5]}}
\newcommand{\Kpolysix}{c_{\K[6]}}
\newcommand{\Kpolyseven}{c_{\K[7]}}
\newcommand{\Kpolyeight}{c_{\K[8]}}
\newcommand{\Kpolynine}{c_{\K[9]}}

\newcommand{\singularminQ}{\sigma_{min}(\Q)}
\newcommand{\singularminR}{\sigma_{min}(\R)}
\newcommand{\singularmininitial}{\sigma_{min}(\Sigma_0)}
\newcommand{\sigmaK}{\Sigma_{\K}}
\newcommand{\sigmaKtwo}{\Sigma_{\Ktwo}}

\newcommand{\EK}{E_{\K}}
\newcommand{\EKtwo}{E_{\Ktwo}}
\newcommand{\PK}{P_{\K}}
\newcommand{\PKtwo}{P_{\Ktwo}}
\newcommand{\TK}{\mathcal{T}_{\K}}
\newcommand{\TKtwo}{\mathcal{T}_{\Ktwo}}
\newcommand{\FK}{\mathcal{F}_{\K}}
\newcommand{\FKtwo}{\mathcal{F}_{\Ktwo}}

\newcommand{\discount}{\gamma} 
\newcommand{\z}{z} 

\newcommand{\Espace}{\mathcal{E}}

\begin{document}

\begin{center}
  {\LARGE{{\bf{Derivative-Free Methods for Policy Optimization:
          \\ Guarantees for Linear Quadratic Systems}}}}

\vspace*{.2in}

{\large{
\begin{tabular}{ccc}
Dhruv Malik$^\dagger$ & Ashwin Pananjady$^\dagger$ & Kush Bhatia$^\dagger$ \\
Koulik Khamaru$^\ddagger$ & Peter L. Bartlett$^{\dagger, \ddagger}$ & Martin J. Wainwright$^{\dagger, \ddagger, \ast}$
\end{tabular}
}}
\vspace*{.2in}

\begin{tabular}{c}
Department of Electrical Engineering and Computer Sciences, UC
Berkeley$^\dagger$ \\ Department of Statistics, UC Berkeley$^\ddagger$
\\ Voleon Group$^\ast$
\end{tabular}

\vspace*{.2in}

\today

\end{center}
\vspace*{.2in}
\begin{abstract}
 We study derivative-free methods for policy optimization over the
 class of linear policies. We focus on characterizing the convergence
 rate of these methods when applied to linear-quadratic systems, and
 study various settings of driving noise and reward feedback. We show
 that these methods provably converge to within any pre-specified
 tolerance of the optimal policy with a number of zero-order
 evaluations that is an explicit polynomial of the error tolerance,
 dimension, and curvature properties of the problem. Our analysis
 reveals some interesting differences between the settings of additive
 driving noise and random initialization, as well as the settings of
 one-point and two-point reward feedback. Our theory is corroborated
 by extensive simulations of derivative-free methods on these
 systems. Along the way, we derive convergence rates for stochastic
 zero-order optimization algorithms when applied to a certain class of
 non-convex problems.
\end{abstract}

\section{Introduction}

Recent years have witnessed a number of successes in applying modern
reinforcement learning (RL) methods to many fields, including
robotics~\cite{tobin17, levine15} and competitive
gaming~\cite{silver16, mnih15}. Impressively, most of these successes
have been achieved by using general-purpose RL methods that are
applicable to a host of problems. Prevalent general-purpose RL
approaches can be broadly categorized into: (a) \emph{model-based
  approaches}~\cite{deisenroth2012,gu2016,lillicrap2015}, in which an
agent attempts to learn a model for the dynamics by observing the
evolution of its state sequence; and (b) \emph{model-free approaches},
including DQN~\cite{mnih15}, and TRPO~\cite{schulman15}, in which the
agent attempts to learn an optimal policy directly, by observing
rewards from the environment. While model-free approaches typically
require more samples to learn a policy of equivalent accuracy, they
are naturally more robust to model mis-specification.

A literature that is closely related to model-free RL is that of
\emph{zero-order or derivative-free} methods for stochastic
optimization; see the book by~\cite{spall03} for an
overview. Here, the goal is to optimize an unknown function from noisy
observations of its values at judiciously chosen points. While most
analytical results in this space apply to convex optimization, many of
the procedures themselves rely on moving along randomized
approximations to the directional derivatives of the function being
optimized, and are thus applicable even to non-convex problems. In
the particular context of RL, variants of derivative-free methods,
including TRPO~\cite{schulman15}, PSNG~\cite{rajeswaran17} and
evolutionary strategies~\cite{salimans2017}, have been used to solve
highly non-convex optimization problems and have been shown to achieve
state-of-the-art performance on various RL tasks.

While many RL algorithms are easy to describe and run in practice,
certain theoretical aspects of their behavior remain mysterious, even
when they are applied in relatively simple settings. One such setting
is the most canonical problem in continuous control, that of
controlling a linear dynamical system with quadratic costs, a problem
known as the linear quadratic regulator (LQR).  A recent line of
work~\cite{abbasi2011, abbasi2018, abeille2017, cohen18, dean17,
  dean18, faradonbeh17, kakade18, tu18, tu182} has sought to delineate
the properties and limitations of various RL algorithms in application
to LQR problems. An appealing property of LQR systems from an
analytical point of view is that the optimal policy is guaranteed to
be linear in the states~\cite{kalman60,Whittle96}.  Thus, when the
system dynamics are known, as in classical control, the optimal policy
can be obtained by solving the discrete-time algebraic Ricatti
equation.

In contrast, methods in reinforcement learning target the case of
unknown dynamics, and seek to learn an optimal policy on the basis of
observations.  A basic form of model-free RL for linear quadratic
systems involves applying derivative-free methods in the space of
linear policies.  It can be used even when the only observations
possible are the costs from a set of rollouts, each referred to as a
sample\footnote{Such an offline setting with multiple, restarted
  rollouts should be contrasted with an online setting, in which the
  agent interacts continuously with the environment, and no hard
  resets are allowed. In contrast to the offline setting, the goal in
  the online setting is to control the system for all time steps while
  simultaneously learning better policies, and performance is usually
  measured in terms of regret.}, and when our goal is to obtain a
policy whose cost is at most $\epsilon$-suboptimal.  The sample
complexity of a given method refers to the number of samples, as a
function of the problem parameters and tolerance, required to meet a
given tolerance $\epsilon$.  With this context, we are led to the
following concrete question: \emph{What is the sample complexity of
  derivative-free methods for the linear quadratic regulator?}
This question underlies the analysis in this paper.  In particular, we
study a standard derivative-free algorithm in an offline setting and
derive explicit bounds on its sample complexity, carefully controlling
the dependence on not only the tolerance $\epsilon$, but also the
dimension and conditioning of the underlying problem.

Our analysis treats two distinct forms of randomness in the underlying
linear system.  In the first setting---more commonly assumed in
practice---the linear updates are driven by an additive noise
term~\cite{dean17}, whereas in the second setting, the initial state
is chosen randomly but the linear dynamics remain
deterministic~\cite{kakade18}.  We refer to these two settings,
respectively, as the \emph{additive noise setting}, and the
\emph{randomly initialized setting.}
We are now in a position to discuss related
work on the problem, and to state our contributions.


\paragraph{Related work:}

Quantitative gaps between model-based and model-free reinforcement
learning have been studied extensively in the setting of finite
state-action spaces~\cite{agrawal2017, dann2017, azar2017}, and
several interesting questions here still remain open.

For continuous state-action spaces and in the specific context of the
linear quadratic systems, classical system identification has been
model-based, with a particular focus on asymptotic results (e.g., see
the book by~\cite{ljung1987} as well as references
therein). Non-asymptotic guarantees for model-based control of linear
quadratic systems were first obtained by~\cite{fietcher97},
who studied the offline problem under additive noise and obtained
non-asymptotic rates for parameter identification using nominal
control procedures. In more recent work, Dean et al.~\cite{dean17}
proposed a robust alternative to nominal control, showing an improved
sample complexity as well as better-behaved policies. The online
setting for model-based control of linear quadratic systems has also
seen extensive study, with multiple algorithms known to achieve
sub-linear regret~\cite{dean18, abbasi2011, abeille2017, ibrahimi12, cohen19}.

In this paper, we study model-free control of these systems, a problem
that has seen some recent work in both the offline~\cite{kakade18} and
online~\cite{abbasi2018} settings.  Most directly relevant to our work
is the paper of Fazel et al.~\cite{kakade18}, who studied the offline
setting for the randomly initialized variant of the LQR, and showed
that a population version of gradient descent (and natural gradient descent), when run on the
non-convex LQR cost objective, converges to the global optimum. In
order to turn this into a derivative-free algorithm, they constructed
near-exact gradient estimates from reward samples and showed that the
sample complexity of such a procedure is bounded polynomially in the
parameters of the problem; however, the dependence on various
parameters is not made explicit in their analysis. We remark that Fazel et al.
also show polynomially bounded sample complexity for a zero order algorithm which builds near exact estimates of the \emph{natural} gradient,
although this requires access to a stronger oracle than the one assumed in this paper.

Also of particular relevance to our paper is the extensive literature
on zero-order optimization. Flaxman et al.~\cite{flax04} showed that
these methods can be analyzed for convex optimization by making an
explicit connection to function smoothing, and Agarwal et al.~\cite{agarwal10} improved some of these convergence rates. Results
are also available for strongly convex~\cite{jamrec12},
smooth~\cite{ghalan13} and convex~\cite{nesterov11, duchi15, wang17}
functions, with Shamir characterizing the
fundamental limits of many problems in this space~\cite{shamir12, shamir17}. Broadly speaking,
all of the methods in this literature can be seen as variants of
\emph{stochastic search}: they proceed by constructing estimates of
directional derivatives of the function from randomly chosen zero
order evaluations. In the regime where the function evaluations are
stochastic, different convergence rates are obtained based on whether
such a procedure uses a \emph{one-point estimate} that is obtained
from a single function evaluation~\cite{flax04}, or a \emph{$k$-point
  estimate}~\cite{agarwal10} for some $k \geq 2$. There has also been
some recent work on zero-order optimization of non-convex functions
satisfying certain smoothness properties that are motivated by
statistical estimation~\cite{wang18}.


\paragraph{Our contributions}

In this paper, we study both randomly initialized and additive-noise
linear quadratic systems in the offline setting through the lens of
derivative-free optimization.  We begin with a general result that
characterizes the convergence behavior of a canonical derivative-free
algorithm when applied to a general class of functions satisfying certain curvature conditions.
In particular, our main contribution is to establish
upper bounds on the sample complexity as a function of the dimension,
error tolerance, and curvature parameters of the problem instance. We then specialize
this result to a variety of LQR models. In
contrast to prior work, the rates that we provide are explicit, and
the algorithms that we analyze are standard and practical one-point
and two-point variants of the random search heuristic. Our results
reveal interesting dichotomies between the settings of one-point and
two-point feedback, as well as the models involving random
initialization and additive noise.  Our main contribution is stated in
the following informal theorem (to be stated more precisely in the sequel):

\paragraph{Main Theorem (informal).} \emph{With high probability,
  one can obtain an $\epsilon$-approximate solution to any linear
  quadratic system from observing the noisy costs of
  $\widetilde{\mathcal{O}}(1 / \epsilon^2)$ trajectories from the
  system, which can be further reduced to $\widetilde{\mathcal{O}}(1 /
  \epsilon)$ trajectories when pairs of costs are observed for each
  trajectory. }
\newline

In our theoretical statements, the multiplicative pre-factors are
explicit lower-order polynomials of the dimension of the state space,
and curvature properties of the cost function.  From a technical
standpoint, we build upon some known properties of the LQR cost
function established in past work on randomly initialized
systems~\cite{kakade18}, and establish de novo some analogous
properties for the additive noise setting. We also isolate and sharpen
some key properties that are essential to establishing sharp rates of
zero-order optimization; as an example, for the setting with
random-initialization and one-point reward feedback studied by Fazel et al.~\cite{kakade18}, establishing these properties allows us to
analyze a natural algorithm that improves\footnote{While the rates
established by Fazel et al.~\cite{kakade18} are not explicit, their
analysis is conservative and yields a bound of order $1/\epsilon^4$ up to logarithmic factors.
To be clear, the properties that we establish also enable us to provide a sharper analysis
of their algorithm; see Appendix~\ref{app:fazel} to follow.}
the dependence of the bound on the error tolerance $\epsilon$ from
at least $\order{1/\epsilon^4}$ to $\order{1/\epsilon^2}$.
Crucially, our analysis is complicated by the fact that we must ensure
that the iterates are confined to the region in which the linear
system is stable, and such stability considerations introduce
additional restrictions on the parameters used in our optimization
procedure.


\section{Background and problem set-up}
\label{sec:setup}

In this section, we discuss the background related to zero-order
optimization and the setup for the linear quadratic control problem.


\subsection{Optimization background}
\label{subsec:opt_back}

We first introduce some standard optimization related background and
assumptions, and make the zero-order setting precise.

\paragraph{Stochastic zero-order optimization:}

We consider optimization problems of the form
\begin{align}
\label{eqn:general_zero_order_prob}
\min_{\x \in \domx} \f(\x) & \defn \Exs_{\factorvar \sim
  \factorvardistribution}{\left[\F(\x, \factorvar)\right]},
\end{align}
where $\factorvar$ is a zero mean random variable\footnote{While the zero mean assumption on $\factorvar$ is not strictly necessary for generic optimization, the canonical (additive noise) LQR settings that we specialize our results to require noise to be zero mean. So we make this assumption at the outset for convenience.} that represents the noise in the problem, and the function $\f$ above can be non-convex in general with a possibly non-convex domain $\domx \subseteq \real^{\dims}$.

In particular, we consider stochastic zero-order optimization methods
with oracle access to noisy function evaluations.
We operate under two distinct oracle models. The first is the one-point setting,
in which the optimizer specifies a point $x \in \domx$, and an evaluation consists of an instantiation of the random variable $\F(\x, \factorvar)$. The second is the two-point extension of such a setting, in which the optimizer specifies a pair of points $(\x, y)$, then an instantiation of the random variable $\factorvar$ occurs, and the optimizer obtains the values $\F(\x, \factorvar)$ and $\F(y, \factorvar)$. Crucially, the function evaluations $\F(\x, \factorvar)$ and $\F(y, \factorvar)$ share the same noise, so the two-point oracle cannot be reduced to querying the one-point oracle twice (where sharing the same noise across multiple function evaluations cannot be guaranteed). Such two-point settings are known in the optimization literature to enjoy reduced variance of gradient estimates~\cite{agarwal10, duchi15, shamir17}.



\paragraph{Function properties:} Before defining the optimization problems considered in this paper by instantiating the pair of functions $(f, F)$, let us
precisely define some standard properties that make repeated appearances in the sequel.

\begin{definition}[Locally Lipschitz Gradients]
\label{def:lipschitz_gradient_lqr}
A continuously differentiable function $\g$ with domain $\domg$ is said to have
$(\smoothK, \beta)$ locally Lipschitz gradients at $x \in \domg$ if
\begin{align}
\euclidnorm{\nabla \g{(\y)} - \nabla \g{(\x)}} \leq \smoothK
\euclidnorm{\y - \x} \qquad \mbox{for all $y \in \domg$ with
  $\|x-y\|_2 \leq \beta$.}
  \end{align}
\end{definition}
We often say that $\g$ has locally Lipschitz gradients, by which we
mean for each $x \in \domg$ the function $\g$ has locally Lipschitz gradients,
albeit with constants $(\smoothK, \beta)$ that may depend on $x$.
This property guarantees that the function $\g$ has at most quadratic
growth locally around every point, but the shape of the quadratic and
the radius of the ball within which such an approximation holds may depend on the
point itself.

\begin{definition}[Locally Lipschitz Function]
\label{def:lipschitz_cost}
A continuously differentiable function $\g$ with domain $\domg$ is said to be $(\Lipcon,
\zeta)$ locally Lipschitz at $x \in \domg$ if
\begin{align}
  \label{EqnLocalLipschitz}
  \vert \g{(\y)} - \g{(\x)} \vert \leq \Lipcon \euclidnorm{\y - \x}
  \qquad \mbox{for all $\y \in \domg$ such that $\|x-y\|_2 \leq
    \zeta$.}
\end{align}
\end{definition}
As before, when we say that the function $\g$ is locally Lipschitz, we mean that
this condition holds for all $x \in \domg$, albeit with parameters
$(\Lipcon, \zeta)$ that may depend on $x$.  The local Lipschitz
property guarantees that the function $\g$ grows no faster than
linearly in a local
neighborhood around each point. 

\begin{definition}[PL Condition]
\label{def:pl_inequality}
A continuously differentiable function $\g$ with domain $\domg$ and a finite global minimum $\g^*$
is said to be $\Plconst$-PL if it satisfies the Polyak-\L ojasiewicz (PL) inequality with
constant $\Plconst > 0$, given by
\begin{align}
\label{EqnPLInequality}
\vecnorm{\grad \g{(\x)}}^2 & \geq \Plconst \; \big( \g{(\x)} - \g^*
\big) \qquad \mbox{for all $x \in \domg$.}
\end{align}
\end{definition}
The PL condition, first introduced by Polyak~\cite{polyak63} and Lojasiewicz~\cite{loj63}, is a relaxation of the notion of strong
convexity. It allows for a certain degree of non-convexity in the
function $\g$.  Note that inequality~\eqref{EqnPLInequality}
yields an upper bound on the gap to optimality that is proportional to
the squared norm of the gradient.  Thus, while the condition admits
non-convex functions, it requires that all first-order stationary
points also be global minimizers. Karimi et al.~\cite{schmidt16}
recently showed that many standard first-order convex optimization algorithms
retain their attractive convergence guarantees over this more general
class.


\subsection{Optimal control background}
\label{sec:LQR_background}

We now turn to some basic background on optimal control and
reinforcement learning.  An optimal control problem is specified by a
dynamics model and a real-valued cost function. The dynamics model
consists of a sequence of functions $\braces{h_t(\state[t],
  \control[t], z_t)}_{t \geq 0}$, which models how the state vector
$\state[t]$ transitions to the next state $\state[t+1]$ when a control
input $\control[t]$ is applied at a timestep $t$. The term $z_t$
captures the noise disturbance in the system. The cost function
$\ct_t(\state[t], \control[t])$ specifies the cost incurred by taking
an action $\control[t]$ in the state $\state[t]$. The goal of the
control problem is to find a sequence of control inputs
$\braces{\control[t]}_{t \geq 0}$, dependent on the history of states
$\mathcal{H}_t \defn (\state[0], \state[1], \ldots, \state[t-1])$, so as to solve the optimization
problem
\begin{align}
\min \mathbb{E} \left[\sum_{t \geq 0} \gamma^t \ct_t(\state[t],
  \control[t]) \right] \qquad \text{s.t. } \state[t+1] =
h_t(\state[t], \control[t], z_t),
\end{align}
where the expectation above is with respect to the noise in the
transition dynamics as well as any randomness in the selection of
control inputs, and $0 < \gamma \le 1$ represents a multiplicative discount factor.
A mapping from histories $\mathcal{H}_t$ to controls $\control[t]$
is called a \emph{policy}, and the
above minimization is effectively over the space of policies.

There is a distinction to be made here between the classical
fully-observed setting in stochastic control in which the dynamics
model $h_t$ is known---in this case, such a problem may be solved (at
least in principle) by the Bellman
recursion~\cite{bertsekas2005}, and the system
identification setting in which the dynamics are completely
unknown. We operate in the latter setting, and accommodate the further
assumption that even the cost function $c_t$ is unknown.

In this paper, we assume that the state space is $\statedim$-dimensional, and the control space is $\controldim$-dimensional, so that $\state[t] \in \real^\statedim$ and $\control[t] \in \real^\controldim$. The linear quadratic system specifies particular forms for the dynamics and costs,
respectively. In particular, the cost function obeys the quadratic form
\begin{align*}
c_t = \state[t]^{\top} \Q \state[t] + \control[t]^{\top} \R \control[t]
\end{align*}
for a pair of positive definite matrices $(\Q, \R)$ of the appropriate dimensions.
Additionally, the dynamics model is linear in both states and controls, and takes the form
\begin{align*}
\state[t+1] = \A \state[t] + \B \control[t] + \error[t],
\end{align*}
where $\A$ and $\B$ are transition matrices of the appropriate dimension, and the random variable $\error[t]$ models additive noise in the problem which is drawn i.i.d. for each $t$ from a distribution $\noisedistributionlqr$. We call this setting the \emph{noisy dynamics} model.

We also consider the \emph{randomly initialized} linear quadratic system without additive noise, in which the state transitions obey
\begin{align*}
\state[t+1] &= A \state[t] + B \control[t],
\end{align*}
and the randomness in the problem comes from choosing the initial state $\state[0]$ at random from a distribution $\statedistributionlqr$.

Throughout this paper, we assume\footnote{It is important to note that
  our assumption of identity covariance of the noise distributions can
  be made without loss of generality: for a problem with known, non-identity
  (but full-dimensional) covariance $\Sigma$, we may reparametrize the
  problem with the modifications
\begin{align*}
A' = \Sigma^{-1/2} A \Sigma^{1/2}, \quad B' = \Sigma^{-1/2} B, \text{
  and } s'_t = \Sigma^{-1/2} \state[t] \text{ for all } t \geq 0,
\end{align*}
in which case the new problem with states $s'_t$ and the pair of
transition matrices $(A', B')$ is driven by noise satisfying the
assumptions~\eqref{eq:propnoise}.}  that for both distributions
$\Dspace \in \{ \noisedistributionlqr, \statedistributionlqr\}$ and
for a random variable $v \sim \Dspace$, we have
\begin{align}
  \label{eq:propnoise}
\EE [v] = 0, \quad \EE [vv^\top] = I, \text{ and } \| v \|_2^2 \leq
\noisebound \; \; \text{ a.s.}
\end{align}
While we assume boundedness of the distribution for convenience, our
results extend straightforwardly to sub-Gaussian distributions by
appealing to high-probability bounds for quadratic forms of
sub-Gaussian random vectors~\cite{HanWri71,Wri73,HsuKakZha12} and
standard truncation arguments. The final iteration complexity also
changes by at most poly-logarithmic factors in the problem parameters;
for brevity, we operate under the assumptions~\eqref{eq:propnoise}
throughout the paper and omit standard calculations for sub-Gaussian
distributions.

By classical results in optimal control
theory~\cite{kalman60,Whittle96}, the optimal controller for the LQR
problem under both of these noise models takes the linear form
$\control[t] = -\Kstar \state[t]$, for some matrix $\Kstar \in
\mathbb{R}^{\controldim \times \statedim}$.  When the system matrices
are known, the controller matrix $\Kstar$ can be obtained by solving
the discrete-time algebraic Riccati equation~\cite{riccati1700}.

With the knowledge that the optimal policy is an invariant linear
transformation of the state, one can re-parametrize the LQR objective
in terms of the linear class of policies, and focus on optimization
procedures that only search over the class of linear policies.
Below, we define such a parametrization under the noise models introduced above,
and make explicit the connections to the stochastic optimization model~\eqref{eqn:general_zero_order_prob}.

\paragraph{Random initialization}

For each choice of the (random) initial state $\state[0]$, let $\Cinit(\K ; \initialstate)$ denote the cost of executing a linear policy $\K$ from initial state $\initialstate$, so that
\begin{align}
\Cinit(\K; \initialstate) \defn \sum_{t=0}^{\infty} \gamma^t \bigg( \state[t]^{\top} \Q \state[t] +
      \control[t]^{\top} \R \control[t] \bigg),
\end{align}
where we have the noiseless dynamics $\state[t+1] = A\state[t] + B
\control[t]$ and $\control[t] = -\K \state[t]$ for each $t \geq 0$,
and $0 < \gamma \le 1$.  While $\Cinit(\K; \initialstate)$ is a random
variable that denotes some notion of sample cost, our goal is to
minimize the population cost
\begin{align}
\Cinit(\K) \defn \mathbb{E}_{\initialstate \sim \statedistributionlqr}
    [\Cinit(\K; \initialstate)]
    \label{eqn:cost_fun_random_init}
\end{align}
over choices of the policy $\K$.

\paragraph{Noisy dynamics}

In this case, the noise in the problem is given by the sequence of
random variables $\mathcal{Z} = \{\error[t] \}_{t \geq 0}$, and for
every instantiation of $\mathcal{Z} \sim
\noisedistributionlqr^{\mathbb{N}} \defn (\noisedistributionlqr
\otimes \noisedistributionlqr \otimes \ldots)$, our sample cost is
given by the function
\begin{align*}
\Cdyn(\K; \mathcal{Z}) \defn \sum_{t = 0}^{\infty} \gamma^t \bigg(
\state[t]^{\top} \Q \state[t] + \control[t]^{\top} \R \control[t]
\bigg),
\end{align*}
where we have $\state[0] = 0$, random state evolution
$\state[t+1] = A \state[t] + B \control[t] + \error[t]$ and
action $\control[t] = -\K \state[t]$ for each $t \geq 0$,
and $0 < \gamma < 1$. In contrast to the
random initialization setting, the discount factor in this setting obeys
$\gamma < 1$, since this is required to keep the costs finite.

Once again, we are interested in optimizing the population cost
function
\begin{align}
\Cdyn(\K) \defn \mathbb{E}_{\mathcal{Z} \sim
  \noisedistributionlqr^{\mathbb{N}}} [\Cdyn(\K; \mathcal{Z})].
\label{eqn:cost_fun_noisy_dyn}
\end{align}
From here on, the word policy will always refer to a linear policy, and
since we work with this natural parametrization of the cost function,
our problem has effective dimension $\lqrdims = \statedim \cdot
\controldim$, given by the product of state and control dimensions.

A policy $\K$ is said to stabilize the system $(\A,\B)$ if we have
$\rho_\text{spec}(\A-\B\K) < 1$, where $\rho_\text{spec}(\cdot)$
denotes the spectral radius of a matrix.  We assume throughout
that the LQR system to be optimized is controllable, meaning that there
exists some policy $\K$ satisfying the condition
$\rho_\text{spec}(\A-\B\K) < 1$. Furthermore, we assume access to~\emph{some}
policy $\K[0]$ with finite cost; this is a mild assumption that is can be satisfied in a variety of ways; see the related
literature by Fazel et al.~\cite{kakade18} and Dean et al.~\cite{dean18}. We use such a policy $\K[0]$ as an initialization
for our algorithms.

\subsubsection{Some properties of the LQR cost function} \label{sec:lqrprop}

Let us turn to establishing properties of the pair of population cost
functions $\left(\Cinit(\K), \Cdyn(\K) \right)$ and their respective sample variants $\left(\Cinit(\K, \initialstate), \Cdyn(\K; \mathcal{Z}) \right)$, in order to place the problem within the context of
optimization.

First, it is important to note that both the population cost functions
$\left(\Cinit(\K), \Cdyn(\K) \right)$ are non-convex.
 In particular, for any unstable policy, the state sequence blows up
 and the costs becomes infinite, but as noted by Fazel et al.~\cite{kakade18}, the stabilizing region \mbox{$\{\K:
 \rho_\text{spec}(\A-\B\K) < 1 \}$} is non-convex, thereby rendering
 our optimization problems non-convex.

In spite of this non-convexity, the cost functions exhibit many
properties that make them amenable to fast stochastic optimization
methods. Variants of the following properties were first established
by Fazel et al.~\cite{kakade18} for the random initialization cost
function $\Cinit$. The following Lemma \ref{lem:lipschitz_cost_lqr}
and Lemma \ref{lem:lipschitz_gradient_lqr} require certain refinements
of their claims, which we prove in
Appendix~\ref{sec:randint_appendix}. Lemma~\ref{lem:pl_inequality}
follows directly from Lemma 3 in Fazel et al.~\cite{kakade18}. Lemma~\ref{lem:noisy-random} relates the
population cost of the noisy dynamics model to that of the random
initialization model in a pointwise sense.


\begin{lemma}[LQR Cost is locally Lipschitz]
\label{lem:lipschitz_cost_lqr}
Given any linear policy $\K$, there exist positive scalars
$(\lipconK{\K}, \widetilde{\lipconK{\K}}, \radiustwo{\K})$, depending
on the function value $\Cinit(\K)$, such that for all policies $\Ktwo$
satisfying ${\fronorm{\Ktwo - \K} \leq \radiustwo{\K}}$, and for all
initial states $s_0$, we have
\begin{subequations}
\begin{align}
\vert \Cinit(\Ktwo) - \Cinit(\K) \vert
&\leq \lipconK{\K} \fronorm{\Ktwo - \K}, \text{ and} \\
\vert \Cinit(\Ktwo; \initialstate) - \Cinit(\K; \initialstate) \vert
&\leq \widetilde{\lipconK{\K}} \fronorm{\Ktwo - \K}.
\end{align}
\end{subequations}
%
\end{lemma}


\begin{lemma}[LQR Cost has locally Lipschitz Gradients]
\label{lem:lipschitz_gradient_lqr}
Given any linear policy $\K$, there exist positive scalars $(\radiusone{\K},
\smoothnessK{\K})$, depending on the function value $\Cinit(\K)$, such
that for all policies $\Ktwo$ satisfying $\fronorm{\Ktwo - \K} \leq \radiusone{\K}$, we have
\begin{align}
\fronorm{\grad \Cinit(\Ktwo) - \grad \Cinit(\K)} \leq \smoothnessK{\K}
\fronorm{\Ktwo - \K}.
\end{align}
\end{lemma}

\begin{lemma}[LQR satisfies PL]
\label{lem:pl_inequality}
There exists a universal constant $\lqrPl > 0$ such that for all stable policies $\K$, we have
\begin{align*}
  \fronorm{\grad \Cinit(\K)}^2 \geq \lqrPl \big( \Cinit(\K) -
  \Cinit(\Kstar) \big),
\end{align*}
where $\Kstar$ is the global minimum of the cost function $\Cinit$.
\end{lemma}
For the sake of exposition, we have stated these properties without
specifying the various smoothness and PL constants. Appendix~\ref{sec:randint_appendix} collects explicit expressions for the
tuple $(\lipconK{\K}, \widetilde{\lipconK{\K}}, \smoothnessK{\K},
\radiusone{\K}, \radiustwo{\K}, \lqrPl)$ as functions of the
parameters of the LQR problem.

\begin{lemma}[Equivalence of population costs up to scaling]
\label{lem:noisy-random}
For all policies $\K$, we have
\begin{align*}
\Cdyn(\K) = \frac{\gamma}{1 - \gamma} \Cinit(\K).
\end{align*}
\end{lemma}

Lemma~\ref{lem:noisy-random} thus shows that, at least in a population
sense, both the noisy dynamics and random initialization models behave
identically when driven by noise with the same first two moments.
Hence, the properties posited by Lemmas~\ref{lem:lipschitz_cost_lqr},
~\ref{lem:lipschitz_gradient_lqr}, and~\ref{lem:pl_inequality} for the
\emph{population} cost function $\Cinit(\K)$ also carry over to the function
$\Cdyn(\K)$.  In particular, the cost function $\Cdyn(\K)$ is
also $\left(\frac{\gamma}{1 - \gamma} \smoothnessK{\K}, \radiusone{\K}
\right)$ locally smooth and $\left(\frac{\gamma}{1 - \gamma}
\lipconK{\K}, \radiustwo{\K} \right)$ locally Lipschitz, and also
globally $\frac{\gamma}{1 - \gamma}\lqrPl$-PL. We stress that although
the population costs are very similar, the observed costs in the
two cases are quite different.

\subsubsection{Stochastic zero-order oracle in LQR}

Let us now describe the form of observations that we make in the LQR system.  Recall
that we are operating in the derivative-free setting, where we have
access to only (noisy) function evaluations and not the problem parameters; in
particular, the tuple $(\A,\B,\Q,\R)$ that parametrizes the LQR
problem is unknown.

Our observations consist of the noisy function evaluations $\Cinit(\K;
\initialstate)$ or $\Cdyn(\K; \mathcal{Z})$. We consider both the
one-point and two-point settings in the former case. In the one-point
setting for the randomly initialized model, a \emph{query} of the
function at the point $\K$ obtains the noisy function value
$\Cinit(\K; \initialstate)$ for an initial state $\initialstate$ drawn
at random from the distribution $\statedistributionlqr$. In the
two-point setting, a query of the function at the points
$(\K,\Kprime)$ obtains the pair of noisy function values $\Cinit(\K;
\initialstate)$ and $\Cinit(\Kprime; \initialstate)$ for an initial
state $\initialstate$ drawn at random; this setting has an immediate
operational interpretation as running two policies with the same
random initialization. The one-point query model is defined
analogously for the noisy dynamics cost $\Cdyn$.


A few points regarding our query model merit discussion. First, note
that in the context of the control objective, each query produces a
noisy sample of the long term trajectory cost, and so our sample
complexity is measured in terms of the number of \emph{rollouts}, or
trajectories. Such an assumption is reasonable since the ``true"
sample complexity that also takes into account the length of the
trajectories is only larger by a small factor---the truncated, finite
cost converges exponentially quickly to the infinite sum for stable
policies.\footnote{To elaborate further on this point, note that the length of the rollout required to obtain a $\delta$-accurate cost evaluation for policy $K$ will depend on both $\delta$ as well as the eigen-structure of the matrix $A - BK$. However, assuming that this matrix has maximum eigenvalue $\rho < 1$ (which is a common assumption in the related literature~\cite{dean18,cohen19}), the dependence on $\delta$ is quite mild: we only require a rollout of length $\order{\log (1 / \delta)}$, with the constant pre-factor depending on $\rho$ (or equivalently, on $\mathcal{C}(K_0)$. Since we are interested in obtaining $\epsilon$-approximations to the optimal policy, it suffices to obtain $\mathsf{poly}(\epsilon)$-approximate cost evaluations per trajectory to avoid a blow-up of the bias in our estimates (see, e.g.,~\cite{kakade18}), and this only adds another factor $\log (1 / \epsilon)$ to our sample complexity when measured in terms of the number of iterations. To avoid tracking these additional factors, we work with the offline setting defined above.} The offline nature of the query model
also assumed access to restarts of the system, which can be obtained in a simulation environment. Second, we note that while the one-point query model was studied by Fazel et al.~\cite{kakade18} for the random initialization model---albeit with
sub-optimal guarantees---we also study a two-point query model, which is known to lead to
faster convergence rates in zero-order stochastic optimization~\cite{duchi15}.

Finally, note that our setting of the problem---in which we are only given access to (noisy) evaluations of the cost of the policy and not to the state sequence---intentionally precludes the use of procedures that rely on observations of the state sequence. This setting allows us to distill the difficulties of truly `model-free' control, since it prevents any possibility of constructing a dynamics model from our observations; the latter is, loosely speaking, the guiding principle of model-based control. This is not to suggest that practical applications of learning-based LQR control take this form, but rather to provide a concrete framework within which model-based and model-free algorithms can be separated, by endowing them with distinct information oracles. In doing so, we hope to lay the broader foundations for studying derivative-free methods in the context of model-free reinforcement learning.

\section{Main results}
\label{sec:results2}

We now turn to a statement of our main result, which characterizes the
convergence rate of a natural derivative-free algorithm for any
(population) function that satisfies certain PL and smoothness
properties. We thus obtain, as corollaries, rates of zero-order
optimization algorithms when applied to the functions $\Cinit$ and $\Cdyn$;
these corollaries are collected in Section~\ref{sec:cons-lqr}.


\subsection{Stochastic zero-order algorithm}

We analyze a standard zero-order algorithm for stochastic
optimization~\cite{agarwal10,shamir17} in application to the LQR
problem.  We begin by introducing some notation required to describe
this algorithm, operating in the general setting where we want to
optimize a function $\f: \domx \mapsto \real$ of the form $\f(x) =
\mathbb{E}_{\factorvar \sim\factorvardistribution} [F(x;
  \factorvar)]$.  Here we assume the inclusion $\domx \subseteq
\real^d$, and let $\factorvardistribution$ denote a generic source of
randomness in the zero-order function evaluation.

The zero-order algorithms that we study here use noisy function
evaluations in order to construct near-unbiased estimates of the
gradient.  Let us now describe how such an estimate is constructed in
the one-point and two-point settings.  Let $\Shell^{\dims - 1} = \{u
\in \real^d: \| u \|_2 = 1\}$ denote the $d$-dimensional unit
shell. Let $\Unif(\Shell^{\dims - 1})$ denote the uniform distribution
over the set $\Shell^{\dims - 1}$.

For a given scalar $\smoothingradius > 0$ and a random direction
$\shelldirection \sim \Unif(\Shell^{\dims - 1})$ chosen independently
of the random variable $\factorvar$, consider the one point gradient
estimate
\begin{subequations}
\begin{align}
\gradest_r^1(\x, \shelldirection, \factorvar) & \defn F(\x +
\smoothingradius \shelldirection,\factorvar) \;
\frac{\dims}{\smoothingradius} \unifdir, \label{eq:gradonepoint}
\end{align}
and its two-point analogue
\begin{align}
\gradest_r^2(\x, \shelldirection, \factorvar) & \defn \big[ F(\x +
  \smoothingradius \shelldirection,\factorvar) - F(\x -
  \smoothingradius \shelldirection, \factorvar) \big] \; \frac{\dims}
        {2\smoothingradius} \unifdir. \label{eq:gradtwopoint}
\end{align}
\end{subequations}
Here $\xi$ should be viewed as an instantiation of the underlying random variable;
in the two point setting, we compute a gradient estimate with the \emph{same instantiation}
of the noise used to evaluate $F$ at the points $x \pm r u$.

In both the one-point and two-point cases, the resulting ratios are almost unbiased
approximations of the secant ratio that defines the derivative at $x$,
and these approximations get better and better as the \emph{smoothing
  radius} $r$ gets smaller. On the other hand, small values of the
radius $r$ may result in estimates with large variance. Our algorithms
make use of such randomized approximations in a sequence of rounds by
choosing appropriate values of the radius $r$; the general form of
such an algorithm is stated below.

\begin{algorithm}
\caption{Stochastic Zero-Order Method}
\label{sgd_simple}
\begin{algorithmic}[1]
\State{Given iteration number $T \geq 1$, initial point $\x[0] \in
  \domx$, step size $\step > 0$ and smoothing radius
  $\smoothingradius > 0$} \For{$t \in \{ 0, 1, \ldots, T-1 \}$}
\State{ Sample $\factorvar_t \sim \factorvardistribution$ and
  $\shelldirection_t \sim \mbox{Unif}(\Shell^{\dims - 1})$}
  \State{
  $\gradest(x_t) \gets
  \begin{cases}
  \gradest_r^1(\x[t], \shelldirection_t, \factorvar_t) \; \text{ if
    operating in one-point setting} \\ \gradest_r^2(\x[t],
  \shelldirection_t, \factorvar_t) \; \text{ if operating in two-point
    setting.}
  \end{cases}
  $}
  \State {$\x[t+1] \gets \x[t] - \step \gradest (x_t)$} \EndFor
  \Return $\x[T]$
\end{algorithmic}
\end{algorithm}

\subsection{Convergence guarantees} \label{sec:conv-guar}

We now turn to analyzing Algorithm~\ref{sgd_simple} in the settings of interest.
In particular, our first (main) theorem is stated as a generic optimization
result for non-convex functions which are (locally) smooth and satisfy the PL inequality,
which we then specialize to various LQR settings.

As mentioned before, the difficulty of optimizing the LQR cost functions is
governed by multiple factors such as stability, non-convexity of the feasible set,
and non-convexity of the objective. Furthermore, the Lipschitz gradient and Lipschitz properties for this cost function only hold locally with the radius of locality depending on
the current iterate. Most crucially, the function is
infinite outside of the region of stability, and so large steps can
have disastrous consequences since we do not have access to a
projection oracle that brings us back into the region of
stability. It is thus essential to control the behavior of our stochastic,
high variance algorithm over the entire course of optimization.

Our strategy to overcome these challenges is to perform a careful
martingale analysis, showing that the iterates remain bounded
throughout the course of the algorithm; the rate depends, among other
things, on the variance of the gradient estimates obtained over the
course of the algorithm. By showing that the algorithm remains within
the region of finite cost, we can also obtain good bounds on the local
Lipschitz constants and gradient smoothness parameters, so that our
step-size can be set accordingly.

Let us now introduce some notation in order to make this intuition
precise. We operate once again in the setting of general function
optimization, i.e., we are interested in optimizing a function $f(\x)
= \mathbb{E}_{\factorvar} [\F(\x; \factorvar)]$ obeying the (global)
PL inequality with constant $\mu$, as well as certain local curvature
conditions.

Recall that we are given an initial point $\x[0]$ with finite cost
$\f(\x[0])$; the global upper bound on the cost that we target in the
analysis is set according to the cost $\f(\x[0])$ of this
initialization.  Given the initial gap to optimality $\Delta_0 \defn
\f(\x[0]) - \f(\xstar)$, we define the set
\begin{align}
\boundedset{0} \defn \bigr \{ \x \mid \f(\x) - \f(\xstar) \leq 10
\diff{0} \bigr \},
\end{align}
corresponding to points $\x$ whose cost gap is at most ten times the
initial cost gap $\diff{0}$.

Assume that the function $\f$ is $(\smoothness_x, \radiusone{\x})$
locally smooth and $(\Lipcon_{\x}, \radiustwo{\x})$ locally Lipschitz
at the point $\x$. Thus, both of these properties hold simultaneously
within a neighborhood of radius \mbox{$\rho_{\x} = \min\{
  \radiusone{\x}, \radiustwo{\x} \}$} of the point $\x$. Now define
the quantities
\begin{align*}
\globalSmooth{0} \defn \sup_{\x \in \boundedset{0}} \smoothness_x,
\qquad \Lipcon_{0} \defn \sup_{\x \in \boundedset{0}} \Lipcon_{\x},
\quad \text{ and } \quad \rho_{0} \defn \inf_{\x \in \boundedset{0}}
\rho_{\x}.
\end{align*}
By defining these quantities, we have effectively transformed the
local properties of the function $\f$ into global properties that hold
over the bounded set $\boundedset{0}$. We also define a convenient
functional of these curvature parameters $\curvature_{0} \defn \min
\left\{ \frac{1}{2 \globalSmooth{0} }, \frac{\rho_{0}}{\Lipcon_{0}}
\right\}$, which simplifies the statements of our results.
Importantly, these smoothness properties only hold locally, and so we
must also ensure that the steps taken by our algorithm are not too
large. This is controlled by both the step-size as well as the norms
of our gradient estimate $\gradest$ computed over the course of the
algorithm. Define the uniform bounds
\begin{align*}
\gradbound_{\infty} = \sup_{\x \in \boundedset{0}} \| \gradest(\x)
\|_2, \qquad \text{ and } \qquad \gradbound_{2} = \sup_{\x \in
  \boundedset{0}} \EE \left[ \| \gradest(\x) - \EE \left[ \gradest(\x)
    \mid \x \right] \|_2^2 \right]
\end{align*}
on the point-wise gradient norm and its variance, respectively.  Note
that these quantities also depend implicitly on the smoothing radius
$\smoothingradius$ and on how the gradient estimate $\gradest$ is
computed.

With this set-up, we are now ready to state the main result regarding
the convergence rate of Algorithm~\ref{sgd_simple} on the functions of
interest.
Note that here and throughout the rest of the paper, $C$ denotes some universal constant (which
may change from line to line). For two sequences $g_n$ and $h_n$, we also use the standard notation $g_n \sim h_n$ and $g_n = \Theta(h_n)$
interchangeably, to mean that the sequences are within a (universal) constant multiplicative factor of each other.
\begin{theorem}
\label{thm:mainthm}
Suppose that the step-size and smoothing radius are chosen so as to satisfy
\begin{subequations}
\begin{align}
\step \leq \min\left\{ \frac{\epsilon \Plconst}{240 \globalSmooth{0}
  \gradbound_{2} }, \; \frac{1}{2 \globalSmooth{0}}, \;
\frac{\rho_0}{\gradbound_{\infty}} \right\}, \qquad \text{ and }
\qquad \smoothingradius \leq \min \left\{ \frac{\curvature_{0}
  \Plconst}{8 \globalSmooth{0}} \sqrt{\frac{\epsilon}{15}}, \;
\frac{1}{2\globalSmooth{0}}\sqrt{\frac{\epsilon \Plconst}{30} }, \;
\rho_{0} \right\}.
\end{align}
Then for a given error tolerance $\epsilon$ such that $\epsilon \log
(120 \Delta_0 / \epsilon) < \frac{10}{3} \Delta_0$, the iterate
$\x[\totIter]$ of Algorithm~\ref{sgd_simple} after \mbox{$\totIter =
  \frac{4}{\step \Plconst} \log\left( \frac{ 120 \diff{0}}{\epsilon}
  \right)$} steps satisfies the bound
\begin{align}
\f(\x[\totIter]) - \f(\xstar) \leq \epsilon
\end{align}
\end{subequations}
with probability greater than $3/4$.
\end{theorem}
A few comments on Theorem~\ref{thm:mainthm} are in order. First,
notice that the algorithm is guaranteed to return an
$\epsilon$-accurate solution with constant probability $\frac{3}{4}$.  This probability bound of $\frac{3}{4}$ in itself
can be sharpened by a slightly more refined analysis with different
constants.  Additionally, by examining the proof, it can be seen that
we establish a result (cf. Proposition~\ref{prop:thm} in
Section~\ref{sec:proofs}) that is slightly stronger than
Theorem~\ref{thm:mainthm}, and then obtain the theorem from this more
general result.  The proof of the theorem itself is relatively short,
and makes use of a carefully constructed martingale along with an
appropriately defined stopping time.  As mentioned before, the main
challenge in the proof is to ensure that we have bounded iterates
while still preserving the strong convergence properties of zero-order
stochastic methods for smooth functions that satisfy the PL property.

It should be noted that Theorem~\ref{thm:mainthm} is a general
guarantee: it characterizes the zero-order complexity of optimizing
locally smooth functions that satisfy a PL inequality in terms of
properties of the gradient estimates obtained over the course of the
algorithm. In particular, two properties of these estimates appear:
the variance of the estimate, as well as a uniform bound on its
size. These quantities, in turn, depend on both the noise in the
zero-order evaluations as well as our choice of query model.  In the
next section, we specialize Theorem~\ref{thm:mainthm} so as to derive
particular consequences for the LQR models introduced above.

\subsection{Consequences for LQR optimization} \label{sec:cons-lqr}

Theorem~\ref{thm:mainthm} yields immediate consequences for LQR
optimization in various settings, and the dependence of the
optimization rates on the tolerance $\epsilon$ is summarized by
Table~\ref{tab:lqr}. We state and discuss precise versions of these
results below.

\begin{table}[]
\centering
\begin{tabular}{c|c|c|c|c|}
\cline{2-5}
\multicolumn{1}{r|}{Parameter settings}                                                                                  & \multirow{2}{*}{\begin{tabular}[c]{@{}c@{}}Smoothing radius \\ $\smoothingradius$\end{tabular}} & \multirow{2}{*}{\begin{tabular}[c]{@{}c@{}}Variance \\ $\gradbound_2$\end{tabular}} & \multirow{2}{*}{\begin{tabular}[c]{@{}c@{}}Step-size \\ $\step$\end{tabular}} & \multirow{2}{*}{\begin{tabular}[c]{@{}c@{}} \#queries \\ $T$\end{tabular}} \\
\multicolumn{1}{l|}{Query Model}                                                                                         &                                                                                                 &                                                                                     &                                                                               &                                                                             \\ \hline
\multicolumn{1}{|c|}{\begin{tabular}[c]{@{}c@{}}One-point LQR \\ (Random initialization/\\ Noisy dynamics)\end{tabular}} & $\order{\sqrt{\epsilon}}$                                                                       & $\order{\epsilon^{-1}}$                                                             & $\order{\epsilon^{2}}$                                                        & $\ordertil{\epsilon^{-2}}$                                                  \\ \hline
\multicolumn{1}{|c|}{\begin{tabular}[c]{@{}c@{}}Two-point LQR \\ (Random initialization)\end{tabular}}                   & $\order{\sqrt{\epsilon}}$                                                                       & $\order{1}$                                                                         & $\order{\epsilon}$                                                            & $\ordertil{\epsilon^{-1}}$                                                  \\ \hline
\end{tabular}
\caption{Derivative-free complexity of LQR optimization under the two
  query models, as a function of the final error tolerance
  $\epsilon$. The multiplicative pre-factors are functions of the
  effective dimension $\lqrdims$ and curvature parameters, and differ
  in the three cases; see the statements of the corollaries below. }
\label{tab:lqr}
\end{table}

First, let us consider the random initialization model. From the
various lemmas in Section~\ref{sec:lqrprop}, we know that the
population objective $\Cinit(\K)$ is locally $(\smoothnessK{\K},
\radiusone{\K})$ smooth and $(\lipconK{\K}, \radiustwo{\K})$
Lipschitz, and also globally $\lqrPl$-PL. By assumption, we are given
a starting point $\K[0]$ having finite population cost
$\Cinit(\K[0])$. Proceeding as in the previous section, we may thus
define the set
\begin{align}
  \label{EqnDefnBoundedSet}
\boundedset{\lqr} \defn \braces{ \K \mid \Cinit(\K) - \Cinit(\Kstar)
  \leq 10 \diff{0}},
\end{align}
corresponding to point $\x$ whose cost gap is at most ten times the
initial cost gap to optimality \mbox{$\diff{0} = \Cinit(\K[0]) -
  \Cinit(\Kstar)$.}

Now define the quantities
\begin{align*}
\globalSmooth{\lqr} \defn \sup_{\K \in \boundedset{\lqr}}
\smoothnessK{\K}, \qquad \Lipcon_{\lqr} \defn \sup_{\K \in
  \boundedset{\lqr}} \lipconK{\K}, \quad \text{ and } \quad
\rho_{\lqr} \defn \inf_{\K \in \boundedset{\lqr}} \localK{\K},
\end{align*}
thereby transforming the local smoothness properties of the function
$\Cinit$ into global properties that hold over the bounded set
$\boundedset{0}$. Once again, let $\curvature_{\lqr} \defn \min
\left\{ \frac{1}{2 \globalSmooth{\lqr} },
\frac{\rho_{\lqr}}{\Lipcon_{\lqr}} \right\}$ be a functional of these
curvature parameters that simplifies the statements of our results.
\footnote{Let us make a brief comment on the finiteness of these quantities in the absence of compactness. The quantity $\phi_{\lqr}$ is finite, simply by definition of the set $\boundedset{\lqr}$. In the sequel, we show that for any $K \in \boundedset{\lqr}$, $\phi_K$ can be bounded by a polynomial of $10 \Delta_0$. Hence, $\phi_\lqr$ can also be bounded by a polynomial of $10 \Delta_0$, implying it is finite. A similar argument shows that $\lambda_\lqr$ is finite and $\rho_\lqr>0$.}

With this setup, we now establish the following corollaries for
derivative-free policy optimization for linear quadratic systems.

\begin{corollary}[One-point, Random initialization]
\label{cor:init1}
Suppose that the step-size and smoothing radius are chosen such that
{\normalsize{
\begin{align*}
\step & \leq C \min\left\{ \frac{\epsilon \lqrPl
  \smoothingradius^2}{\globalSmooth{\lqr} \noisebound^2 \lqrdims^2
  [\Cinit(\K[0])]^2 }, \; \frac{1}{\globalSmooth{\lqr}}, \;
\frac{\rho_{\lqr} \smoothingradius}{\noisebound \lqrdims
  [\Cinit(\K[0])]} \right\}, \text{ and } \\ \smoothingradius &\leq
\min \left\{ \frac{\curvature_{\lqr} \lqrPl}{8 \globalSmooth{\lqr}}
\sqrt{\frac{\epsilon}{15}}, \;
\frac{1}{2\globalSmooth{\lqr}}\sqrt{\frac{\epsilon \lqrPl}{30} }, \;
\rho_{\lqr}, \; \frac{10 \Cinit(\K[0])}{\Lipcon_{\lqr}} \right\},
\end{align*}
}}
for some universal constant $C$. Then for any error tolerance $\epsilon$ such that $\epsilon \log (120
\Delta_0 / \epsilon) < \frac{10}{3} \Delta_0$, running
Algorithm~\ref{sgd_simple} for \mbox{$\totIter = \frac{4}{\step
    \Plconst} \log\left( \frac{ 120 \diff{0}}{\epsilon} \right)$}
iterations yields an iterate $\K[\totIter]$ such that
\begin{align*}
\Cinit(\K[\totIter]) - \Cinit(\Kstar) \leq \epsilon
\end{align*}
with probability greater than $3/4$.
\end{corollary}

\begin{comment}
\begin{remark} \label{rem:randint}
If the initial state distribution $\statedistributionlqr$ is also
Gaussian, the $\noisebound^2$ term in the step-size bound above can be
replaced by $1$.
\end{remark}
\apcomment{This case is interesting, but the Gaussian distribution
  doesn't satisfy the uniform boundedness condition required to bound
  $\gradbound_{\infty}$. Should probably leave it out.}
\end{comment}


Let us parse this result briefly. Treating the other parameters as constants, note that it is valid to choose $r \sim \epsilon^{1/2}$;
the above result then shows that with a choice of step-size $\eta \sim \epsilon^2$, the canonical zero-order algorithm converges using $T \sim \eta^{-1} \log (1 / \epsilon) = \ordertil{\epsilon^{-2}}$ steps. This is in spite of the high-variance estimates obtained by the algorithm, and the theorem also guarantees stability of all the iterates with constant probability.

Interestingly, the result above (or more generally, Theorem~\ref{thm:mainthm}) also yields an $\ordertil{\epsilon^{-2}}$
convergence rate for the family of high-variance \emph{minibatch}
derivative-free algorithms, where $k$ zero-order samples are used to
estimate the gradient at any point, thereby reducing its variance. The canonical algorithm corresponds to the
case $k = 1$, while that of Fazel et al. corresponds to the case of
some large $k$.
In particular, choosing a minibatch of size $k$ results in the
variance of the gradient $G_2$ being reduced by a factor $k$, allowing us to
increase our step-size proportionally and converge in $1/k$-fraction
of the number of iterations (but with the same number of zero-order
evaluations in total). For completeness, we provide an analysis
tailored to the algorithm of Fazel et al.~\cite{kakade18} in Appendix~\ref{app:fazel},
which shows that our techniques can be used to sharpen their rates to guarantee
$\epsilon$-approximate policy optimization with $\ordertil{\epsilon^{-2}}$ zero-order
evaluations.

Let us also briefly discuss the upper bounds on the step-size
that are required for the corollary to hold.
As stated, the step-size is required to satisfy the bound
$\step \leq \frac{\smoothingradius \rho_{\lqr}}{10 \Cinit(\K[0])}$, but
this condition is an artifact of the analysis and can be removed
(see Appendix~\ref{app:fazel}).
In addition, the step-size is also required to be bounded by
the curvature properties of the
function. Operationally speaking, this means that for larger
step-sizes, we are unable to guarantee stability of the policies
obtained over the course of the algorithm. Such a bottleneck is in
fact also observed in practice, as shown in
Figure~\ref{fig:plt_minibatch} for both the one-point and two-point settings.
\pgfplotsset{width=6.5cm,compat=1.9}
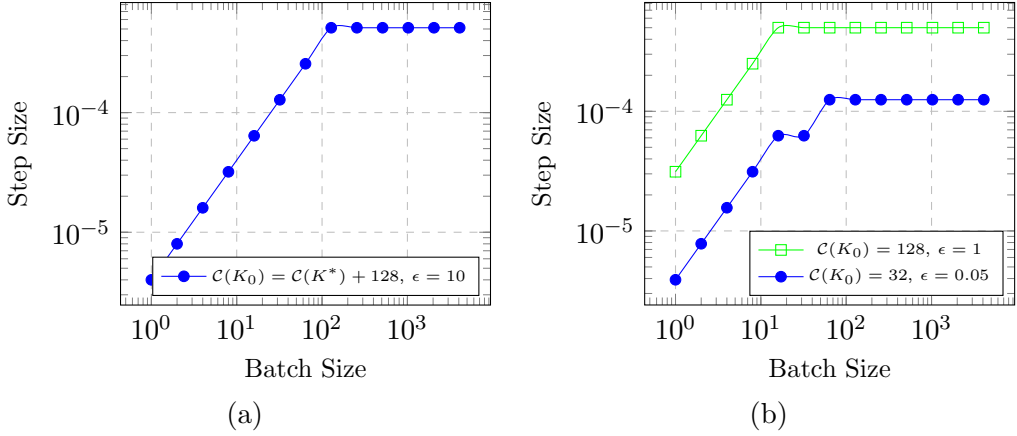
\begin{figure*}[t!]
\centering\hspace*{-4ex}
\begin{tabular}{cc}
  \begin{tikzpicture}
  \begin{loglogaxis}[
      xlabel={Batch Size},
      ylabel={Step Size},
      label style={font=\small},
      legend style={legend pos=south east,font=\tiny},
      log base y={10},
      log base x={10},
    xmajorgrids=true,
      ymajorgrids=true,
    grid style=dashed,
  max space between ticks=30
  ]
  \addplot[
      color=blue,
  smooth,
      mark=*,
      ]
      coordinates {
      (1, 0.000004)
      (2, 0.000004*2)
      (4, 0.000004*4)
      (8, 0.000004*8)
      (16, 0.000004*16)
      (32, 0.000004*32)
      (64, 0.000004*64)
      (128, 0.000004*128)
  (256, 0.000004*256/2)
  (512, 0.000004*512/4)
  (1024, 0.000004*1024/8)
  (2048, 0.000004*2048/16)
  (4096, 0.000004*4096/32)
      };
  \legend{$\C{\Kzero} = \C{\Kstar} + 128 \text{, } \accuracy = 10$
      }
  \end{loglogaxis}
  \end{tikzpicture} &\begin{tikzpicture}
  \begin{loglogaxis}[
      xlabel={Batch Size},
      ylabel={Step Size},
      label style={font=\small},
      legend style={legend pos=south east,font=\tiny},
      log base y={10},
      log base x={10},
    xmajorgrids=true,
      ymajorgrids=true,
    grid style=dashed,
  max space between ticks=30
  ]
  \addplot[
      color=green,
  smooth,
      mark=square,
      ]
      coordinates {
      (1, 0.0005/16)
      (2, 0.0005/16*2)
      (4, 0.0005/16*4)
      (8, 0.0005/16*8)
      (16, 0.0005/16*16)
      (32, 0.0005/16*16)
      (64, 0.0005/16*16)
      (128, 0.0005/16*16)
  (256, 0.0005/16*16)
  (512, 0.0005/16*16)
  (1024, 0.0005/16*16)
  (2048, 0.0005/16*16)
  (4096, 0.0005/16*16)
      };
  \addplot[
      color=blue,
  smooth,
      mark=*,
      ]
      coordinates {
      (1, 0.0005/128)
      (2, 0.0005/128*2)
      (4, 0.0005/128*4)
      (8, 0.0005/128*8)
      (16, 0.0005/128*16)
      (32, 0.0005/128*16)
      (64, 0.0005/128*32)
      (128, 0.0005/128*32)
  (256, 0.0005/128*32)
  (512, 0.0005/128*32)
  (1024, 0.0005/128*32)
  (2048, 0.0005/128*32)
  (4096, 0.0005/128*32)
      };
  \legend{$\C{\Kzero} = 128\text{, } \accuracy = 1$,
      $\C{\Kzero} = 32\text{, } \accuracy = 0.05$
      }
  \end{loglogaxis}
  \end{tikzpicture}
\\
(a)&(b)
\end{tabular}
\caption{Plot of the maximum step-size that allows for convergence, plotted against the size of the mini-batch used to estimate the gradient in randomly initialized LQR with (a) one-point evaluations and (b) two-point evaluations. The step-size plateaus due to stability considerations, leading to a higher zero-order complexity in spite of the lower variance estimates afforded by large batch-sizes. Plots were obtained by averaging $20$ runs of Algorithm~\ref{sgd_simple}. For more problem details, see Appendix~\ref{sec:additional_exp}.}
  \label{fig:plt_minibatch}
\end{figure*}

We now turn to the two-point setting, in which we obtain two noisy
evaluations per query.
\begin{corollary}[Two-point, Random initialization] \label{cor:init2}
Suppose that the step-size and smoothing radius are chosen so as to
satisfy
\begin{align*}
\step \leq \min\left\{ \frac{\epsilon \lqrPl}{240 \globalSmooth{\lqr}
  \lqrdims \Lipcon_{\lqr}^2 }, \; \frac{1}{2 \globalSmooth{\lqr}}, \;
\frac{\rho_{\lqr}}{\lqrdims \Lipcon_{\lqr}} \right\}, \qquad \text{
  and } \qquad \smoothingradius \leq \min \left\{
\frac{\curvature_{\lqr} \lqrPl}{8 \globalSmooth{\lqr}}
\sqrt{\frac{\epsilon}{15}}, \;
\frac{1}{2\globalSmooth{\lqr}}\sqrt{\frac{\epsilon \lqrPl}{30} }, \;
\rho_{\lqr} \right\}.
\end{align*}
Then for any error tolerance $\epsilon$ such that $\epsilon \log (120
\Delta_0 / \epsilon) < \frac{10}{3} \Delta_0$, running
Algorithm~\ref{sgd_simple} for \mbox{ $\totIter = \frac{4}{\step
    \Plconst} \log\left( \frac{ 120 \diff{0}}{\epsilon} \right)$}
iterations yields an iterate $\K[\totIter]$ such that
\begin{align*}
\Cinit(\K[\totIter]) - \Cinit(\Kstar) \leq \epsilon
\end{align*}
with probability greater than $3/4$.
\end{corollary}
As known from the literature on zero-order optimization in convex
settings~\cite{duchi15, shamir17}, the two-point query model allows us
to substantially reduce the variance of our gradient estimate, thus
ensuring much faster convergence than with one-point evaluations. The
most salient difference is the fact that we now converge with
$\ordertil{1/\epsilon}$ iterations as opposed to the
$\ordertil{1/\epsilon^2}$ iterations required in
Corollary~\ref{cor:init1}. This gap between the two settings is
substantial and merits further investigation, but in general, it is
clear that two-point evaluations should certainly be used if
available. This gap, and other differences, are discussed shortly.

Let us now turn to establishing convergence results for the noisy
dynamics model in the one-point setting. Note that
Lemma~\ref{lem:noisy-random} provides a way to directly relate the
population costs of the random initialization and noisy dynamics
models; furthermore, the set $\boundedset{\lqr}$ is exactly the
same. In addition, since we look at a discounted cost $\Cdyn$ in this
setting, the corresponding curvature parameters have an inherent
dependence on $\discount$ which we denote using corresponding
subscripts. With an additional computation of the variance and norm of
the gradient estimates, we then obtain
the following corollary for one-point optimization of the noisy
dynamics model.  Our statement involves the constants
\begin{align*}
\gradbound_{2, \lqr} & \defn \left( \frac{\lqrdims}{r}\cdot \frac{2 (\opnorm{\Q} + \opnorm{R} \lambda_{\lqr, \discount}^2)
  \noisebound}{1-\sqrt{\gamma}}\right)^2 \cdot \left(
\frac{20\Cdyn(\K[0])}{\sigma_{\min}(\Q)} \bigg( \frac{1 -
  \discount}{\discount} \bigg) \right)^{3} \text{ and
} \\
\gradbound_{\infty, \lqr} & \defn \frac{\lqrdims}{r}\cdot \frac{2 (\opnorm{\Q} + \opnorm{R} \lambda_{\lqr, \discount}^2)
  \noisebound}{1-\sqrt{\gamma}} \cdot \left(
\frac{20\Cdyn(\K[0])}{\sigma_{\min}(\Q)} \bigg( \frac{1 -
  \discount}{\discount} \bigg) \right)^{3/2}.
\end{align*}

\begin{corollary}[One-point, Noisy dynamics] \label{cor:noisydyn}
Suppose that the step-size and smoothing radius are chosen so as to
satisfy
\begin{align*}
\step & \leq \min\left\{ \frac{\epsilon \mu_{\lqr, \discount}}{240
  \globalSmooth{\lqr, \discount} \gradbound_{2, \lqr} }, \; \frac{1}{2
  \globalSmooth{\lqr, \discount}}, \; \frac{\rho_{\lqr,
    \discount}}{\gradbound_{\infty, \lqr}} \right\}, \text{ and } \\
\smoothingradius & \leq \min \left\{ \frac{\curvature_{\lqr,
    \discount}\cdot \mu_{\lqr, \discount}}{8 \globalSmooth{\lqr,
    \discount}} \sqrt{\frac{\epsilon}{15}}, \;
\frac{1}{2\globalSmooth{\lqr, \discount}}\sqrt{\frac{\epsilon
    \cdot\mu_{\lqr, \discount}}{30} }, \; \rho_{\lqr, \discount}
\right\}.
\end{align*}
Then for any error tolerance
$\epsilon$ such that $\epsilon \log (120 \Delta_0 / \epsilon) <
\frac{10}{3} \Delta_0$, Algorithm~\ref{sgd_simple} with \mbox{
  $\totIter = \frac{4}{\step \Plconst_{\lqr, \discount}} \log\left(
  \frac{ 120 \diff{0}}{\epsilon} \right)$} iterations yields an
iterate $\K[\totIter]$ such that
\begin{align*}
\Cdyn(\K[\totIter]) - \Cdyn(\Kstar) \leq \epsilon
\end{align*}
with probability greater than $3/4$.
\end{corollary}
\begin{comment}
\apcomment{modify above bound with correct noise variance, and by
  introducing factors of $\gamma$ as necessary.}
\begin{remark}
\label{rem:noisydyn}
If the noise distribution $\noisedistributionlqr$ is also Gaussian, we
may instead use the bound
\begin{align*}
\gradbound_{2, \lqr} \leq \frac{D^2}{r^2} \cdot \left( 20 \Cdyn
(\K[0]) \right)^2.
\end{align*}
\end{remark}
\apcomment{Similarly to above, this case is interesting, but the
  Gaussian distribution doesn't satisfy the uniform boundedness
  condition required to bound $\gradbound_{\infty}$. Should probably
  leave it out.}
\end{comment}

Thus, we have shown that the one-point settings for both the random
initialization and noisy dynamics models exhibit similar behaviors in
the different parameters. Reasoning heuristically, such a behavior is
due to the fact that the additional additive noise in the dynamics is
quickly damped away by the discount factor, so that the cost is
dominated by the noise in the initial iterates. The variance bound,
however, is substantially different, and this leads to the differing
dependence on the smoothness parameters and dimension of the problem.

Another interesting problem studied in the noisy dynamics model is one
of bounding the regret of online procedures. Equipped with a high
probability bound on convergence---as opposed to the constant
probability bound currently posited by Corollary~\ref{cor:noisydyn}---the offline guarantee and associated algorithm can in principle be
turned into a no-regret learner in the online setting. We leave
this extension to future work.

Let us now briefly discuss the dependence of the various bounds on the
different parameters of the LQR objective, in the various cases above.

\pgfplotsset{width=5.8cm,compat=1.9}
\begin{figure*}[t!]
\centering\hspace*{-4ex}
\begin{tabular}{ccc}
  \begin{tikzpicture}
    \begin{loglogaxis}[
      xlabel={$\epsilon^{-1}$},
      ylabel={Zero Order Complexity},
      label style={font=\small},
      legend style={legend pos=north west,font=\tiny},
      log base y={10},
      log base x={10},
      ymin=20,
      ymax=10000000,
      xmajorgrids=true,
      ymajorgrids=true,
      grid style=dashed,
      style=thick,
      max space between ticks=30
  ]
      \addplot[
  color = blue,
  dotted,
  forget plot
  ]
  coordinates { 
  (1, 660.185269305)
  (10^2, 706742.1451)
  };
  \addplot[
      color=blue,
      smooth,
      only marks=True,
      mark=*,
      ]
      coordinates {
      (1, 1002)
      (1.93069773, 1742)
      (3.72759372, 4298)
      (7.19685673, 10684)
      (13.89495494, 22411)
      (26.82695795, 99713)
      (51.79474679, 254031)
      (100, 1041577)
      };
      \legend{$\C{\Kzero} = \C{\Kstar} + 3$
      }
  \end{loglogaxis}
  \end{tikzpicture}&
  \begin{tikzpicture}
  \begin{loglogaxis}[
      xlabel={$\epsilon^{-1}$},
      ylabel={Zero Order Complexity},
      label style={font=\small},
      legend style={legend pos=north west,font=\tiny},
      log base y={10},
      log base x={10},
      ymin=20,
      ymax=10000000,
      xmajorgrids=true,
      ymajorgrids=true,
      grid style=dashed,
      style=thick,
      max space between ticks=30,
      ytick={100, 1000, 10000, 100000, 1000000, 10000000},
    ]
  \addplot[
  color = red,
  dotted,
  forget plot
  ]
  coordinates { 
  (1, 1276.73073201)
  (10^2, 151656.689838)
  };
  \addplot[
      color=red,
      smooth,
      only marks=True,
      mark=+,
      ]
      coordinates {
      (1, 574*2)
      (1.93069773, 1204*2)
      (3.72759372, 1857*2)
      (7.19685673, 6981*2)
      (13.89495494, 12361*2)
      (26.82695795, 24826*2)
      (51.79474679, 36157*2)
      (100, 55231*2)
      };
  \addplot[
  color = green,
  dotted,
  forget plot
  ]
  coordinates { 
  (1, 125.20883186)
  (10^2, 16551.3339731)
  };
  \addplot[ 
      color=green,
      smooth,
      only marks=True,
      mark=square,
      ]
      coordinates {
      (1, 48*2)
      (1.93069773, 146*2)
      (3.72759372, 286*2)
      (7.19685673, 497*2)
      (13.89495494, 983*2)
      (26.82695795, 2662*2)
      (51.79474679, 4211*2)
      (100, 6564*2)
      };
      \addplot[
  color = blue,
  dotted,
  forget plot
  ]
  coordinates { 
  (1, 33.4712050348)
  (10^2, 4240.47001468)
  };
  \addplot[
      color=blue,
      smooth,
      only marks=True,
      mark=*,
      ]
      coordinates {
      (1, 14*2)
      (1.93069773, 37*2)
      (3.72759372, 65*2)
      (7.19685673, 127*2)
      (13.89495494, 323*2)
      (26.82695795, 632*2)
      (51.79474679, 1104*2)
      (100, 1645*2)
      };
      \legend{$\C{\Kzero} = \C{\Kstar} + 26.92$,
      $\C{\Kzero} = \C{\Kstar} + 10.92$,
      $\C{\Kzero} = \C{\Kstar} + 2.92$
      }
  \end{loglogaxis}
  \end{tikzpicture}&
  \begin{tikzpicture}
  \begin{loglogaxis}[
      xlabel={$\epsilon^{-1}$},
      ylabel={Zero Order Complexity},
      label style={font=\small},
      legend style={legend pos=north west,font=\tiny},
      log base y={10},
      log base x={10},
      ymin=20,
      ymax=10000000,
      xmajorgrids=true,
      ymajorgrids=true,
      grid style=dashed,
      style=thick,
      max space between ticks=30
  ]
      \addplot[
  color = blue,
  dotted,
  forget plot
  ]
  coordinates { 
  (1, 489.554518103)
  (10^2, 7639307.60407)
  };
  \addplot[
      color=blue,
      smooth,
      only marks=True,
      mark=*,
      ]
      coordinates {
      (1, 556)
      (1.93069773, 1989)
      (3.72759372, 6478)
      (7.19685673, 26831)
      (13.89495494, 115073)
      (26.82695795, 569382)
      (51.79474679, 2441412)
      (100, 6362307)
      };
      \legend{$\C{\Kzero} = \C{\Kstar} + 3$
      }
  \end{loglogaxis}
  \end{tikzpicture}
\\
(a)&(b)&(c)
\end{tabular}
\caption{Number of samples required to reach an error tolerance of $\epsilon$, plotted against $1/\epsilon$, for (a) Randomly initialized LQR with one-point evaluations (b) Randomly initialized LQR with two-point evaluations for differing values of the initial cost, and (c) Noisy dynamics LQR model with one-point evaluations. We use $\mathcal{C}$ to denote the population cost in the various cases, and the plots were obtained by averaging $20$ runs of Algorithm~\ref{sgd_simple}. Each dotted line represents the line of best fit for the corresponding data points. For more problem details, see Appendix~\ref{sec:additional_exp}.}
  \label{fig:plt_eps}
\end{figure*}
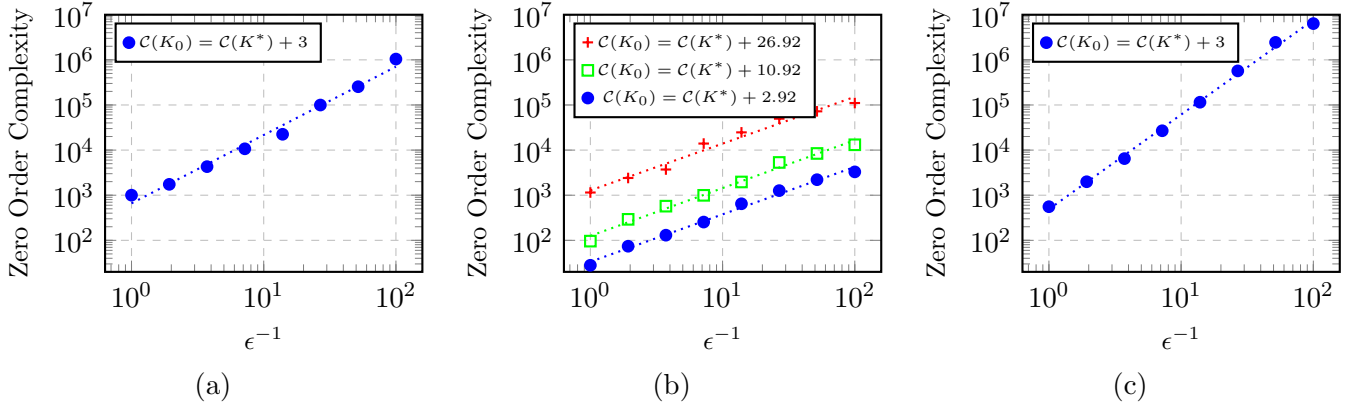

\paragraph{Dependence on $\epsilon$:} Our bounds illustrate
two distinct dependences on the tolerance parameter $\epsilon$. In
particular, the zero-order complexity scales proportional to
$\epsilon^{-2}$ for both one-point settings
(Corollaries~\ref{cor:init1} and~\ref{cor:noisydyn}), but proportional
to $\epsilon^{-1}$ in the two-point setting
(Corollary~\ref{cor:init2}). As alluded to before, this distinction
arises due to the lower variance of the gradient estimator in the
two-point setting. Lemma~\ref{lem:lipschitz_cost_lqr} establishes the
Lipschitz property of the LQR cost function for each instantiation of
the noise variable $s_0$, which ensures that the Lipschitz constant of
our \emph{sample} cost function is also bounded; therefore, the noise
of the problem reduces as we approach the optimum solution. In
contrast, the optimization problem with one-point evaluations becomes
more difficult the closer we are to the optimum solution, since the
noise remains constant, while the ``signal" in the problem (measured by
the rate of decrease of the population cost function) reduces as we
approach the optimum. The $O(1/ \epsilon^2)$ dependence in the
one-point settings is reminiscent of the complexity required to
optimize strongly convex and smooth functions~\cite{agarwal10,
  shamir12}, and it would be interesting if a matching lower bound
could also be proved in this LQR setting\footnote{Note that this lower
  bound follows immediately for the class of PL and smooth
  functions.}. Even in the absence of such a lower bound, the
one-point setting is strictly worse than the two-point setting even
with respect to the other parameters of the problem, which we discuss
next. Figure~\ref{fig:plt_eps} shows the convergence rate of the
algorithm in all three settings as a function of $\epsilon$, where we confirm
that scalings in practice corroborate our theory quite accurately.
It is also worth noting that model-based algorithms for this problem
require $\order{\epsilon^{-1}}$ trajectory samples to return an $\epsilon$-approximate policy in the noisy dynamics setting
 (see, e.g.~\cite{dean17}). Thus, while a one-point zero-order method is outperformed by
these algorithms---note that the comparison is not quite fair, since zero-order
algorithms only require access to noise cost evaluations and not the state sequence---a two-point
variant is similar to model-based methods in its dependence\footnote{Note that the comparison
is inherently imprecise, since we are comparing upper bounds to upper bounds. In practice, one
would certainly prefer the use of a model-based method when provided access to the state sequence.}
on $\epsilon$.

\paragraph{Dependence on dimension:} The dependence on dimension
enters once again via our bound on the variance of the gradient
estimate, as is typical of many derivative-free
procedures~\cite{duchi15, shamir17}. The two-point setting gives rise
to the best dimension dependence (linear in $\lqrdims$), and the
reason is similar to why this occurs for convex
optimization~\cite{shamir17}. It is particularly interesting to
compare the dimension dependence to results in model-based
control. There, in the noisy dynamics model, the sample complexity
scales with the sum of state and control dimensions $\statedim +
\controldim$, whereas the dependence in the two-point setting is on their
product $\lqrdims = \statedim \cdot \controldim$. However, each
observation in that setting consists of a state vector of length
$\statedim$, while here we only get access to scalar cost values, and
so in that loose sense, the complexities of the two settings are
comparable.

In the one-point setting, the dependence on dimension is
significantly poorer, and at least quadratic. This of course ignores
other dimension-dependent factors such as $\noisebound$, as well as
the curvature parameters $(\smoothness_{\lqr}, \Lipcon_{\lqr}, \mu)$
(see the discussion below).

\pgfplotsset{width=6.5cm,compat=1.9}
\begin{figure*}[t!]
\centering\hspace*{-4ex}
\begin{tabular}{cc}
  \begin{tikzpicture}
  \begin{loglogaxis}[
      xlabel={$\C{\Kzero}$},
      ylabel={Zero Order Complexity},
      label style={font=\small},
      legend style={legend pos=south east,font=\tiny},
      log base y={10},
      log base x={10},
      ymin=500,
      xmax = 1000,
      xmajorgrids=true,
      ymajorgrids=true,
      grid style=dashed,
      style=thick,
      max space between ticks=30
  ]
\addplot[
  color = blue,
  dotted,
  forget plot
  ]
  coordinates { 
  (8, 1189.46848832)
  (512, 4931430.23943)
  };
  \addplot[
      color=blue,
      smooth,
      only marks=True,
      mark=*,
      ]
      coordinates {
      (8, 641*2)
      (16, 2039*2)
      (32, 9596*2)
      (64, 28020*2)
      (128, 286034*2)
      (256, 619056*2)
      (512, 1940659*2)
      };
      \legend{
      $\epsilon = 0.1$
      }
  \end{loglogaxis}
  \end{tikzpicture}&
  \begin{tikzpicture}
  \begin{loglogaxis}[
      xlabel={$\C{\Kzero}$},
      ylabel={Zero Order Complexity},
      label style={font=\small},
      legend style={legend pos=south east,font=\tiny},
      log base y={10},
      log base x={10},
      ymin=200,
      xmax = 1000,
      xmajorgrids=true,
      ymajorgrids=true,
      grid style=dashed,
      style=thick,
      max space between ticks=30
  ]
\addplot[
  color = blue,
  dotted,
  forget plot
  ]
  coordinates { 
  (4, 395.588600654)
  (512, 164551.927279)
  };
  \addplot[
      color=blue,
      smooth,
      only marks=True,
      mark=*,
      ]
      coordinates {
      (4, 385)
      (8, 784)
      (16, 3268)
      (32, 4041)
      (64, 9925)
      (128, 41839)
      (256, 95600)
      (512, 113467)
      };
      \legend{
      $\epsilon = 1$
      }
  \end{loglogaxis}
  \end{tikzpicture}
\\
(a)&(b)
\end{tabular}
\caption{Number of samples required to reach a fixed error tolerance of $\accuracy$, plotted against the cost of the initialization $\Kzero$, for (a) Randomly initialized LQR with two-point evaluations (b) Noisy dynamics LQR with one-point evaluations. The plots were obtained by averaging 20 runs of Algorithm~\ref{sgd_simple}. Each dotted line represents the line of best fit for the corresponding data points. For more problem details, see Appendix~\ref{sec:additional_exp}.}
  \label{fig:plt_init}
\end{figure*}
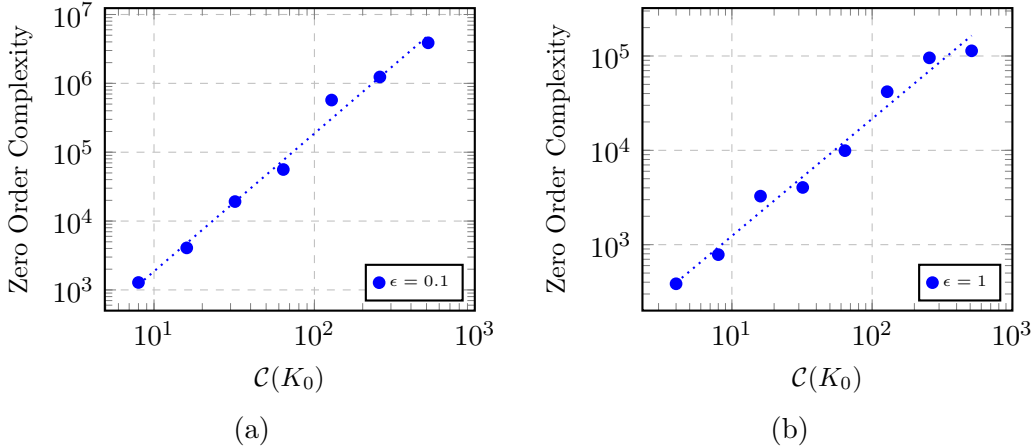

\paragraph{Dependence on curvature parameters:} The iteration complexity
scales linearly in the smoothness parameter of the problem
$\smoothness_{\lqr}$, and quadratically in the other curvature
parameters. See Appendix~\ref{sec:polynomials_bounded} for precise definitions of these parameters
for the LQR problem.
In particular, it is worth noting that our tightest bounds
for these quantities depend on the dimension of the problem implicitly
for some LQR instances, and are actually lower-order polynomials of the
initial cost. In practice, however, it is likely that much sharper bounds
can be proved on these parameters, e.g., in simulation (see
Figure~\ref{fig:plt_init}), the dependence of the sample complexity on
the initial cost is in fact relatively weak---of the order $\mathcal{C}(\K[0])^2$---and our bounds are
clearly not sharp in that sense.


\section{Proofs of main results}
\label{sec:proofs}

In this section, we provide proofs of Theorem~\ref{thm:mainthm}, and
Corollaries~\ref{cor:init1},~\ref{cor:init2},
and~\ref{cor:noisydyn}. The proofs of the corollaries require many
technical lemmas, whose proofs we postpone to the appendix.


\subsection{Proof of Theorem~\ref{thm:mainthm}}

Recall that by assumption, the population function $f$ has domain
$\domx \subseteq \real^d$ and satisfies the following properties over
the restricted domain $\boundedset{0} \subseteq \domx$, previously
defined in equation~\eqref{EqnDefnBoundedSet}:
\begin{enumerate}
  \item[(a)] It has $(\globalSmooth{0}, \rho_{0})$-locally
    Lipschitz gradients,
  \item[(b)] It is $(\Lipcon_{0}, \rho_{0})$-locally Lipschitz, and
  \item[(c)] It is globally $\Plconst$-PL.
\end{enumerate}
Recall the values of the step-size $\step$, smoothing radius
$\smoothingradius$, and iteration complexity $T$ posited by
Theorem~\ref{thm:mainthm}.  For ease of exposition, it is helpful to
run our stochastic zero-order method on this problem for $2T$
iterations; we thus obtain a (random) sequence of iterates $\{ \x[t]
\}_{t=0}^{2T}$. For each \mbox{$t = 0, 1, 2, \ldots$,} we define the
cost error \mbox{$\Delta_t = \f(\x[t]) - \f(\xstar)$}, as well as the
stopping time
\begin{align}
  \label{EqnDefnStoppingTime}
\tau & \defn \min \Big \{t \mid \Delta_t > 10 \Delta_0 \Big\}.
\end{align}
In words, the time $\tau$ is the index of the first iterate that exits
the bounded region $\boundedset{0}$.  The gradient estimate $\gradest$
at any point $\x \in \boundedset{0}$ is assumed to satisfy the bounds
\begin{align*}
\var(\gradest(\x)) \leq \gradbound_{2} \quad \text{ and } \quad \|
\gradest(\x) \|_2 \leq \gradbound_{\infty} \text{ almost surely}.
\end{align*}
With this set up in place, we now state and prove a proposition that
is stronger than the assertion of Theorem~\ref{thm:mainthm}.
\begin{proposition}
\label{prop:thm}
With the parameter settings of Theorem~\ref{thm:mainthm}, we have
\begin{align*}
\EE[ \Delta_T 1_{\tau > T}] \leq \epsilon/20,
\end{align*}
and furthermore, the event $\{\tau > T \}$ occurs with probability
greater than $4/5$.
\end{proposition}
\noindent Let us verify that Proposition~\ref{prop:thm} implies the
claim of Theorem~\ref{thm:mainthm}.  We have
\begin{align*}
\Prob\{ \diff{T} \geq \epsilon \} &\leq \Prob\{ \diff{T} 1_{\tau > T}
\geq \epsilon \} + \Prob\{ 1_{\tau \leq T} \} \\ &\stackrel{(i)} \leq
\frac{1}{\epsilon} \EE[ \diff{T} 1_{\tau > T} ] + \Prob\{ 1_{\tau \leq
  T} \} \\ &\stackrel{(ii)} \leq 1/20 + 1/5 \\ &\leq 1/4,
\end{align*}
where step (i) follows from Markov's inequality, and step (ii) from
Proposition~\ref{prop:thm}.  Thus, Theorem~\ref{thm:mainthm} follows
as a direct consequence of Proposition~\ref{prop:thm}, and we dedicate
the rest of the proof to establishing Proposition~\ref{prop:thm}.

Let $\EE^t$ to represent the expectation conditioned on the randomness
up to time $t$.  The following lemma bounds the progress of one step
of the algorithm:
\begin{lemma}
  \label{lem:PLsmooth}
Given any function satisfying the previously stated properties,
suppose that we run Algorithm~\ref{sgd_simple} with smoothing radius
$\smoothingradius \leq \rho_{0}$, and with a step-size $\step$ such
that $\| \step \estgrad{t} \|_2 \leq \rho_{0}$ almost surely.  Then
for any $t = 0, 1, \ldots$ such that $\x[t] \in \boundedset{0}$, we
have
\begin{align}
\EE^t \left[ \diff{t + 1} \right] & \leq \Big ( 1 - \frac{\step
  \Plconst}{4} \Big ) \diff{t} + \frac{\globalSmooth{0} \step^2}{2}
\gradbound_{2} + \step \Plconst \frac{\epsilon}{120}.
\end{align}
\end{lemma}
\noindent The proof of the lemma is postponed to
Section~\ref{sec:PLsmooth}. Taking it as given, let us now establish
Proposition~\ref{prop:thm}.

Proposition~\ref{prop:thm} has two natural parts; let us focus first on proving the bound on the expectation.
Let $\sigmafield_{t}$
denote the $\sigma$-field containing all the randomness in the first
$t$ iterates. Conditioning on this $\sigma$-field yields
\begin{align*}
  \EE [\Delta_{t+1} 1_{\tau > t+1} \mid \sigmafield_{t} ] & \leq \EE
      [\Delta_{t+1} 1_{\tau > t} \; \mid \; \sigmafield_{t} ] \;
      \stackrel{(i)}{=} \; \EE
        [\Delta_{t+1} \; \mid \; \sigmafield_{t} ] 1_{\tau > t},
\end{align*}
where step (i) follows since $\tau$ is a stopping time, and so the
random variable $1_{\tau > t}$ is determined completely by the
sigma-field $\sigmafield_{t}$. \\

\noindent We now split the proof into two cases.

\paragraph{Case 1:} Assume that $\tau > t$, so that we have the inclusion
$\x[t] \in \boundedset{0}$. In addition, note that the iterate
$\x[t+1]$ is obtained after a stochastic zero-order step whose size is
bounded as
\begin{align*}
\enorm{\step \estgrad{t}} \; \leq \; \step \gradbound_{\infty} \leq
\rho_{0},
\end{align*}
where we have used the fact that $\step \leq
\frac{\rho_{0}}{\gradbound_{\infty}}$.

%
We may thus apply Lemma~\ref{lem:PLsmooth} to obtain
\begin{subequations}
  \begin{align}
    \label{EqnCaseOneBound}
\EE [\Delta_{t+1} \; \mid \; \sigmafield_{t} ] & \leq \Big(1 -
\frac{\step \Plconst}{4} \Big) \diff{t} + \frac{\globalSmooth{0}
  \step^2}{2} \gradbound_{2} + \step \Plconst \frac{\epsilon}{120}.
\end{align}

\paragraph{Case 2:} In this case, we have $\tau \leq t$, so that
\begin{align}
  \label{EqnCaseTwoBound}
\EE [\Delta_{t+1} \; \mid \; \sigmafield_{t} ] 1_{\tau > t} = 0.
\end{align}
\end{subequations}

Now combining the bounds~\eqref{EqnCaseOneBound}
and~\eqref{EqnCaseTwoBound} from the the two cases yields the
inequality
\begin{align}
\EE [\Delta_{t+1} \; \mid \; \sigmafield_{t} ] 1_{\tau > t} &\leq
\left\{ \Big(1 - \frac{\step \Plconst}{4} \Big) \diff{t} +
\frac{\globalSmooth{0} \step^2}{2} \gradbound_{2} + \step \Plconst
\frac{\epsilon}{120} \right\} 1_{{\tau > t}} \label{eq:recbound1} \\
& \leq \Big(1 - \frac{\step \Plconst}{4} \Big) \: \Delta_t 1_{{\tau >
    t}} + \frac{\globalSmooth{0} \step^2}{2} \gradbound_{2} + \step
\Plconst \frac{\epsilon}{120}.  \nonumber
\end{align}
Taking expectations over the sigma-field $\sigmafield_{t}$ and then
arguing inductively yields
\begin{align*}
\EE [\Delta_{t+1} 1_{\tau > t+1} ] & \leq \Big( 1 - \frac{\step
  \Plconst}{4} \Big)^{t+1} \Delta_0 + \left( \frac{\globalSmooth{0}
  \step^2}{2} \gradbound_{2} + \step \Plconst
\frac{\epsilon}{120}\right) \sum_{i = 0}^{t} \Big(1 - \frac{\eta
  \Plconst}{4} \Big)^i \\
& \leq \Big(1 - \frac{\step \Plconst}{4} \Big)^{t+1} \Delta_0 +
2\frac{\step}{\Plconst} \globalSmooth{0} \gradbound_{2} +
\frac{4\epsilon}{120}.
\end{align*}
Setting $t +1 = T$ then establishes the first part of the proposition with
substitutions of the various parameters.

We now turn to establishing
that $\Prob \{\tau > T \} \geq 4/5$. We do so
by setting up a suitable super-martingale on our iterate sequence and
appealing to classical maximal inequalities. Recall that we run the
algorithm for $2T$ steps for convenience, and thereby obtain a set of
$2T$ random variables $\{\Delta_1, \ldots, \Delta_{2T}\}$. With the
stopping time $\tau$ defined as before~\eqref{EqnDefnStoppingTime},
define the stopped process
\begin{align*}
Y_t & \defn \Delta_{\tau \wedge t} + (2T - t) \left( \frac{\globalSmooth{0} \step^2}{2} \gradbound_{2} + \step \Plconst \frac{\epsilon}{120}\right) \quad \mbox{
  for each $t \in [2T]$.}
\end{align*}
Note that by construction, each random variable $Y_t$ is non-negative and almost surely
bounded by the locally Lipschitz nature of the function.

We claim that $\{Y_t\}_{t = 0}^{2T}$ is a super-martingale.  In order
to prove this claim, we first write
\begin{align}
  \label{EqnAshwinOwesPizza}
\EE[Y_{t + 1} \mid \sigmafield_{t} ] = \EE [ \Delta_{\tau \wedge
    (t+1)} 1_{\tau \leq t} \mid \sigmafield_{t} ] + \EE [ \Delta_{\tau
    \wedge (t+1)} 1_{\tau > t} \mid \sigmafield_{t} ] + (2T - (t+1))
\left( \frac{\globalSmooth{0} \step^2}{2} \gradbound_{2} + \step \Plconst \frac{\epsilon}{120}\right).
\end{align}
Beginning by bounding the first term on the right-hand side, we have
\begin{subequations}
\begin{align}
\label{EqnTermOne}
\EE [ \Delta_{\tau \wedge (t+1)} 1_{\tau \leq t} \mid \sigmafield_{t}
] = \EE [ \Delta_{\tau \wedge t} 1_{\tau \leq t} \mid \sigmafield_{t}
] = \Delta_{\tau \wedge t} 1_{\tau \leq t}.
\end{align}
As for the second term, we have
\begin{align}
\EE [ \Delta_{\tau \wedge (t+1)} 1_{\tau > t} \mid \sigmafield_{t} ] &
= \EE [ \Delta_{t+1} 1_{\tau > t} \mid \sigmafield_{t} ] \nonumber \\
& = \EE [ \Delta_{t+1} \mid \sigmafield_{t} ] 1_{\tau > t} \nonumber \\
%
%
& \stackrel{(iii)}{\leq} \Big(1 - \frac{\step \Plconst}{4} \Big)
\Delta_t 1_{\tau > t} + \left( \frac{\globalSmooth{0} \step^2}{2}
\gradbound_{2} + \step \Plconst \frac{\epsilon}{120}\right) 1_{\tau >
  t} \nonumber \\
\label{EqnTermTwo}
& \leq \Big( 1 - \frac{\step \Plconst}{4} \Big) \Delta_{\tau \wedge t}
1_{\tau > t} + \frac{\globalSmooth{0} \step^2}{2} \gradbound_{2} +
\step \Plconst \frac{\epsilon}{120},
\end{align}
\end{subequations}
where step (iii) follows from using inequality~\eqref{eq:recbound1}.

Substituting the bounds~\eqref{EqnTermOne} and~\eqref{EqnTermTwo} into
our original inequality~\eqref{EqnAshwinOwesPizza}, we find that
\begin{align*}
\EE[Y_{t + 1} \mid \sigmafield_{t} ] & = \EE [ \Delta_{\tau \wedge
    (t+1)} 1_{\tau \leq t} \mid \sigmafield_{t} ] + \EE [ \Delta_{\tau
    \wedge (t+1)} 1_{\tau > t} \mid \sigmafield_{t} ] + (2T - (t+1))
\left( \frac{\globalSmooth{0} \step^2}{2} \gradbound_{2} + \step \Plconst \frac{\epsilon}{120}\right) \\
& \leq \Delta_{\tau \wedge t} 1_{\tau \leq t} + (1 - \step \Plconst/4)
\Delta_{\tau \wedge t} 1_{\tau > t} + \left( \frac{\globalSmooth{0} \step^2}{2} \gradbound_{2} + \step \Plconst \frac{\epsilon}{120}\right) + (2T - (t+1)) \left( \frac{\globalSmooth{0} \step^2}{2} \gradbound_{2} + \step \Plconst \frac{\epsilon}{120}\right)
\\
& \stackrel{(iv)}{\leq} \Delta_{\tau \wedge t} + (2T - t) \left( \frac{\globalSmooth{0} \step^2}{2} \gradbound_{2} + \step \Plconst \frac{\epsilon}{120}\right) \\
& = Y_t,
\end{align*}
where step (iv) follows from the inequality $\step \Plconst \Delta_{\tau
  \wedge t} \geq 0$.  We have thus verified the super-martingale
property.

Finally, applying Doob's maximal inequality for
super-martingales (see, e.g.~\citealp{durrett10}) yields
\begin{align*}
\Pr\{ \max_{t \in [2T]} Y_t \geq \nu \} & \leq \frac{\EE[Y_0]}{\nu} \\
& = \frac{1}{\nu} \left( \Delta_0 + 2T \left\{ \frac{\globalSmooth{0} \step^2}{2} \gradbound_{2} + \step \Plconst \frac{\epsilon}{120} \right\} \right) \\
& \stackrel{(v)}{=} \frac{1}{\nu} \left( \Delta_0 + \frac{\epsilon}{5} \log
(120 \Delta_0/ \epsilon) \right),
\end{align*}
where step (v) follows from the substitutions $T = \frac{4}{\step
  \Plconst} \log (120 \Delta_0 / \epsilon)$, and $\step \leq \frac{\epsilon \Plconst}{240 \globalSmooth{0} \gradbound_{2}}$.
As long as $\epsilon$ is sufficiently small so as to ensure that
$\epsilon \log (120 \Delta_0 / \epsilon) < 5 \Delta_0$, setting $\nu =
10 \Delta_0$ completes the proof.

\subsubsection{Proof of Lemma~\ref{lem:PLsmooth}} \label{sec:PLsmooth}

Recall that the domain of the function $\f$ is $\domx \subseteq
\mathbb{R}^{\dims}$.  For a scalar $\smoothingradius > 0$, the
smoothed version $\fr(\x)$ is given by $\fr(\x) \defn \Exs \left[
  \f(\x + \smoothingradius \unifdirball) \right]$, where the
expectation above is taken with respect to the randomness in
$\unifdirball$, and $\unifdirball$ has uniform distribution on a
$\dims$-dimensional ball $\Ball^{\dims}$ of unit radius. The estimate
$\gradest$ of the gradient $\grad \fr$ at $\x$ is given by
\begin{align*}
\gradest(\x) =
  \begin{cases}
  F(\x + \smoothingradius \shelldirection,\factorvar) \;
  \frac{\dims}{\smoothingradius} \unifdir \; &\text{ if operating in
    one-point setting} \\
\big [ F{(\x + \smoothingradius \shelldirection,\factorvar)} - F{(\x -
    \smoothingradius \shelldirection, \factorvar)} \big] \;
\frac{\dims}{2\smoothingradius} \unifdir \; &\text{ if operating in
  two-point setting,}
  \end{cases}
\end{align*}
%
where $\unifdir$ has a uniform distribution on the shell of the sphere
$\Shell^{\dims - 1}$ of unit radius, and $\factorvar$ is sampled at
random from $\factorvardistribution$. 
The following result summarizes some useful properties of the smoothed version of
$\f$, and relates it to the gradient estimate $\gradest$.
\begin{lemma}
\label{lem:Thm1_lemma}
The smoothed version $\fr$ of $\f$ with smoothing radius
$\smoothingradius$ has the following properties:
\begin{itemize}
\item[(a)] $\grad \fr(x) = \Exs \left[ \gradest(x) \right]$.
\item[(b)] $\enorm{\grad \fr(\x) - \grad \f(x)} \leq
  \globalSmooth{0} \smoothingradius$.
\end{itemize}
\end{lemma}
\noindent Versions of these properties have appeared in past
work~\cite{flax04,agarwal10,shamir17}, but we provide proofs in
Appendix~\ref{AppSmoothing} for completeness.

Taking Lemma~\ref{lem:Thm1_lemma} as given, we now prove
Lemma~\ref{lem:PLsmooth}.  Let $\sigmafield_{t}$ denote the sigma
field generated by the randomness up to iteration $t$, and $\Exs$
denote the total expectation operator.  We define $\CondExs{t} \defn
\Exs\left[ \cdot \mid \sigmafield_{t} \right]$ as the expectation
operator conditioned on the sigma field $\sigmafield_{t}$. Recall that
the function $\f$ is smooth with smoothness parameter
$\globalSmooth{0}$, and we have
\begin{align*}
\CondExs{t} \bracket{ \f(\x[t + 1]) - \f(\x[t]) } & \leq \CondExs{t}
\left[ \inprod{\grad \f(\x[t])}{\x[t + 1] - \x[t]} +
  \frac{\globalSmooth{0}}{2} \enorm{\x[t + 1] - \x[t]}^2 \right] \\
& \stackrel{(i)}{=} - \inprod{ \step \grad \f(\x[t])}{\grad
  \fr(\x[t])} + \frac{\globalSmooth{0} \step^2}{2} \CondExs{t} \left[
  \enorm{\gradest(\x[t])}^2 \right] \\
& \stackrel{(ii)}{=} - \step \enorm{\grad \f(\x[t])}^2 + \step
\globalSmooth{0} \smoothingradius \enorm{\grad \f(\x[t])} +
\frac{\globalSmooth{0} \step^2}{2} \CondExs{t} \left[
  \enorm{\gradest(\x[t])}^2 \right].
\end{align*}
Steps (i) and (ii) above follow from parts (a) and (b), respectively,
of Lemma~\ref{lem:Thm1_lemma}.
Now make the observation that
\begin{align*}
\CondExs{t} \left[ \enorm{\gradest(\x[t])}^2 \right] &=
\var(\gradest(\x[t])) + \enorm{\grad \fr(\x[t])}^2 \\
& \leq \var(\gradest(\x[t])) + 2 \enorm{\grad \f(\x[t])}^2 + 2
\enorm{\grad \fr(\x[t]) - \grad \f(\x[t])}^2 \\
& \leq \gradbound_{2} + 2 \enorm{\grad \f(\x[t])}^2 + 2
(\globalSmooth{0} \smoothingradius)^2.
\end{align*}
In addition, since the function is locally smooth at the point
$\x[t]$, we have
\begin{align*}
(\curvature - \curvature^2 \globalSmooth{0}/2) \enorm{ \grad \f(\x[t])
  }^2 & \leq \f(\x[t]) - \f(\x[t] - \curvature \grad \f(\x[t]) )
  \\ &\leq \f(\x[t]) - \f(\xstar),
\end{align*}
for some parameter $\curvature$ chosen small enough such that the
relation $\curvature \enorm{ \grad \f( \x[t]) } \leq \rho_{0}$ holds.
We may thus set $\curvature = \curvature_0 = \min \left\{ \frac{1}{2
  \globalSmooth{0} }, \frac{\rho_{0}}{\Lipcon_{0}} \right\}$ and
recall the notation $\diff{t} = \f(\x[t]) - \f(\xstar)$ to obtain
\begin{align*}
\CondExs{t} \bracket{ \diff{t+1} - \diff{t}} & \leq -\step \enorm{\grad
  \f(\x[t])}^2 + \step \globalSmooth{0} \smoothingradius
\frac{2}{\curvature_{0}} \diff{t}^{1/2} + \frac{\globalSmooth{0}
  \step^2}{2} \gradbound_{2} + \globalSmooth{0} \step^2 \left(
\enorm{\grad \f(\x[t])}^2 + (\globalSmooth{0} \smoothingradius)^2
\right) \\
& \stackrel{(iii)}{\leq} - \frac{\step \Plconst}{2} \diff{t} + 2
\frac{\step \globalSmooth{0} \smoothingradius }{\curvature_{0}}
\diff{t}^{1/2} + \frac{\globalSmooth{0} \step^2}{2} \gradbound_{2} +
\globalSmooth{0} \step^2 (\globalSmooth{0} \smoothingradius)^2, \\
& \stackrel{(iv)}{\leq} - \frac{\step \Plconst}{2} \diff{t} +
\frac{\step \Plconst}{4} \diff{t} + 4 \frac{\step (\globalSmooth{0}
  \smoothingradius)^2 }{\Plconst \curvature_{0}^2} +
\frac{\globalSmooth{0} \step^2}{2} \gradbound_{2} + \globalSmooth{0}
\step^2 (\globalSmooth{0} \smoothingradius)^2,
\end{align*}
where step (iii) follows from applying the PL inequality and using the
fact that $\step \leq \frac{1}{2 \globalSmooth{0}}$, and step (iv)
from the inequality $2ab \leq a^2 + b^2$ which holds for any pair of
scalars $(a, b)$.

Recall the assumed bounds on our parameters, namely
\begin{align*}
\step \leq \min \left\{ \frac{\epsilon \Plconst}{240
  \globalSmooth{0}}, \frac{1}{2 \globalSmooth{0}} \right\}, \quad
\text{ and} \quad \smoothingradius \leq \frac{1}{2 \globalSmooth{0}}
\min\left \{ \curvature_{0} \Plconst \sqrt{\frac{\epsilon}{240}},
\frac{1}{\globalSmooth{0}}\sqrt{\frac{\epsilon \Plconst}{30} }
\right\}.
\end{align*}
Using these bounds, we have
\begin{align*}
\CondExs{t} \bracket{ \diff{t+1} - \diff{t} } & \leq - \frac{\step \Plconst}{4}
\diff{t} + \frac{\globalSmooth{0} \step^2}{2} \gradbound_{2} + \step
\Plconst \frac{\epsilon}{120}.
\end{align*}
Finally, rearranging yields
\begin{align}
\label{eqn:PL_exp_recursion}
    \CondExs{t} \left[ \diff{t+1} \right] \leq \Big(1 - \frac{\step
      \Plconst}{4} \Big) \diff{t} + \frac{\globalSmooth{0} \step^2}{2}
    \gradbound_{2} + \step \Plconst \frac{\epsilon}{120},
\end{align}
which completes the proof of Lemma~\ref{lem:PLsmooth}.


\subsection{Proof of Corollary~\ref{cor:init1}}
\label{sec:cor1}

Recall the properties of the LQR cost function $\Cinit$ that were
established in Lemmas~\ref{lem:lipschitz_cost_lqr}
through~\ref{lem:pl_inequality}.  Taking these properties as given
(see Appendix~\ref{sec:randint_appendix} for the proofs of the
lemmas), the only remaining detail is to establish the bounds
\begin{align}
  \label{EqnCoffee}
\gradbound_2 \leq C \left(\frac{\lqrdims}{r} \noisebound \Cinit(\K[0])
\right)^2 \quad \text{ and } \quad \gradbound_{\infty} &\leq C
\frac{\lqrdims}{r} \noisebound \Cinit(\K[0]).
\end{align}
In fact, it suffices to prove the second bound in
equation~\eqref{EqnCoffee}, since we have $\gradbound_{2} \leq
\gradbound_{\infty}^2$.

Given a unit vector $u$, the norm of the gradient estimate can be
bounded as
\begin{align*}
\| \gradest_t \|_2 &= \frac{\lqrdims}{r} \Cinit(\K[t] + ru; s_0) \\
& \stackrel{(i)}{=} \frac{\lqrdims}{r} s_0^\top \PK s_0 \\ &\leq
\frac{\lqrdims}{r} \enorm{ s_0}^2 \opnorm{\PK}
\\
& \stackrel{(ii)}{\leq} \noisebound \frac{\lqrdims}{r} \Cinit(\K[t] +
ru),
\end{align*}
where step (i) follows from the relation~\eqref{eq:ckquad}, and step
(ii) from the relation~\eqref{eq:ctop}, since $\PK$ is a PSD matrix.
Finally, since $r \leq \rho_{\lqr}$, the local Lipschitz property of
the function $\Cinit$ yields
\begin{align*}
  \Cinit(\K[t] + ru) &\leq \Cinit(\K[t]) + r \Lipcon_{\K} \\
  & \leq \Cinit(\K[t]) + r \Lipcon_{\lqr} \\
  & \stackrel{(iii)}{\leq} 10 \Cinit(\K[0]) + 10 \Cinit(\K[0]),
\end{align*}
where step (iii) uses the fact that $\K[t] \in \boundedset{\lqr}$ so
that $\Cinit(\K[t]) \leq 10 \Cinit(\K[0])$, and the upper bound $r
\leq \frac{10 \Cinit(\K[0])}{\Lipcon_{\lqr}}$. Putting together the
pieces completes the proof.


\subsection{Proof of Corollary~\ref{cor:init2}}

As before, establishing Corollary~\ref{cor:init2} requires bounds on
the values of the pair $(\gradbound_{2}, \gradbound_{\infty})$, since
the remaining properties are established in
Lemmas~\ref{lem:lipschitz_cost_lqr} through~\ref{lem:pl_inequality}.

In particular, let us establish bounds on these quantities for general
optimization of a function with a two-point gradient
estimate. The following computations closely follow those of Shamir~\cite{shamir17}.


\paragraph{Second moment control:}

Using the law of iterated expectations, we have
\begin{align*}
\mathbb{E} \bigg[ \euclidnorm{\dims \frac{\F(\x + \smoothingradius
      \unifdir, \factorvar) - \F(\x - \smoothingradius \unifdir,
      \factorvar)}{2\smoothingradius} \unifdir }^2 \bigg] & =
\mathbb{E} \bigg[ \mathbb{E} \bigg[ \euclidnorm{\dims \frac{\F(\x +
        \smoothingradius \unifdir, \factorvar) - \F(\x -
        \smoothingradius \unifdir, \factorvar)}{2\smoothingradius}
      \unifdir }^2 \bigg \vert \factorvar \bigg] \bigg].
\end{align*}
Define the placeholder variable $q$ and now evaluate:
\begin{align*}
\mathbb{E} \bigg[ \euclidnorm{\dims \frac{\F(\x + \smoothingradius
      \unifdir, \factorvar) - \F(\x - \smoothingradius \unifdir,
      \factorvar)}{2\smoothingradius} \unifdir }^2 \bigg \vert
  \factorvar \bigg] &= \frac{\dims^2}{4 \smoothingradius^2} \mathbb{E}
\bigg[ (\F(\x + \smoothingradius \unifdir, \factorvar) - \F(\x -
  \smoothingradius \unifdir, \factorvar))^2 \euclidnorm{\unifdir}^2
  \bigg \vert \factorvar \bigg]. \\
& \stackrel{(i)}{=} \frac{\dims^2}{4 \smoothingradius^2} \mathbb{E}
\bigg[ (\F{(\x + \smoothingradius \unifdir, \factorvar)} - \F(\x -
  \smoothingradius \unifdir, \factorvar))^2 \bigg \vert \factorvar
  \bigg] \\
& = \frac{\dims^2}{4 \smoothingradius^2} \mathbb{E} \bigg[ (\F(\x +
  \smoothingradius \unifdir, \factorvar) - q - \F(\x -
  \smoothingradius \unifdir, \factorvar) + q)^2 \bigg \vert \factorvar
  \bigg] \\
& \stackrel{(ii)}{\leq} \frac{\dims^2}{2 \smoothingradius^2}
\mathbb{E} \bigg[ (\F(\x + \smoothingradius \unifdir, \factorvar) -
  q)^2 + (\F(\x - \smoothingradius \unifdir, \factorvar) - q)^2 \bigg
  \vert \factorvar \bigg],
\end{align*}
where equality (i) follows from the fact that $\unifdir$ is a unit
vector and inequality (ii) follows from the inequality $(a - b)^2 \leq
2(a^2 + b^2)$. We further simplify this to obtain:
\begin{align*}
\mathbb{E} \bigg[ \euclidnorm{\dims \frac{\F(\x + \smoothingradius
      \unifdir, \factorvar) - \F(\x - \smoothingradius \unifdir,
      \factorvar)}{2\smoothingradius} \unifdir }^2 \bigg \vert
  \factorvar \bigg] & \stackrel{(i)}{\leq}
\frac{\dims^2}{\smoothingradius^2} \mathbb{E} \bigg[ (\F(\x +
  \smoothingradius \unifdir, \factorvar) - q)^2 \bigg \vert \factorvar
  \bigg] \\
& \stackrel{(ii)}{\leq} \frac{\dims^2}{\smoothingradius^2} \sqrt {
  \mathbb{E} \bigg[ (\F(\x + \smoothingradius \unifdir, \factorvar) -
    q)^4 \bigg \vert \factorvar \bigg] },
\end{align*}
where inequality (i) follows from the symmetry of the uniform
distribution on the sphere, and inequality (ii) follows from
Jensen's inequality. For a fixed $\factorvar$, we now define $q =
\mathbb{E}[\F(\x + \smoothingradius \unifdir, \factorvar) \vert
  \factorvar]$. Substituting this expression yields
\begin{align*}
\mathbb{E} \bigg[ \euclidnorm{\dims \frac{\F(\x + \smoothingradius
      \unifdir, \factorvar) - \F(\x - \smoothingradius \unifdir,
      \factorvar)}{2\smoothingradius} \unifdir }^2 \bigg \vert
  \factorvar \bigg] & \leq \frac{\dims^2}{\smoothingradius^2} \sqrt {
  \mathbb{E} \bigg[ (\F(\x + \smoothingradius \unifdir, \factorvar) -
    \mathbb{E}[\F(\x + \smoothingradius \unifdir, \factorvar) \vert
      \factorvar] )^4 \bigg \vert \factorvar \bigg] } \\
& \stackrel{(i)}{\leq} \frac{\dims^2}{\smoothingradius^2}
\frac{(\Lipcon r)^2}{\dims} \\
& = \dims \Lipcon^2,
\end{align*}
where inequality (i) follows directly from Lemma 9 in Shamir~\cite{shamir17}.  The lemma can be applied since we are
conditioning on $\factorvar$, and all the randomness lies in the
selection of $\unifdir$. We have thus established the claim in part
(c).


\paragraph{Gradient estimates are bounded:}

Note that smoothing radius $\smoothingradius$ satisfies
$\smoothingradius \leq \rho_{0}$, where $\rho_{0}$ is the radius
within which the function is Lipschitz.  Consequently, the local
Lipschitz property of $\F$ implies that
\begin{align*}
\enorm{\gradest_{t}} &\defn \euclidnorm{ \dims \frac{\F(\x[t] +
  \smoothingradius \shelldirection_t,\factorvar_t) - \F(\x[t] -
  \smoothingradius \shelldirection_t,
  \factorvar_t)}{2\smoothingradius} \shelldirection_t } \\
& \leq \euclidnorm{ \dims \frac{\F(\x[t] + \smoothingradius
  \shelldirection_t; \factorvar_t) - \F(\x[t];
  \factorvar_t)}{2\smoothingradius} \shelldirection_t } + \euclidnorm{
\dims \frac{\F(\x[t]; \factorvar_t) - \F(\x[t] - \smoothingradius
  \shelldirection_t; \factorvar_t)}{2\smoothingradius}
\shelldirection_t } \\ & \leq \dims \Lipcon_{0} \frac{2 \|
  \smoothingradius \shelldirection_t \|_2}{2 \smoothingradius} =
\dims \Lipcon_{0}.
\end{align*}


\subsection{Proof of Corollary~\ref{cor:noisydyn}}
\label{sec:cor3}

As in Section~\ref{sec:cor1}, we establish bounds on the values
$\gradbound_{2}$ and $\gradbound_{\infty}$ for the noisy LQR dynamics
model. In particular, we derive a bound on $\gradbound_{\infty}$ and
use the fact that $\gradbound_{2} \leq \gradbound_{\infty}^2$ to
establish the bound on $\gradbound_{2}$. For deriving these bounds, we
use properties of the cost function $\Cdyn$ and its connections with
$\Cinit$ which are established in Lemma~\ref{lem:noisy-random} and
Lemma~\ref{lem:unifcost}; the proofs of these are deferred to
Appendix~\ref{sec:noisy_case}.

In particular, we establish the bounds
\begin{align*}
\gradbound_2 & \leq \left( \frac{\lqrdims}{r}\cdot \frac{2  (\opnorm{Q} + \opnorm{R} \lambda_{\lqr, \discount}^2)
  \noisebound}{1-\sqrt{\gamma}}\right)^2 \cdot \left(
\frac{20\Cdyn(\K[0])}{\sigma_{\min}(\Q)} \bigg( \frac{1 -
  \discount}{\discount} \bigg) \right)^{3}, \text{ and} \\
\gradbound_{\infty} &\leq \frac{\lqrdims}{r}\cdot \frac{2 (\opnorm{Q} + \opnorm{R} \lambda_{\lqr, \discount}^2)
  \noisebound}{1-\sqrt{\gamma}} \cdot \left(
\frac{20\Cdyn(\K[0])}{\sigma_{\min}(\Q)} \bigg( \frac{1 -
  \discount}{\discount} \bigg) \right)^{3/2}.
\end{align*}
For any unit vector $u$, we have,
\begin{align*}
  \| \gradest_t \|_2 &= \frac{\lqrdims}{r} \Cdyn(\K[t] + ru;
  \mathcal{Z}) \\
& \stackrel{\1}{\leq} \frac{\lqrdims}{r}\cdot \frac{2
    (\opnorm{Q} + \opnorm{R} \lambda_{\lqr, \discount}^2) \noisebound}{1-\sqrt{\gamma}} \cdot \left( \frac{\Cdyn(\K[t] +
    ru)}{\sigma_{\min}(\Q)} \bigg( \frac{1 - \discount}{\discount}
  \bigg) \right)^{3/2},
\end{align*}
where $\1$ follows from using the bound in
Lemma~\ref{lem:unifcost}, as well as the explicit choice of $\lambda_{\lqr, \discount}$ made using Lemma~\ref{bound_smoothness_etc}. Finally, using Lemma~\ref{lem:noisy-random}
and since $r \leq \rho_{\lqr, \discount}\,$, the local Lipschitz
property of the function $\Cinit$ yields
\begin{align}
\label{eq:boundCKru}
\Cdyn(\K[t] + ru) & \leq \frac{\discount}{1-\discount}\cdot
\Cinit(\K[t] + ru) \nonumber\\
& \leq \frac{\discount}{1-\discount}\cdot\left(\Cinit(\K[t]) + r
\Lipcon_{\K, \discount} \right)\nonumber\\
& \leq \frac{\discount}{1-\discount}\cdot\left( \Cinit(\K[t]) + r
\Lipcon_{\lqr, \discount} \right)\nonumber\\
& \stackrel{\1}{\leq} \frac{\discount}{1-\discount} \cdot \left( 10
\Cinit( \K[0]) + 10 \Cinit(\K[0]) \right),
\end{align}
where step $\1$ uses the fact that $\K[t] \in \boundedset{\lqr}$ so
that $\Cinit(\K[t]) \leq 10 \Cinit(\K[0])$, and the upper bound $r
\leq \frac{10 \Cinit(\K[0])}{\Lipcon_{\lqr, \discount}}$. Putting
together the pieces completes the proof.

\section{Discussion}
\label{SecDiscussion}

In this paper, we studied the model-free control problem over linear
policies through the lens of derivative-free optimization.  We derived
quantitative convergence rates for various zero-order methods when
applied to learn optimal policies based on data from noisy linear
systems with quadratic costs.  In particular, we showed that one-point
and two-point variants of a canonical derivative-free optimization
method achieve fast rates of convergence for the non-convex LQR
problem. Notably, our proof deals directly with some additional
difficulties that are specific to this problem and do not arise in the
analysis of typical optimization algorithms.  More precisely, our
proof involves careful control of both the (potentially) unbounded
nature of the cost function, and the non-convexity of the underlying
domain. Interestingly, our proof only relies on certain local
properties of the function that can be guaranteed over a bounded set;
for this reason, the optimization-theoretic result in this paper
(stated as Theorem~\ref{thm:mainthm}) is more broadly applicable
beyond the RL setting.

While this paper analyzes a canonical zero-order optimization
algorithm for model-free control of linear quadratic systems, many
open questions remain. One such question concerns lower bounds for LQR
problems in the model-free setting, thereby showing quantitative gaps
between such a setting and that of model-based control. While we
conjecture that the convergence bounds of
Corollaries~\ref{cor:init1},~\ref{cor:init2}, and~\ref{cor:noisydyn}
are sharp in terms of their dependence on the error tolerance
$\epsilon$, establishing this rigorously will require ideas from the
extensive literature on lower bounds in zero-order
optimization~\cite{shamir12}. Another important direction is establish
the sharpness (or otherwise) of our bounds in terms of the dimension
of the problem, as well as to obtain tight characterizations of the
local curvature parameters of the problem around a particular policy
$\K$ in terms of the cost at $\K$.

We also mention that our sharp characterizations of the cost function
are likely to be useful in sharpening analyses\footnote{Here again, the techniques of Fazel et al.~\cite{kakade18} yield
a bound of the order $\ordertil{\epsilon^{-4}}$, but we conjecture that this bound should be improvable
at least to $\ordertil{\epsilon^{-2}}$.} of the natural gradient
algorithm~\cite{kakade18} as well as in analyzing the popular REINFORCE
algorithm as applied to the LQR problem. We leave these interesting questions
to future work.

In the broader context of model-free reinforcement learning as well,
there are many open questions. First, a derivative-free algorithm over
linear policies is reasonable even in other systems; can we establish
provable guarantees over larger classes of problems? Second, there is
no need to restrict ourselves to linear policies; in practical RL
systems, derivative-free algorithms are run for policies that
parametrized in a much more complex fashion. How does the sample
complexity of the problem change with the class of policies over which
we are optimizing?

\subsection*{Acknowledgements}
This work was partially supported by National Science Foundation
  grant NSF-DMS-1612948 and Office of Naval Research grant
  ONR-N00014-18-1-2640 to MJW. AP was additionally supported in part
  by NSF CCF-1704967.  PLB gratefully acknowledges the support of the
  NSF through grant IIS-1619362. KB was supported in part by AFOSR
  through grant FA9550-17-1-0308 and NSF through grant IIS-1619362.


\appendix

\section{Properties of the randomly initialized LQR problem}
\label{sec:randint_appendix}

In this section, we establish some fundamental properties of the cost
function $\Cinit$, and provide proofs of
Lemmas~\ref{lem:lipschitz_cost_lqr}
and~\ref{lem:lipschitz_gradient_lqr}. As part of these proofs, we
provide explicit bounds for the local curvature parameters
$(\widetilde{\Lipcon}_{\lqr}, \Lipcon_{\lqr}, \rho_{\lqr},
\smoothness_{\lqr}, \lqrPl)$. We make frequent use of results
established by Fazel et al.~\cite{kakade18}, and as mentioned before,
Lemmas~\ref{lem:lipschitz_cost_lqr}
and~\ref{lem:lipschitz_gradient_lqr} are refinements of their results.


\paragraph{Notation:}

In this section, we introduce some shorthand to reduce notational
overhead. Throughout, we assume that $\gamma = 1$; the general case is
straightforward to obtain with the substitutions
\begin{align*}
A \mapsto \sqrt{\gamma} A, \text{ and } B \mapsto \sqrt{\gamma} B.
\end{align*}
We also use the shorthand $\C{\K} \defn \Cinit(\K)$ for this section. Much (but not all) of the notation we use overlaps with the notation used in Fazel et al.~\cite{kakade18}.

We define the matrix $\PK$ as the solution to the following fixed
point equation:
\begin{align*}
\PK = \Q + \K^{\top} \R \K + (\A - \B \K)^{\top} \PK (\A - \B \K),
\end{align*}
and we define the state correlation matrix $\sigmaK$ as:
\begin{align}
\label{eq:defn_state_corr}
\sigmaK = \mathbb{E} \left[ \sum_{t = 0}^{\infty} \state[t]
  \state[t]^{\top} \right] \quad \mbox{such that} \quad \state[t] =
(\A - \B \K) \state[t-1].
\end{align}
It is straightforward to see that we have
\begin{align}
\C{\K} = \EE [ s_0^\top \PK s_0 ], \label{eq:ckquad}
\end{align}
and we make frequent use of this representation in the sequel.

Recall that we have $\mathbb{E}[\initialstate \initialstate^{\top}] =
I$, so that
\begin{align}
 \label{eq:ctop}
\C{\K} & = \trace(\PK).
\end{align}
Moreover, under this assumption, the cost function $\Cfunc$ satisfies
the PL inequality with PL constant
$\frac{\opnorm{\Sigma_{\K^*}}}{\singularminR}$, see Lemma 3 in the
paper by Fazel et al.~\cite{kakade18}.

Also define the \emph{natural gradient} of the cost function as
\begin{align*}
\EK \defn 2(\R + \B^{\top} \PK \B)\K - \B^{\top} \PK \A,
\end{align*}
so that we have $\grad \C{\K} = \EK \sigmaK$. For any symmetric matrix
$X$, the perturbation operators $\TK (\cdot)$ and $\FK (\cdot)$ are
defined as
\begin{align*}
\TK (X) = \sum_{t = 0}^{\infty} (\A - \B \K)^t X [(\A - \B
  \K)^{\top}]^t, \quad \text{ and } \quad \FK (X) = (\A - \B \K) X (\A
- \B \K)^{\top}.
\end{align*}
Finally, the operator norms of the operators $\TK (\cdot)$ and $\FK
(\cdot)$ are defined as
\begin{align*}
\begin{gathered}
\opnorm{ \TK } = \sup_{X} \frac{\opnorm{ \TK (X) } }{\opnorm{X}}
\text{ and} \\ \opnorm{ \FK } = \sup_{X} \frac{\opnorm{ \FK (X) }
}{\opnorm{X}}.
\end{gathered}
\end{align*}

\subsubsection*{Useful constants:}

\label{sec:useful_constants}
We now define several polynomials of $\C{\K}$, which are useful in
various proofs in this section.
\begin{itemize}
\item $\Kpolyzero = \frac{1}{\singularminR} (\sqrt{(\opnorm{\R} +
  \opnorm{\B}^2 \C{\K}) (\C{\K} - \C{\Kstar})} + \opnorm{B} \opnorm{A} \C{\K})$
\item $\Kpolyone = \max \{ \frac{\C{\K}}{\singularminQ} \sqrt{(\opnorm{\R} +
  \opnorm{\B}^2 \C{\K}) (\C{\K} - \C{\Kstar})}, \Kpolyzero \}$
\item $\Kpolytwo = 4 \bigg( \frac{\C{\K}}{ \singularminQ} \bigg)^2 \opnorm{\Q} \opnorm{\B} (\opnorm{A} + \opnorm{B} \Kpolyone + 1)$
\item $\Kpolythree = 8 \bigg( \frac{\C{\K}}{ \singularminQ} \bigg)^2 (
  \Kpolyone )^2 \opnorm{\R} \opnorm{\B} (\opnorm{A} + \opnorm{B}
  \Kpolyone + 1)$
\item $\Kpolyfour = 2 \bigg( \frac{\C{\K}}{ \singularminQ} \bigg)^2 (
  \Kpolyone + 1 ) \opnorm{\R}$
\item $\Kpolyfive = \sqrt{(\opnorm{\R} + \opnorm{\B}^2 \C{\K} )
  (\C{\K} - \C{\Kstar})}$
\item $\Kpolysix = \fronorm{\R} + \fronorm{\B}^2 (\Kpolyone + 1)
  (\Kpolytwo + \Kpolythree + \Kpolyfour) + \fronorm{\B}^2 \C{\K} +
  \fronorm{\B} \opnorm{\A} (\Kpolytwo + \Kpolythree + \Kpolyfour)$
\item $\Kpolyseven = 5 \Kpolysix \frac{\C{\K}}{\singularminQ} + 4
  \Kpolyfive \bigg( \frac{\C{\K}}{\singularminQ} \bigg)^2 \opnorm{\B}
  (\opnorm{A} + \opnorm{B} \Kpolyone) + \Kpolyone$.
\item $\Kpolyeight = \noisebound (\Kpolytwo + \Kpolythree +
  \Kpolyfour)$.
\item $\Kpolynine = \min \bigg \{ \frac{\singularminQ}{4 \C{\K}
  \opnorm{\B}(\opnorm{\A} + \opnorm{\B} \Kpolyone + 1)}, 1 \bigg \}$.
\end{itemize}


With these definitions at hand, we are now in a position to establish
Lemmas~\ref{lem:lipschitz_cost_lqr}
and~\ref{lem:lipschitz_gradient_lqr}.

\subsection{Proof of Lemma~\ref{lem:lipschitz_cost_lqr}}

Let us restate a precise version of the lemma for convenience.
\begin{lemma} \label{lem:lipcostrestated}
For any pair $(\Ktwo, \K)$ such that $\fronorm{\Ktwo - \K} \leq
\Kpolynine$, we have
\begin{align*}
\vert \C{\Ktwo} - \C{\K} \vert &\leq
\left(\frac{\statedim}{\noisebound}\right) \Kpolyeight \fronorm{\Ktwo
  - \K}, \text{ and} \\ \vert \Cstate{\Ktwo}{\initialstate} -
\Cstate{\K}{\initialstate} \vert &\leq \Kpolyeight \fronorm{\Ktwo -
  \K}.
\end{align*}
\end{lemma}
Comparing Lemma~\ref{lem:lipcostrestated} with the statement of
Lemma~\ref{lem:lipschitz_cost_lqr}, we have therefore established that
\begin{align*}
\radiustwo{\K} &= \Kpolynine, \\ \Lipcon_{\K} &= \left(
\frac{\statedim}{\noisebound} \right) \Kpolyeight, \text{ and}
\\ \widetilde{\Lipcon}_{\K} &= \Kpolyeight.
\end{align*}
are valid choices for the local radius and Lipschitz constants respectively. Note that we have $\Lipcon_{\K} \leq \widetilde{\Lipcon}_{\K}$, since
$\statedim \leq \noisebound$. Let us now prove
Lemma~\ref{lem:lipcostrestated}.

\begin{proof}
We have
\begin{align*}
|\C{\Ktwo} - \C{\K}| & = \trace(\PKtwo) - \trace ( \PK) \\
& \leq \statedim \opnorm{\PKtwo - \PK}.
\end{align*}

Moreover, the
sample cost satisfies the relation
\begin{align}
|\Cstate{\Ktwo}{\initialstate} - \Cstate{\K}{\initialstate}| & =
|\initialstate^{\top} \PKtwo \initialstate - \initialstate^{\top} \PK
\initialstate| \notag \\
& = |\trace(\initialstate^{\top} (\PKtwo - \PK) \initialstate)|
\notag \\
& \leq \opnorm{\PKtwo - \PK} \euclidnorm{\initialstate}^2 \notag \\
& \leq \opnorm{\PKtwo - \PK} \noisebound.
\label{eqn:CK_to_PK_bound}
\end{align}
Hence, it remains to bound $\opnorm{\PKtwo - \PK}$. To this end,
substituting the definition of the linear operator $\TK$, we have
\begin{align}
\opnorm{\PKtwo - \PK} &= \opnorm{\TKtwo(\Q + (\Ktwo)^{\top} \R \Ktwo)
  - \TK (\Q + \K^{\top} \R \K)} \notag \\
& = \opnorm{(\TKtwo - \TK)(\Q + (\Ktwo)^{\top} \R \Ktwo) -
  \TK(\K^{\top} \R \K - (\Ktwo)^{\top} \R \Ktwo) } \notag \\
\label{eqn:Pk_minus_Pk_prime_bound1}
& \leq \opnorm{(\TKtwo - \TK)\Q} + \opnorm{(\TKtwo -
  \TK)((\Ktwo)^{\top} \R \Ktwo)} \notag \\ & \phantom{\leq} \qquad
\qquad \qquad \qquad + \opnorm{\TK}\opnorm{\K^{\top} \R \K -
  (\Ktwo)^{\top} \R \Ktwo}.
\end{align}
We provide upper bounds for the three terms above as follows:
\begin{subequations}
\begin{align}
\label{eqn:PK_term_two}
\opnorm{(\TKtwo - \TK)(\Ktwo)^{\top} \R \Ktwo)} & \leq \Kpolythree
\opnorm{\K - \Ktwo} \\
\label{eqn:PK_term_one}
\opnorm{(\TKtwo - \TK)\Q} &\leq \Kpolytwo \opnorm{\K - \Ktwo} \\
\label{eqn:PK_term_three}
\opnorm{\TK}\opnorm{\K^{\top} \R \K - (\Ktwo)^{\top} \R \Ktwo} & \leq
\Kpolyfour \opnorm{\K - \Ktwo}.
\end{align}
\end{subequations}
Taking the above bounds as given at the moment, we have from
equation~\eqref{eqn:Pk_minus_Pk_prime_bound1} that
\begin{align}
  \opnorm{\PKtwo - \PK} \leq (\Kpolytwo + \Kpolythree + \Kpolyfour)
  \opnorm{\Ktwo - \K},
\label{eqn:Pk_mius_Pk_prime_bound}
\end{align}
Putting together the pieces completes the proof of
Lemma~\ref{lem:lipschitz_cost_lqr}.
\end{proof}
\noindent It remains to prove the upper
bounds~\eqref{eqn:PK_term_two}-~\eqref{eqn:PK_term_three}.


\paragraph{Auxiliary bounds:}

Proofs of the bounds~\eqref{eqn:PK_term_two}
through~\eqref{eqn:PK_term_three} are based on the following
intermediate bounds:
\begin{subequations}
  \begin{align}
\opnorm{(\Ktwo)^{\top} \R \Ktwo - \K^{\top} \R \K} & \leq (\Kpolyone +
1) \opnorm{\R} \opnorm{\Ktwo - \K}
\label{eqn:PK_helper_bound_one} \\
\opnorm{\FKtwo - \FK} & \leq 2 \opnorm{\B} (\opnorm{\A} + \opnorm{\B}
\Kpolyone + 1) \opnorm{\Ktwo - \K}
\label{eqn:PK_helper_bound_two} \\
\opnorm{\TK}
& \leq \frac{\C{\K}}{\singularminQ}
\label{eqn:PK_helper_bound_three}\\
\opnorm{\K^{\top} \R \K} & \leq \Kpolyone^2 \opnorm{\R}.
\label{eqn:PK_helper_bound_four}
  \end{align}
\end{subequations}
We prove these bounds at the end, but let us complete the rest of the
proofs assuming these auxiliary bounds.

\paragraph{Proof of the bound~\eqref{eqn:PK_term_two}:}

The proof of this upper bound is based on Lemma~20 from the
paper by Fazel et al.~\cite{kakade18}. Accordingly, we start by verifying the
following condition for Lemma 20:
\begin{align}
\label{eq:fazelcond}
\opnorm{\FK - \FKtwo}\opnorm{(\Ktwo)^\top \R \Ktwo} \leq \frac{1}{2}.
\end{align}
Note that $\opnorm{\K} \leq \Kpolyone$ (see Lemma~22 in the paper by~\citealp{kakade18}). Also, observe that our assumption $\fronorm{\Ktwo - \K} \leq \Kpolynine$,
satisfies the assumption of Lemma~10 in the paper by Fazel et al.~\cite{kakade18},
whence we have
\begin{align}
\opnorm{\B} \opnorm{\Ktwo - \K} &\stackrel{(i)}{\leq} \opnorm{\B}
\frac{\singularminQ}{4 \C{\K} \opnorm{\B}(\opnorm{\A} + \opnorm{\B}
  \Kpolyone + 1)} \notag \\
& \stackrel{(ii)}{\leq} \frac{\singularminQ}{4 \C{\K} (\opnorm{\A - \B
    \K} + 1)} \notag \\
& \stackrel{(iii)}{\leq} \frac{1}{4},
~\label{eqn:one_fourth_bound}
\end{align}
where step (i) follows by substituting the value of $\Kpolynine$, and
step (ii) follows since $\opnorm{\A - \B \K} \leq \opnorm{\A} +
\opnorm{\B} \Kpolyone + 1$ (since $\opnorm{\K} \leq \Kpolyone$, see Lemma~22 in the paper by~\citealp{kakade18}). Step (iii) above follows since $\C{\K}
\geq \singularminQ$.  Combining the inequality~\eqref{eqn:one_fourth_bound} with Lemma~19 in
the paper by Fazel et al.~\cite{kakade18} yields
\begin{align*}
\opnorm{\FKtwo - \FK} &\leq 2 \opnorm{\A - \B \K} \opnorm{\B}
\opnorm{\Ktwo - \K} + \opnorm{B}^2 \opnorm{\Ktwo - \K}^2 \\ &\stackrel{(iv)}{\leq} 2
\opnorm{\B} (\opnorm{\A - \B \K} + 1) \opnorm{\Ktwo - \K}
\end{align*}
where step (iv) follows from the bound~\eqref{eqn:one_fourth_bound} we derived above. Finally, invoking Lemma~14 from the paper by Fazel et al.~\cite{kakade18} guarantees
that $\opnorm{\TK} \leq \frac{\C{\K}}{\singularminQ}$, and we deduce
that
\begin{align*}
\opnorm{\TK} \opnorm{\FKtwo - \FK} &\leq \frac{\C{\K}}{\singularminQ}
2 \opnorm{\B} (\opnorm{\A - \B \K} + 1) \opnorm{\Ktwo - \K} \\ &\leq
\frac{1}{2},
\end{align*}
where the last inequality follows from the assumption $\fronorm{\Ktwo
  - \K} \leq \Kpolynine$.

Now that we have verified that condition~\ref{eq:fazelcond}, invoking
Lemma~20 in the paper by Fazel et al.~\cite{kakade18} yields
\begin{align*}
  \opnorm{(\TKtwo - \TK)(\Ktwo)^{\top} \R \Ktwo} & \leq 2
  \opnorm{\TK}^2 \opnorm{\FK - \FKtwo} \opnorm{(\Ktwo)^{\top} \R
    \Ktwo} \notag \\
& \leq 2 \opnorm{\TK}^2 \opnorm{\FK - \FKtwo} \opnorm{\K^{\top} \R \K}
  \notag \\ & \phantom{\leq} \qquad + 2 \opnorm{\TK}^2 \opnorm{\FK -
    \FKtwo} \opnorm{(\Ktwo)^{\top} \R \Ktwo - \K^{\top} \R \K} \\ &
  \leq \Kpolythree \opnorm{\K - \Ktwo},
\end{align*}
where the last step above follows by substituting the
bounds~\eqref{eqn:PK_helper_bound_one}-~\eqref{eqn:PK_helper_bound_four}.


\paragraph{Proof of the bounds~\eqref{eqn:PK_term_one} and~\eqref{eqn:PK_term_three}:}

The proof of the bound~\eqref{eqn:PK_term_one} is similar to the
part~\eqref{eqn:PK_term_two} and is based on Lemma~20 from the
paper by Fazel et al.~\cite{kakade18}.  More concretely, we have
\begin{align*}
  \opnorm{(\TKtwo - \TK)\Q} \leq 2 \opnorm{\TK}^2 \opnorm{\FK -
    \FKtwo} \opnorm{\Q} \leq \Kpolytwo \opnorm{\K - \Ktwo}
\end{align*}
where the last step above follows from the
bounds~\eqref{eqn:PK_helper_bound_two}
and~\eqref{eqn:PK_helper_bound_three}.  The proof of the
bound~\eqref{eqn:PK_term_three} is a direct consequence of the
bounds~\eqref{eqn:PK_helper_bound_one}
and~\eqref{eqn:PK_helper_bound_three}.


\subsubsection{Proofs of the auxiliary bounds}

In this section we prove the auxiliary
bounds~\eqref{eqn:PK_helper_bound_one} through
to~\eqref{eqn:PK_helper_bound_four}.

\vspace{10pt}

 \noindent \underline{\textit{Bound~\eqref{eqn:PK_helper_bound_one}:}}
  Observe that
  \begin{align*}
  \opnorm{\K^{\top} \R \K - (\Ktwo)^{\top} \R \Ktwo}
  & =  \opnorm{(\Ktwo - \K)^{\top} \R (\Ktwo - \K) +
    (\Ktwo)^{\top} \R \K + \K^{\top}\R (\Ktwo) - 2 \K^{\top} \R \K} \\
  & \leq  (2 \opnorm{\R} \opnorm{\K} \opnorm{\Ktwo - \K} +
  \opnorm{\R} \opnorm{\Ktwo - \K}^2) \\
  & \stackrel{(i)}{\leq}  (2 \opnorm{\K} + 1) \opnorm{\R} \opnorm{\Ktwo -
    \K} \\
  & \stackrel{(ii)}{\leq} (2 \Kpolyone + 1) \opnorm{\R} \opnorm{\Ktwo
    - \K}.
  \end{align*}
  where step (i) follows since $\opnorm{\K - \Ktwo} \leq 1$ by assumption, and step (ii)
  follows since $\opnorm{\K} \leq \Kpolyone$ (see Lemma~22 in the paper by~\citealp{kakade18}).
  This completes the proof of bound~\eqref{eqn:PK_helper_bound_one}.

\vspace{10pt}

\noindent \underline{\textit{Bound~\eqref{eqn:PK_helper_bound_two}:}}
In order to prove bound~\eqref{eqn:PK_helper_bound_two}, we invoke
Lemma~19 in the paper by Fazel et al.~\cite{kakade18} to obtain
\begin{align*}
\opnorm{\FKtwo - \FK} &\leq 2 \opnorm{\A - \B \K} \opnorm{\B}
\opnorm{\Ktwo - \K} + \opnorm{B}^2 \opnorm{\Ktwo - \K}^2  \\
& \stackrel{(iii)}{\leq}
2 \opnorm{\A - \B \K} \opnorm{\B}
\opnorm{\Ktwo - \K} + \frac{1}{4} \opnorm{B} \opnorm{\Ktwo - \K} \\
 & \leq 2 \opnorm{\B} (\opnorm{\A} + \opnorm{\B} \Kpolyone + 1) \opnorm{\Ktwo - \K}
\end{align*}
where step (iii) above follows from the upper bound~\eqref{eqn:one_fourth_bound}.
This completes the proof of the bound~\eqref{eqn:PK_helper_bound_two}.

\vspace{10pt}

\noindent \underline{\textit{Bound~\eqref{eqn:PK_helper_bound_three} and~\eqref{eqn:PK_helper_bound_four}:}}
The bound~\eqref{eqn:PK_helper_bound_three} above follows from Lemma 17 in the paper by Fazel et al.~\cite{kakade18},
whereas the bound~\eqref{eqn:PK_helper_bound_three} follows from the fact that
$\opnorm{\K} \leq \Kpolyone$ (see Lemma 22 in the paper by~\citealp{kakade18}).

Having established all of our auxiliary bounds, let us now proceed to a proof of Lemma~\ref{lem:lipschitz_gradient_lqr}.

\subsection{Proof of Lemma~\ref{lem:lipschitz_gradient_lqr}}
Lemma~\ref{lem:lipschitz_gradient_lqr} is a consequence of the following
result.
\begin{lemma} \label{lem:smoothnessrestated}
If $\fronorm{\Ktwo - \K} \leq \Kpolynine$, then
\begin{align*}
\fronorm{\grad \C{\Ktwo} - \grad \C{\K}} \leq \Kpolyseven \fronorm{\Ktwo -
  \K}.
\end{align*}
\end{lemma}
Indeed, comparing Lemmas~\ref{lem:smoothnessrestated}
and~\ref{lem:lipschitz_gradient_lqr}, we have that
\begin{align*}
\radiusone{\K} = \Kpolynine \quad \text{ and } \quad
\smoothnessK{\K} = \Kpolyseven,
\end{align*}
are valid choices for the local radius and smoothness constant respectively. \\

Let us now prove Lemma~\ref{lem:smoothnessrestated}.
\begin{proof}
We start by noting that from Lemma~\ref{lem:lipschitz_cost_lqr}
we have that  the cost function $\C{\K}$ is locally Lipschitz in a ball
of $\radiustwo{\K}$ around the point $\K$. Before moving
into the main argument, we mention a few auxiliary results that are
helpful in the sequel.
We start by invoking Lemma 13 from the paper by Fazel et al.~\cite{kakade18}, whence we have
\begin{align*}
  \opnorm{\PK} \leq \C{\K}
  \quad \text{and} \quad  \opnorm{\sigmaK}
\leq \frac{\C{\K}}{\singularminQ}.
\end{align*}
We also have
\begin{subequations}
\begin{align}
  \opnorm{\A - \B \K} &\leq
  \opnorm{\A} + \opnorm{\B} \opnorm{K}
  \stackrel{(i)}{\leq}
   \opnorm{\A} + \opnorm{\B} \Kpolyone \quad
   \text{and}
    \label{eqn:A_minus_BK_bd} \\
\opnorm{\sigmaKtwo} &\leq \opnorm{\sigmaK} +
\opnorm{\sigmaKtwo - \sigmaK} \stackrel{(ii)}{\leq} 5\frac{\C{\K}}{\singularminQ}.
\label{eqn:Sigma_K_prime_bound}
\end{align}
\end{subequations}
Step (i) above follows since $\opnorm{K} \leq \Kpolyone$ (see Lemma 22 in the paper by~\citealp{kakade18}),
 whereas step (ii)
follows since $\opnorm{\sigmaKtwo -\sigmaK} \leq 4 \frac{\C{\K}}{\singularminQ}$
(see Lemma 16 in the paper by~\citealp{kakade18}).


Recalling the gradient expression $\grad \C{\K} = \EK \sigmaK$. Let $\Ktwo$ be a policy such that \mbox{$\fronorm{\Ktwo - \K} \leq \Kpolynine$}. We have

\begin{align*}
\fronorm{\grad \C{\Ktwo} - \grad \C{\K}} & = \fronorm{(\EKtwo - \EK)
  \sigmaKtwo + \EK (\sigmaKtwo - \sigmaK)} \\
& \leq \fronorm{(\EKtwo - \EK)} \opnorm{\sigmaKtwo} + \fronorm{\EK}
\opnorm{(\sigmaKtwo - \sigmaK)} \\
& \stackrel{(iii)}{\leq}
 5 \Kpolysix \frac{\C{\K}}{\singularminQ} \fronorm{\Ktwo - \K} \\
& \phantom{\leq} \qquad  + 4 \Kpolyfive \bigg( \frac{\C{\K}}{\singularminQ} \bigg)^2
\frac{\opnorm{\B}(\opnorm{\A} + \opnorm{\B} \Kpolyone)}{\singularmininitial}
\fronorm{\Ktwo - \K}.
\end{align*}
The upper bound in step (iii) on the term $\fronorm{(\EKtwo - \EK)} \opnorm{\sigmaKtwo}$
follows from equation~\eqref{eqn:Sigma_K_prime_bound} and from the following upper bound which
we prove later:
\begin{align}
  \fronorm{\EKtwo - \EK} \leq \Kpolysix \fronorm{\Ktwo - K} \text{ provided } \fronorm{\Ktwo - \K} \leq \Kpolynine.
  \label{eqn:EK_minus_Ek_prime}
\end{align}
The upper bound on the term $\fronorm{\EK}
\opnorm{(\sigmaKtwo - \sigmaK)}$ in step (iii) follows from the fact
that ${\fronorm{\EK} \leq \Kpolyfive}$
(see Lemma 11 in the paper by~\citealp{kakade18}) and from the  fact that
\begin{align*}
   \opnorm{(\sigmaKtwo - \sigmaK)} & \stackrel{(iv)}{\leq}
   4 \bigg( \frac{\C{\K}}{\singularminQ} \bigg)^2
\frac{\opnorm{\B}(\opnorm{\A - \B \K} + 1)}{\singularmininitial}
\fronorm{\Ktwo - \K} \\
& \stackrel{(v)}{\leq} 4 \bigg( \frac{\C{\K}}{\singularminQ} \bigg)^2
\frac{\opnorm{\B}(\opnorm{\A} + \opnorm{\B} \Kpolyone + 1)}{\singularmininitial}
\fronorm{\Ktwo - \K},
 \end{align*}
where step (iv) follows from Lemma~16 in the
paper by Fazel et al.~\cite{kakade18}, and step (v) follows from
inequality~\eqref{eqn:A_minus_BK_bd}.


\noindent Putting together the pieces, we conclude that
the function $\grad \C{\K}$ is Lipschitz with constant
$\smoothnessK{\K}$, where $\smoothnessK{\K}$ is given by
\begin{align*}
  \smoothnessK{\K} = 5 \Kpolysix \frac{\C{\K}}{\singularminQ}
+ 4 \Kpolyfive \bigg( \frac{\C{\K}}{\singularminQ} \bigg)^2
\opnorm{\B}(\opnorm{\A} + \opnorm{\B} \Kpolyone + 1) = \Kpolyseven.
\end{align*}
\end{proof}
\noindent It remains to prove inequality~\eqref{eqn:EK_minus_Ek_prime}.

\vspace{8pt}

\paragraph{Proof of inequality~\eqref{eqn:EK_minus_Ek_prime}:}
From the definition of $\EK$, we have
\begin{align}
\fronorm{\EKtwo - \EK} & = 2 \fronorm{(\R + \B^{\top}\PKtwo \B)\Ktwo -
  \B^{\top}\PKtwo \A - (\R + \B^{\top}\PK \B)\K + \B^{\top}\PK \A}
   \notag \\
%
& = 2 \fronorm{\R (\Ktwo - \K) + \B^{\top}(\PKtwo - \PK) \B \Ktwo +
  \B^{\top} \PK \B (\Ktwo - \K) - \B^{\top}(\PKtwo - \PK) \A}
  \notag \\
& \leq 2 \fronorm{\R}  \fronorm{\Ktwo - \K} + 2\fronorm{\B^{\top}(\PKtwo - \PK) \B \Ktwo} \notag  \\
& \phantom{\leq} \qquad + 2\fronorm{\B^{\top} \PK \B (\Ktwo - \K)} + 2\fronorm{\B^{\top}(\PKtwo - \PK) \A}
\label{eqn:Ek_minus_Ek_prime_first_bound}
\end{align}
We provide upper bounds for the three terms above as follows. First, we have
\begin{subequations}
\begin{align*}
\fronorm{\B^{\top} (\PKtwo - \PK) \B \Ktwo}
& \leq
\fronorm{\B}^2 (\Kpolyone + 1) (\Kpolytwo + \Kpolythree + \Kpolyfour)
\fronorm{\Ktwo - \K},
\end{align*}
which follows from the bound~\eqref{eqn:Pk_mius_Pk_prime_bound}, since $\fronorm{\Ktwo - \K} \leq \Kpolynine$, and the relation $\opnorm{\Ktwo} \leq \opnorm{\K} + \opnorm{\Ktwo - \K} \leq
\Kpolyone + 1$. The same reasoning also yields the bound
\begin{align*}
\fronorm{\B^{\top} (\PKtwo - \PK) \A}
& \leq \fronorm{\B}
\opnorm{\A} (\Kpolytwo + \Kpolythree + \Kpolyfour) \fronorm{\Ktwo - K},
\end{align*}
Finally, since $\opnorm{\PK} \leq \C{\K}$, we have
\begin{align*}
\fronorm{\B^{\top} \PK \B (\Ktwo - \K)}
& \leq \fronorm{\B}^2 \C{\K}
\fronorm{\Ktwo - \K}.
\end{align*}
\end{subequations}
Combining the above upper bounds with the upper
bound~\eqref{eqn:Ek_minus_Ek_prime_first_bound} we conclude that
\begin{align*}
\fronorm{\EKtwo - \EK}
&\leq \Kpolysix \fronorm{\Ktwo - K},
\end{align*}
where $\Kpolysix$ is given by
\begin{align*}
  \Kpolysix =  2 \left[ \fronorm{\R} + \fronorm{\B}
  \opnorm{\A} (\Kpolytwo + \Kpolythree + \Kpolyfour) + \fronorm{\B}^2 \left( (\Kpolyone + 1) (\Kpolytwo
  + \Kpolythree + \Kpolyfour) +
  \C{\K} \right) \right].
\end{align*}

\subsection{Explicit choices for the parameters ($\rho_{\lqr}$, $\Lipcon_{\lqr}$,
$\smoothness_{\lqr}$)} \label{sec:polynomials_bounded} In order to
ease notation,
we define constants $\widetilde{c}_{K_7}$,
$\widetilde{c}_{K_8}$ and $\widetilde{c}_{K_9}$ by replacing
the scalar $\C{\K}$ by $10 \C{\Kzero} - 9 \C{\Kstar}$ in the
definitions of $\Kpolyseven$, $\Kpolyeight$ and $\Kpolynine$
respectively (see Section~\ref{sec:useful_constants}).


\begin{lemma}
\label{bound_smoothness_etc}
The parameters $\rho_{\lqr}$, $\Lipcon_{\lqr}$, $\smoothness_{\lqr}$
can be picked as follows
\begin{align*}
\begin{gathered}
\rho_{\lqr} = \widetilde{c_{K_9}} , \;\; \smoothness_{\lqr} =
\widetilde{c_{K_7}} \;\; \text{and} \;\; \Lipcon_{\lqr} =
\widetilde{c_{K_8}}.
\end{gathered}
\end{align*}
\end{lemma}
\begin{proof}
Observe that from the definition of the set $\boundedset{\lqr}$ we
have that for all $\K \in \boundedset{\lqr}$, the function value
$\C{\K}$ is upper bounded as $\C{\K} \leq 10 \C{\Kzero} - 9
\C{\Kstar}$.  Consequently, for any $\K \in \boundedset{\lqr}$ and any
$\Ktwo$ such that $\fronorm{\Ktwo - \K} \leq \widetilde{c_{K_9}}$,
we can use Lemma~\ref{lem:lipschitz_gradient_lqr} and
Lemma~\ref{lem:lipschitz_cost_lqr} respectively to show that the cost
function $\C{\K}$ has locally Lipschitz gradients with parameter
$\widetilde{c_{K_8}}$ and the function $\C{\K}$ has locally
Lipschitz function values parameter $\widetilde{c_{K_7}}$.
Combining the last observation with the definitions of $\rho_{\lqr}$,
$\Lipcon_{\lqr}$ and $\smoothness_{\lqr}$ we have that $\rho_{\lqr}
\geq \widetilde{c_{K_9}}$, $\smoothness_{\lqr} \leq
\widetilde{c_{K_7}}$ and $\Lipcon_{\lqr} \leq
\widetilde{c_{K_8}}$. This completes the proof.
\end{proof}


\section{Properties of the LQR problem with noisy dynamics}
\label{sec:noisy_case}

Recall that we consider the infinite horizon discounted LQR problem
where the cost function $\Cdyn(\K; \mathcal{Z})$ and the state
transition dynamics are given by
\begin{align}\label{eq:app_lqr_dyn_mod}
\begin{gathered}
  \Cdyn(\K; \mathcal{Z}) \defn \sum_{t \geq 0} \discount^t \bigg(
  \state[t]^{\top} \Q \state[t] + \control[t]^{\top} \R \control[t]
  \bigg) \\
  \state[t] = (\A - \B \K) \state[t-1] + \z_t, \quad \text{where}
  \quad \state[0] = 0 \quad \text{and} \quad \z_t
  \stackrel{i.i.d.}{\sim} \noisedistributionlqr,
\end{gathered}
\end{align}
where $\discount \in (0,1)$ denotes the discount factor. Also recall
that the distribution $\noisedistributionlqr$ has zero mean, identity
covariance, and obeys the relation $\sup \|\z_t \|_2^2 \leq
\noisebound$ almost surely.

The goal of this section is two-fold: to prove
Lemma~\ref{lem:noisy-random} that relates the cost functions $\Cinit$
and $\Cdyn$, and to establish properties of the gradient estimate in
the noisy dynamics setting required to prove
Corollary~\ref{cor:noisydyn}. In particular, our main results are
stated below, with Lemma~\ref{lem:noisy-random} reproduced for
convenience.

\begin{lemma}[Equivalence of population costs up to scaling]\label{lem:noisy-random+rep}
For any policy $\K$, we have
\begin{align*}
\Cdyn(\K) = \frac{\gamma}{1 - \gamma} \Cinit(\K).
\end{align*}
\end{lemma}

\begin{lemma} \label{lem:unifcost}
For any policy $\K$, we have the uniform bound
\begin{align*}
\Cdyn(\K; \mathcal{Z}) \leq \frac{2 (\opnorm{\Q} + \opnorm{R} \Kpolyone^2) \noisebound}{1-\sqrt{\gamma}}
\cdot \left( \frac{\Cdyn(\K)}{\sigma_{\min}(\Q)} \bigg( \frac{1 -
  \discount}{\discount} \bigg) \right)^{3/2}.
\end{align*}
\end{lemma}
\begin{comment}
\begin{lemma} \label{lem:varcost}
For any policy $\K$ and noise variables $z_t \sim \mathcal{N}(0, I)$,
we have the variance bound
\begin{equation*}
  \var({g_t}) \leq \frac{\lqrdims^2}{r^2}\left( \frac{2 m}{1-{\gamma}}
  \cdot \left( \frac{20\cdot \Cdyn(\K[0])}{\sigma_{\min}(\Q)} \bigg(
  \frac{1 - \discount}{\discount} \bigg) \right)^{3} + 800\cdot
  \Cdyn^2(\K[0])\right).
\end{equation*}
\end{lemma}
\apcomment{Should we state the above bound? Related to whether we want
  to state the remarks in the main text.}
\end{comment}

Before moving to the proofs of these lemmas, let us now define some
additional notation to facilitate the proofs.  Let
\begin{align*}
  \begin{gathered}
    \costmat \defn \Q + \K^{\top} \R \K$,\quad $\abk \defn (\A - \B
    \K)\quad\text{ and }\quad c_j \defn \discount^j \bigg(
    \sum_{i=1}^j \abk^{j-i} \z_i \bigg)^{\top} \costmat \bigg(
    \sum_{i=1}^j \abk^{j-i} \z_i \bigg).
  \end{gathered}
  \end{align*}
Also define the cumulative cost up to time $t$ by $\mathcal{C}^t =
\sum_{j = 1}^t c_j$, so that a simple computation yields the relation
$\Cdyn(\K; \mathcal{Z}) = \lim_{t \to \infty} \mathcal{C}^t$.

Additionally, define the matrix $\X_{\K, t}$ via its partition into
$t^2$ blocks $\X_{\K, t}^{i,j} \in \real^{\statedim \times \statedim}$
for each pair $(i, j) \in [t] \times [t]$, as
\begin{equation*}
\X_{\K, t} = \begin{bmatrix} \X_{\K, t}^{1,1} & \X_{\K, t}^{1,2} &
  \dots & \X_{\K, t}^{1,t} \\
\X_{\K, t}^{2,1} & \X_{\K, t}^{2,2} & \dots & \X_{\K, t}^{2,t} \\
\vdots & \vdots & \vdots & \vdots \\
  \X_{\K, t}^{t,1} & \X_{\K, t}^{t,2} & \dots & \X_{\K, t}^{t,t}
\end{bmatrix}.
\end{equation*}
Each sub-block $\X_{\K, t}^{i,j}$ of $\X_{\K, t}$ is given by
\begin{equation}
  \label{eq:form_X}
\begin{gathered}
\X_{\K, t}^{i,j} = \sum_{k=j}^{t} \discount^k (\abk^{k-i})^{\top}
\costmat \abk^{k-j} \text{ if $j \geq i$,} \\ \X_{\K, t}^{i,j} =
\sum_{k=i}^{t} \discount^k (\abk^{k-i})^{\top} \costmat \abk^{k-j}
\text{ if $j < i$}.
\end{gathered}
\end{equation}
Using this matrix notation, a simple computation yields
\begin{align*}
\mathcal{C}^t = \sum_{\substack{i \in [t] \\ j \in [t]} }\z_i ^\top
\X_{\K, t}^{i,j} \z_j.
\end{align*}
Finally, define the discounted state correlation matrix as
\begin{align*}
{\Sigma}_{\K,\discount} = \sum_{k = 0}^{\infty}(\sqrt{\discount} \A -
\sqrt{\discount} \B \K)^{k} ((\sqrt{\discount} \A - \sqrt{\discount}
\B \K)^{k})^{\top},
\end{align*}
and note that this matrix is equal to $\Sigma_{\K}$ from
equation~\eqref{eq:defn_state_corr} in
Appendix~\ref{sec:randint_appendix} with the pair of matrices $(\A,
\B)$ replaced by $(\sqrt{\discount}\A, \sqrt{\discount}\B)$. For ease in notation
define $G_{\discount} = \sqrt{\discount}G$.

The following technical lemma is required for the argument.
\begin{lemma}
\label{lem:sum_spec_g_new}
For any policy $\K$ and discount factor $\discount \in (0,1)$, we have
\begin{subequations}
  \begin{align}
    \label{eq:tech1}
\trace \left[ {\Sigma}_{\K,\discount} \right] = \trace \left[
  \sum_{k=0}^{\infty} \abk_\discount^k (\abk_\discount^k)^{\top}
  \right] &\leq \frac{\Cdyn(\K)}{\sigma_{\min}(\Q)} \bigg( \frac{1 -
  \discount}{\discount} \bigg), \\
\label{eq:tech2}
\sum_{j=0}^{\infty} \opnorm{ \abk_\discount^{j} }^2 & \leq
\frac{\Cdyn(\K)}{\sigma_{\min}(\Q)} \bigg( \frac{1 -
  \discount}{\discount} \bigg), \text{ and} \\
\label{eq:tech3}
\sum_{j=0}^{\infty} \opnorm{ \discount^j \abk^{j} } &\leq
\frac{\left(\trace \left[ \Sigma_{\K, \discount} \right]
  \right)^{1/2}}{1-\sqrt{\gamma}}.
\end{align}
\end{subequations}
\end{lemma}
\noindent See Section~\ref{sec:techlemmaproof} for the proof
of this auxiliary claim.

With this set-up, we are now equipped to prove
Lemmas~\ref{lem:noisy-random} and \ref{lem:unifcost}.


\subsection{Proof of Lemma~\ref{lem:noisy-random}}

Working with the cumulative cost, we have
\begin{align*}
\EE [\mathcal{C}^t ] & = \EE \left[ \sum_{\substack{i \in [t] \\ j \in
      [t]} } \trace( \X_{\K, t}^{i,j} \z_j \z_i^\top) \right] \\ &=
\sum_{i = 1}^t \trace\left( \X_{\K, t}^{i, i} \right),
\end{align*}
where we have used the fact that $\EE[\z_j \z_i^\top] = \mathbb{I}_{i
  = j} I$.

Substituting the definition of the matrix $\X_{\K, t}^{i, i}$, we have
\begin{align*}
\EE [\mathcal{C}^t ] &= \sum_{i=1}^{t} \trace \bigg[ \sum_{k=i}^{t}
  \discount^k (\abk^{k-i})^{\top} \costmat \abk^{k-i} \bigg] \\ &=
\sum_{i=1}^{t} \discount^i \trace \bigg[ \sum_{k=0}^{t - i}
  (\abk_\discount^{k})^{\top} \costmat \abk_\discount^{k} \bigg].
\end{align*}
Now for each fixed summand above, taking $t \to \infty$ yields
\begin{align*}
\trace \bigg[ \sum_{k=0}^{\infty} (\abk_\discount^{k})^{\top} \costmat
  \abk_\discount^{k} \bigg] &= \trace \left[ \costmat {\Sigma}_{\K,
    \discount} \right],
\end{align*}
where we have used the cyclic property of the trace.

Putting together the pieces, we have
\begin{align*}
\Cdyn(\K) &= \sum_{i=1}^{\infty} \discount^i \trace \left[ \costmat
  {\Sigma}_{\K, \discount} \right] \\
& = \bigg( \frac{\discount}{1 -
  \discount} \bigg) \trace \left[ \costmat {\Sigma}_{\K, \discount}
  \right] \\
& = \left( \frac{\gamma}{1 - \gamma}\right)\cdot
\Cinit(\K),
\end{align*}
thereby establishing Lemma~\ref{lem:noisy-random}.


\subsection{Proof of Lemma~\ref{lem:unifcost}}

As before, let us begin by analyzing the cumulative cost up to time
$t$, and write
\begin{align}
  \label{eq:twobounds}
\mathcal{C}^t \; = \; \sum_{\substack{i \in [t] \\ j \in [t]} }
\z_i^\top \X_{\K, t}^{i,j} \z_j & \stackrel{(i)}{\leq} \noisebound
\sum_{\substack{i \in [t] \\ j \in [t]} } \opnorm{ \X_{\K, t}^{i,j}}
\; = \; \noisebound \left(\sum_{i = 1}^t \sum_{j \geq i} \opnorm{
  \X_{\K, t}^{i,j}} + \sum_{j = 1}^t \sum_{i > j} \opnorm{ \X_{\K,
    t}^{i,j}} \right),
\end{align}
where in step (i), we have used the fact that $\| \z_i \|_2 \| \z_j
\|_2 \leq \noisebound$.

Bounding the first term on the RHS of equation~\eqref{eq:twobounds},
we have
\begin{align*}
\sum_{i = 1}^t \sum_{j \geq i} \opnorm{ \X_{\K, t}^{i,j}} &=
\sum_{i=1}^t \sum_{j=i}^t \opnorm{\sum_{k=j}^{t} \discount^k
  (\abk^{k-i})^{\top} \costmat \abk^{k-j}} \\
& \leq \sum_{i=1}^t \sum_{j=i}^t \opnorm{ \discount^j \abk^{j - i} }
\cdot \opnorm{\sum_{k=j}^{t} \discount^{k-j} (\abk^{k- j})^{\top}
  \costmat \abk^{k-j}} \\
& = \sum_{i=1}^t \sum_{j=i}^t \opnorm{ \discount^j \abk^{j - i} }
\cdot\opnorm{\sum_{k=0}^{t-j} (\abk_{\discount}^{k})^{\top} \costmat
  \abk_{\discount}^{k}}.
\end{align*}
By symmetry, an identical argument bounds the second term of
equation~\eqref{eq:twobounds} to yield the uniform bound
\begin{align*}
\mathcal{C}^t &\leq 2 \noisebound \sum_{i=1}^t \sum_{j=i}^t \opnorm{
  \discount^j \abk^{j - i} } \cdot\opnorm{\sum_{k=0}^{t-j}
  (\abk_{\discount}^{k})^{\top} \costmat \abk_{\discount}^{k}} \\
& \stackrel{(ii)}{\leq} 2 \noisebound \sum_{i=1}^t \sum_{j=i}^t
\opnorm{ \discount^j \abk^{j - i} } \cdot \trace \left(
\sum_{k=0}^{\infty} (\abk_{\discount}^{k})^{\top} \costmat
\abk_{\discount}^{k} \right) \\
& \stackrel{(iii)}{\leq} 2 (\opnorm{\Q} + \opnorm{R} \Kpolyone^2) \noisebound \cdot \left( \frac{\left(\trace
  \left[ \Sigma_{\K, \discount} \right]
  \right)^{1/2}}{1-\sqrt{\gamma}} \right) \cdot \left(
\frac{\Cdyn(\K)}{\sigma_{\min}(\Q)} \bigg( \frac{1 -
  \discount}{\discount} \bigg) \right) \\
& \stackrel{(iv)}{\leq} \frac{2 (\opnorm{\Q} + \opnorm{R} \Kpolyone^2) \noisebound}{1-\sqrt{\gamma}} \cdot
\left( \frac{\Cdyn(\K)}{\sigma_{\min}(\Q)} \bigg( \frac{1 -
  \discount}{\discount} \bigg) \right)^{3/2},
\end{align*}
where in step (ii), we have used the PSD nature of the matrices being
summed, and steps (iii) and (iv) follow from
inequalities~\eqref{eq:tech1} and~\eqref{eq:tech3} of
Lemma~\ref{lem:sum_spec_g_new}, respectively.  Since the above
relation holds for all $t$, we can take the limit $t \rightarrow
+\infty$ on the left-hand side so as to obtain the claim of
Lemma~\ref{lem:unifcost}.


\subsection{Proof of Lemma~\ref{lem:sum_spec_g_new}} \label{sec:techlemmaproof}

In this section we prove the auxiliary bounds~\eqref{eq:tech1}
through~\eqref{eq:tech3}.

\paragraph{Proof of the bound \eqref{eq:tech1}:} Following
the proof of Lemma~\ref{lem:noisy-random}, we have
\begin{align*}
\Cdyn(\K) & = \bigg( \frac{\discount}{1 - \discount} \bigg) \trace
\bigg[ \costmat \Sigma_{\K, \discount} \bigg] \\
& = \bigg( \frac{\discount}{1 - \discount} \bigg) \trace \bigg[ (\Q +
  \K^{\top} \R \K) \Sigma_{\K, \discount} \bigg] \\
& \stackrel{\1}{\geq} \bigg( \frac{\discount}{1 - \discount} \bigg)
\sigma_{\min}(\Q) \trace ( \Sigma_{\K, \discount} ),
\end{align*}
where $\1$ follows from Von Neumann's trace inequality. Multiplying
both sides above by $\frac{1-\discount}{\discount \cdot
  \sigma_{\min}(\Q)}$ completes the proof.

\paragraph{Proof of the bound \eqref{eq:tech2}:}
Observe that for any $j$, there exists some unit vector $v_j$ such
that $\opnorm{ \abk_\discount^{j} } = \opnorm{ \abk_\discount^{j}
  v_j}$.  Using this fact, we have
\begin{align*}
\sum_{j=0}^{\infty} \opnorm{ \abk_\discount^{j} }^2 =
\sum_{j=0}^{\infty} \opnorm{ \abk_\discount^{j} v_j }^2 & =
\sum_{j=0}^{\infty} \trace \left[ (\abk_\discount^{j})^{\top}
  \abk_\discount^{j} v_j v_j^{\top} \right] \\
& \stackrel{\1}{\leq} \sum_{j=0}^{\infty} \trace \left[
  (\abk_\discount^{j})^{\top} \abk_\discount^{j} \right] \cdot
\opnorm{ v_j v_j^{\top} } \\
& \stackrel{\2}{=} \trace \left[ \Sigma_{\K, \discount} \right]
\end{align*}
where step \1 follows from Von Neumann's trace inequality and \2
follows from the definition of $\Sigma_{\K, \discount}$.  Applying the
bound from equation~\eqref{eq:tech1} completes the proof.


\paragraph{Proof of the bound \eqref{eq:tech3}:} Similar to
the proof of \eqref{eq:tech2}, observe that,
\begin{align*}
\sum_{j=0}^{\infty} \opnorm{ \discount^j \abk^{j} } &=
\sum_{j=0}^{\infty} \discount^{j/2}\left(\trace \left[
  (\abk_\discount^{j})^{\top} \abk_\discount^{j} v_j v_j^{\top}
  \right]\right)^{1/2} \\
& \leq \sum_{j=0}^{\infty} \discount^{j/2}\left(\trace \left[
  (\abk_\discount^{j})^{\top} \abk_\discount^{j} \right]\right)^{1/2}
\\
& \stackrel{\1}{\leq} \sum_{j=0}^{\infty} \discount^{j/2} \left(\trace
\left[ \Sigma_{\K, \discount} \right]
\right)^{1/2} \\
& = \frac{\left(\trace \left[ \Sigma_{\K, \discount} \right]
  \right)^{1/2}}{1-\sqrt{\gamma}},
\end{align*}
where step \1 follows from using $\trace \left[
  (\abk_\discount^{j})^{\top} \abk_\discount^{j} \right] \leq \sum_{j
  = 0}^{\infty}\trace \left[ (\abk_\discount^{j})^{\top}
  \abk_\discount^{j} \right] = \trace \left[ \Sigma_{\K, \discount}
  \right]$.


\section{Proof of Lemma~\ref{lem:Thm1_lemma}}
\label{AppSmoothing}

We now provide the proof of Lemma~\ref{lem:Thm1_lemma}, splitting our
analysis into the two separate claims.

\subsection{Proof of part (a)}
Unwrapping the definition of $\grad \fr(\x)$ yields
\begin{align*}
\grad \fr{(\x)} &\stackrel{(i)}{=} \frac{\dims}{\smoothingradius}
\mathbb{E}[\f(\x + \smoothingradius \unifdir) \unifdir] \\
& = \frac{\dims}{2\smoothingradius} (\mathbb{E}[\f(\x +
    \smoothingradius \unifdir) \unifdir] + \mathbb{E}[\f(\x +
    \smoothingradius \unifdir) \unifdir] ) \\
& \stackrel{(ii)}{=} \frac{\dims}{2\smoothingradius}
(\mathbb{E}[\f (\x + \smoothingradius \unifdir) \unifdir] -
\mathbb{E}[\f (\x - \smoothingradius \unifdir) \unifdir] ) \\
& = \frac{\dims}{2\smoothingradius} \mathbb{E}[\f (\x +
    \smoothingradius \unifdir) \unifdir -\f (\x - \smoothingradius
    \unifdir) \unifdir],
\end{align*}
where equality (i) follows from Lemma 1 in~\cite{flax04}, and equality (ii) follows from the symmetry of the
uniform distribution on the shell $\Shell^{\dims - 1}$. Now observe
that
\begin{align*}
\mathbb{E}[\F (\x + \smoothingradius \unifdir, \factorvar) \unifdir -
  \F (\x - \smoothingradius \unifdir, \factorvar) \unifdir] & =
\mathbb{E}\bigg[ \mathbb{E}[\F(\x + \smoothingradius \unifdir,
      \factorvar) - \F(\x - \smoothingradius \unifdir, \factorvar)
    \unifdir \vert \unifdir] \bigg] \\
& \stackrel{(i)}{=} \mathbb{E}\bigg[ \f(\x + \smoothingradius
    \unifdir) \unifdir - \f(\x - \smoothingradius \unifdir)
  \unifdir \bigg],
\end{align*}
where equality (i) follows from the assumption that $\f(\x) =
\mathbb{E}_{\factorvar \sim \factorvardistribution} [\F(\x,
  \factorvar)]$. Putting the equations together establishes the claim
in part (a).  \qed


\paragraph{Proof of Lemma~\ref{lem:Thm1_lemma}, part (b)}
Observe that
\begin{align*}
\euclidnorm{\grad \fr{(\x)} - \grad \f(\x)} &= \euclidnorm{\grad
  \mathbb{E}[\f(\x + \smoothingradius \unifdirball)] - \grad
  \f(\x)} \\
& = \euclidnorm{ \mathbb{E}[\grad [\f(\x + \smoothingradius
        \unifdirball) - \grad \f(\x) ] } \\
& \stackrel{(i)}{\leq} \mathbb{E} [\euclidnorm{ \grad [\f(\x +
          \smoothingradius \unifdirball) - \grad \f(\x) }] \\
& \stackrel{(ii)}{\leq} \globalSmooth{0} \smoothingradius,
\end{align*}
where inequality (i) above follows from Jensen's inequality, whereas
step (ii) follows since $\smoothingradius \leq \rho$ and $\grad \f$ is locally Lipschitz continuous with
parameter $\globalSmooth{0}$.
\qed

\section{Experimental Details \& Additional Experiments}
\label{sec:additional_exp}
For each LQR problem used, the initial $\Kzero$ was picked by randomly perturbing the entries of $\Kstar$. The step size was tuned manually and the smoothing radius was always chosen to be the minimum of $\sqrt{\epsilon}$ and the largest value required to ensure stability. The rollout length was also tuned manually until the cost from a rollout converged arbitrarily close to the true value.

\subsection{Details of Experiments from Section~\ref{sec:results2}}
To generate the plot in Figure~\ref{fig:plt_minibatch} (a), we used the following one dimensional LQR problem:
\begin{align*}
\A = 5, \quad \B = 0.33, \quad \Q = 1, \quad \R = 1,
\end{align*}
where we operated in the one-point random initialization setting, the
initial state was sampled uniformly at random from the set $\{ 4, 5, 6
\}$, and the discount factor was set to $1$.

To generate the plots in Figure~\ref{fig:plt_minibatch} (b), Figure~\ref{fig:plt_eps} (b) and Figure~\ref{fig:plt_init} (a), we used the following LQR problem:
\begin{align*}
\A = \begin{bmatrix} 1 & 0 & -10 \\ -1 & 1 & 0 \\ 0 & 0 & 1
\end{bmatrix}, \quad
\B = \begin{bmatrix} 1 & -10 & 0 \\ 0 & 1 & 0 \\ -1 & 0 & 1
\end{bmatrix}, \quad
\Q = \begin{bmatrix} 2 & -1 & 0 \\ -1 & 2 & -1 \\ 0 & -1 & 2
\end{bmatrix}, \quad
\R = \begin{bmatrix} 5 & -3 & 0 \\ -3 & 5 & -2 \\ 0 & -2 & 5
\end{bmatrix},
\end{align*}
where we operated in the two-point random initialization setting, the
initial state was sampled uniformly at random from the canonical basis
vectors, and the discount factor was set to $1$.

To generate the plots in Figure~\ref{fig:plt_eps} (a) and ~\ref{fig:plt_eps} (c), we used the following LQR problem:
\begin{align*}
\A = 0.1 \times I \quad \B = 0.01 \times I \quad \Q = 100 \times I \quad \R = 100 \times I,
\end{align*}
where $I$ represents the $3 \times 3$ identity matrix. For Figure~\ref{fig:plt_eps} (a) we operated in the random initialization setting, and used initial states which were sampled uniformly at random from the rows of the matrix $\frac{\sqrt{3}}{25} \times I$. For Figure~\ref{fig:plt_eps} (c), we operated in the one-point additive noise setting. Here the initial state was set to the zero vector, and we used additive noise at each timestep sampled from a zero mean Gaussian with covariance matrix $\frac{1}{25} \times I$. In both settings, the discount factor was set to $0.9$. For this example, the population level costs in the two settings are equal up to a constant scaling factor.

To generate the plot in Figure~\ref{fig:plt_init} (b), we used the following LQR problem:
\begin{align*}
\A = 0.1 \times I \quad \B = 0.01 \times I \quad \Q = 25 \times I \quad \R = 25 \times I,
\end{align*}
where $I$ represents the $3 \times 3$ identity matrix. We operated in the one-point additive noise setting. The initial state was set to the zero vector, and we used additive noise at each timestep sampled from a zero mean Gaussian with covariance matrix $\frac{1}{25} \times I$. The discount factor was set to $0.9$.

\subsection{Additional Experiments}
In the two point random initialization setting, we performed experiments on several additional LQR instances to test the robustness of the behavior observed in Figures~\ref{fig:plt_minibatch} and~\ref{fig:plt_init}. For ease in notation, we use $\Cfunc$ to denote the population cost for the remainder of this section. Note that for all figures shown in this section, each dotted line represents the line of best fit for its corresponding data points, as in Figures~\ref{fig:plt_eps} and~\ref{fig:plt_init}. Using the same example used to generate the plots in Figure~\ref{fig:plt_eps} (b) and Figure~\ref{fig:plt_init} (a), we tested the performance of our two-point algorithm with different values of $\accuracy$ and $\C{\Kzero}$.


\pgfplotsset{width=5.5cm,compat=1.9}
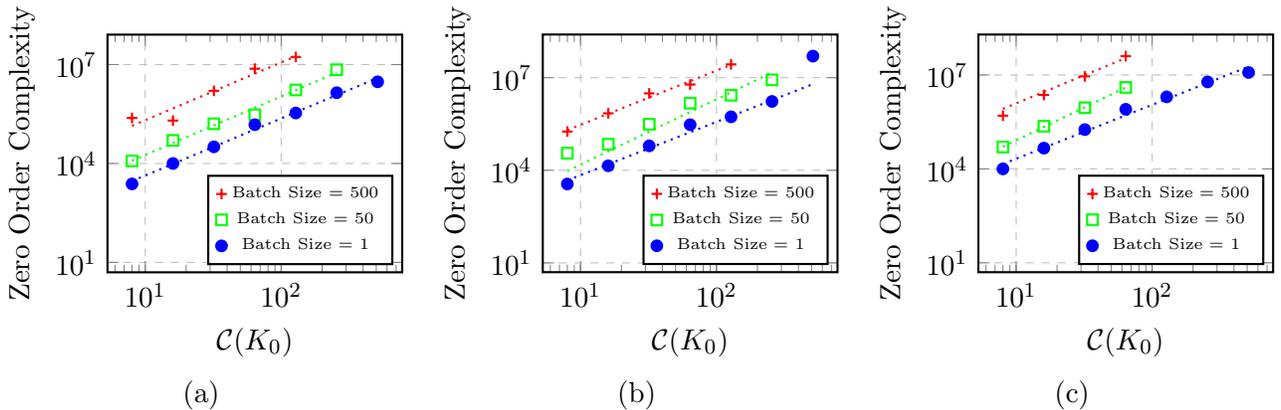
\begin{figure}[H]
\centering \hspace*{-4ex}
\begin{tabular}{ccc}
\begin{tikzpicture}
\begin{loglogaxis}[
    xlabel={$\C{\Kzero}$},
    ylabel={Zero Order Complexity},
    legend style={legend pos=south east,font=\tiny},
    log base y={10},
    log base x={10},
    ymin=5,
    xmajorgrids=true,
    ymajorgrids=true,
    grid style=dashed,
    style=thick
]
  \addplot[
  color = red,
  dotted,
  forget plot
  ]
  coordinates { 
  (2^3, 139082.649572)
  (2^7, 17909176.0793)
  };
\addplot[ 
    color=red,
    smooth,
    only marks=True,
    mark=+,
    ]
    coordinates {
    (8, 500*1200*2/5)
    (16, 500*5000*2/25)
    (32, 500*16000*2/10)
    (64, 500*75000*2/10)
    (128, 500*170000*2/10)
    };
 \addplot[
  color = green,
  dotted,
  forget plot
  ]
  coordinates { 
  (2^3, 12217.6684985)
  (2^8, 5727847.92812)
  };
\addplot[ 
    color=green,
    smooth,
    only marks=True,
    mark=square,
    ]
    coordinates {
    (8, 50*1200*2/10)
    (16, 50*5000*2/10)
    (32, 50*16000*2/10)
    (64, 50*75000*2/25)
    (128, 50*170000*2/10)
    (256, 50*700000*2/10)
    };
     \addplot[
  color = blue,
  dotted,
  forget plot
  ]
  coordinates { 
  (2^3, 2921.45499207)
  (2^9, 3946087.49519)
  };
\addplot[ 
    color=blue,
    smooth,
    only marks=True,
    mark=*,
    ]
    coordinates {
    (8, 1200*2)
    (16, 5000*2)
    (32, 16000*2)
    (64, 75000*2)
    (128, 170000*2)
    (256, 700000*2)
    (512, 1500000*2)
    };
    \legend{Batch Size = $500$,
    Batch Size = $50$,
    Batch Size = $1$,
    }
\end{loglogaxis}
\end{tikzpicture}&

\begin{tikzpicture}
\begin{loglogaxis}[
    xlabel={$\C{\Kzero}$},
    ylabel={Zero Order Complexity},
    legend style={legend pos=south east,font=\tiny},
    log base y={10},
    log base x={10},
    ymin=5,
    xmajorgrids=true,
    ymajorgrids=true,
    grid style=dashed,
    style=thick
]
     \addplot[
  color = red,
  dotted,
  forget plot
  ]
  coordinates { 
  (2^3, 201002.891524)
  (2^7, 26139459.091)
  };
\addplot[ 
    color=red,
    smooth,
    only marks=True,
    mark=+,
    ]
    coordinates {
    (8, 500*1800*2/10)
    (16, 500*7000*2/10)
    (32, 500*31000*2/10)
    (64, 500*150000*2/25)
    (128, 500*270000*2/10)
    };
         \addplot[
  color = green,
  dotted,
  forget plot
  ]
  coordinates { 
  (2^3, 9508.19882644)
  (2^8, 14625603.8687)
  };
    \addplot[ 
    color=green,
    smooth,
    only marks=True,
    mark=square,
    ]
    coordinates {
    (8, 50*1800*2/5)
    (16, 50*7000*2/10)
    (32, 50*31000*2/10)
    (64, 50*150000*2/10)
    (128, 50*270000*2/10)
    (256, 50*850000*2/10)
    };
     \addplot[
  color = blue,
  dotted,
  forget plot
  ]
  coordinates { 
  (2^3, 4741.28339391)
  (2^9, 6178621.07588)
  };
    \addplot[ 
    color=blue,
    smooth,
    only marks=True,
    mark=*,
    ]
    coordinates {
    (8, 1800*2)
    (16, 7000*2)
    (32, 31000*2)
    (64, 150000*2)
    (128, 270000*2)
    (256, 850000*2)
    (512, 25000000*2)
    };
    \legend{Batch Size = $500$,
    Batch Size = $50$,
    Batch Size = $1$,
    }
    \end{loglogaxis}
    \end{tikzpicture}&
    \begin{tikzpicture}
\begin{loglogaxis}[
    xlabel={$\C{\Kzero}$},
    ylabel={Zero Order Complexity},
    legend style={legend pos=south east,font=\tiny},
    log base y={10},
    log base x={10},
    ymin=5,
    xmajorgrids=true,
    ymajorgrids=true,
    grid style=dashed,
    style=thick
]
     \addplot[
  color = red,
  dotted,
  forget plot
  ]
  coordinates { 
  (2^3, 831195.889961)
  (2^6, 34618781.7907)
  };
\addplot[ 
    color=red,
    smooth,
    only marks=True,
    mark=+,
    ]
    coordinates {
    (8, 500*5000*2/10)
    (16, 500*23000*2/10)
    (32, 500*90000*2/10)
    (64, 500*400000*2/10)
    };
      \addplot[
  color = green,
  dotted,
  forget plot
  ]
  coordinates { 
  (2^3, 51166.108816)
  (2^6, 3976653.7695)
  };
    \addplot[ 
    color=green,
    smooth,
    only marks=True,
    mark=square,
    ]
    coordinates {
    (8, 50*5000*2/10)
    (16, 50*23000*2/10)
    (32, 50*90000*2/10)
    (64, 50*400000*2/10)
    };
         \addplot[
  color = blue,
  dotted,
  forget plot
  ]
  coordinates { 
  (2^3, 14330.9118205)
  (2^9, 18468667.0597)
  };
    \addplot[ 
    color=blue,
    smooth,
    only marks=True,
    mark=*,
    ]
    coordinates {
    (8, 5000*2)
    (16, 23000*2)
    (32, 90000*2)
    (64, 400000*2)
    (128, 1000000*2)
    (256, 3000000*2)
    (512, 6000000*2)
    };
    \legend{Batch Size = $500$,
    Batch Size = $50$,
    Batch Size = $1$,
    }
    \end{loglogaxis}
    \end{tikzpicture}
\\
(a)&(b)&(c)
\end{tabular}
\caption{Scaling of complexity vs. $\C{\Kzero}$ while using minibatches of size 1, 50 and 500, to achieve an error tolerance of (a) $\accuracy =  0.1$, (b) $\accuracy = 0.05$ and (c) $\accuracy = 0.01$. Due to the prohibitive complexity when using batches of size $50$ and $500$, we omit data points for large values of $\C{\Kzero}$.}
\label{fig:plt_minibatch_app}
\end{figure}
In Figure~\ref{fig:plt_minibatch_app} (a) (b) and (c), we plot the
scaling of the zero-order complexity with $\Cfunc(\K[0])$ for
different values of the tolerance $\epsilon$, and each figure
additionally contains plots for different values of the batch-size. We
observe that the scaling of our algorithm with respect to $\C{\Kzero}$
is approximately on the order of $\mathcal{O}(\C{\Kzero}^2)$,
suggesting that our bounds for the Lipschitz and smoothness constants
are not sharp in this respect. The same plots also demonstrate that
using larger batch sizes is often suboptimal: while the step size can
be increased with increasing batch-size, it eventually plateaus due to
stability considerations, leading to higher overall zero-order
complexity.

We also ran our algorithm
on the following problem introduced by~\cite{dean17}, who
used this example in their study of model based control methods for
the LQR problem. Consider the LQR problem defined by:
\begin{align*}
\A = \begin{bmatrix} 1.01 & 0.01 & 0 \\ 0.01 & 1.01 & 0.01 \\ 0 & 0.01
  & 1.01
\end{bmatrix}, \quad
\B = \identity, \quad \Q = 10^{-3} \times \identity, \quad \R =
\identity.
\end{align*}

For three different values of $\C{\Kzero}$, we picked $8$ evenly
spaced (logarithmic scale) values of $\accuracy$ in the interval
$(0.005, 1)$. The initial state was sampled uniformly at random from $\{ [5,0,0], [5,5,5], [0,0,5] \}$. The cost of the optimal policy in our example was
$\C{\Kstar} = 2.36$. We then measured the total zero order complexity
required to attain $\accuracy$ convergence. These results are plotted
in Figure \ref{fig:epsilon_scaling_dean}.

\begin{minipage}{.5\textwidth}
\begin{figure}[H]
\centering
\begin{tikzpicture}
\begin{loglogaxis}[
    xlabel={$\accuracy^{-1}$},
    ylabel={Zero Order Complexity},
    legend style={legend pos=south east,font=\tiny},
    log base y={10},
    log base x={10},
    ymin=3,
    xmajorgrids=true,
    ymajorgrids=true,
    grid style=dashed,
    style=thick
]
  \addplot[
  color = red,
  dotted,
  forget plot
  ]
  coordinates { 
  (1, 23235.7148315)
  (10^2.301029, 275515.77063)
  };
\addplot[
    color=red,
    smooth,
    only marks=True,
    mark=+,
    ]
    coordinates {
    (1, 8000*2)
    (2.13164208, 15000*2)
    (4.54389797, 30000*2)
    (9.68596415, 45000*2)
    (20.6470088, 60000*2)
    (44.01203286, 75000*2)
    (93.81790145, 90000*2)
    (199.98618696, 100000*2)
    };
  \addplot[
  color = green,
  dotted,
  forget plot
  ]
  coordinates { 
  (1, 10532.8666446)
  (10^2.301029, 86131.2729393)
  };
\addplot[
    color=green,
    smooth,
    only marks=True,
    mark=square,
    ]
    coordinates {
    (1, 3500*2)
    (2.13164208, 8000*2)
    (4.54389797, 12000*2)
    (9.68596415, 15000*2)
    (20.6470088, 20000*2)
    (44.01203286, 25000*2)
    (93.81790145, 30000*2)
    (199.98618696, 35000*2)
    };
  \addplot[
  color = blue,
  dotted,
  forget plot
  ]
  coordinates { 
  (1, 386.234413188)
  (10^2.301029, 62817.9293529)
  };
\addplot[
    color=blue,
    smooth,
    only marks=True,
    mark=*,
    ]
    coordinates {
    (1, 125)
    (2.13164208, 400*2)
    (4.54389797, 700*2)
    (9.68596415, 2500*2)
    (20.6470088, 6500*2)
    (44.01203286, 10000*2)
    (93.81790145, 14000*2)
    (199.98618696, 17000*2)
    };
    \legend{$\C{\Kzero} = 33$,
    $\C{\Kzero} = 13$,
    $\C{\Kzero} = 5.8$
    }
\end{loglogaxis}
\end{tikzpicture}
\caption{Scaling of complexity vs. $\accuracy^{-1}$ in LQR instance from~\cite{dean17}}
\label{fig:epsilon_scaling_dean}
\end{figure}
\end{minipage}
\begin{minipage}{.5\textwidth}
\begin{figure}[H]
\centering
\begin{tikzpicture}
\begin{loglogaxis}[
    xlabel={$\accuracy^{-1}$},
    ylabel={Zero Order Complexity},
    legend style={legend pos=south east,font=\tiny},
    log base y={10},
    log base x={10},
    ymin=2500,
    ymax=1300000,
    xmajorgrids=true,
    ymajorgrids=true,
    grid style=dashed,
    style=thick
]
  \addplot[
  color = red,
  dotted,
  forget plot
  ]
  coordinates { 
  (1, 17087.0116362)
  (10^2.301029, 1587865.93985)
  };
\addplot[
    color=red,
    smooth,
    only marks=True,
    mark=+,
    ]
    coordinates {
    (1, 6000*2)
    (2.13164208, 15000*2)
    (4.54389797, 35000*2)
    (9.68596415, 80000*2)
    (20.6470088, 140000*2)
    (44.01203286, 300000*2)
    (93.81790145, 400000*2)
    (199.98618696, 500000*2)
    };
  \addplot[
  color = green,
  dotted,
  forget plot
  ]
  coordinates { 
  (1, 13328.4685919)
  (10^2.301029, 1348141.99952)
  };
\addplot[
    color=green,
    smooth,
    only marks=True,
    mark=square,
    ]
    coordinates {
    (1, 4500*2)
    (2.13164208, 12500*2)
    (4.54389797, 27000*2)
    (9.68596415, 60000*2)
    (20.6470088, 125000*2)
    (44.01203286, 275000*2)
    (93.81790145, 325000*2)
    (199.98618696, 400000*2)
    };
  \addplot[
  color = blue,
  dotted,
  forget plot
  ]
  coordinates { 
  (1, 8514.19604245)
  (10^2.301029, 659702.165788)
  };
\addplot[
    color=blue,
    smooth,
    only marks=True,
    mark=*,
    ]
    coordinates {
    (1, 2500*2)
    (2.13164208, 8000*2)
    (4.54389797, 18000*2)
    (9.68596415, 36000*2)
    (20.6470088, 75000*2)
    (44.01203286, 125000*2)
    (93.81790145, 160000*2)
    (199.98618696, 200000*2)
    };
    \legend{$\C{\Kzero} = 250$,
    $\C{\Kzero} = 200$,
    $\C{\Kzero} = 150$,
    }
\end{loglogaxis}
\end{tikzpicture}
\caption{Scaling of complexity vs. $\accuracy^{-1}$ in randomly generated $8 \times 8$ example.}
\label{fig:epsilon_scaling_big}
\end{figure}
\end{minipage}
\\ \\

Finally, we also obtained data for the scaling with respect to
$\accuracy$ on an example in slightly higher dimensions, to
empirically verify the fact that our algorithm can be used for LQR
problems larger than $3 \times 3$. We randomly generated $\A$, $\B$,
$\Q$ and $\R$ as $8 \times 8$ matrices. Each entry of $\A$ was independently sampled from the Gaussian distribution $\mathcal{N}(2,1)$, and each entry of $\B$ was independently sampled from the Gaussian distribution $\mathcal{N}(0,1)$. To generate each of $\Q$ and $\R$, we generated a matrix where each entry was independently sampled from the Gaussian distribution $\mathcal{N}(5,1)$, then symmetrized the matrix by adding it to its transpose, finally adding $10I$ to ensure positive definiteness. The initial states
were sampled uniformly at random from the columns of the $8 \times 8$
identity matrix. For three different values of
$\C{\Kzero}$, we picked $8$ evenly spaced (logarithmic scale) values
of $\accuracy$ in the interval $(0.005, 1)$. We then measured the
total zero order complexity required to attain $\accuracy$
convergence. These results are plotted in
Figure~\ref{fig:epsilon_scaling_big}.

\section{Improved analysis of minibatching algorithm from \\ Fazel et al.~\cite{kakade18}}
\label{app:fazel}

At the suggestion of an anonymous reviewer, we now use our techniques
to analyze the minibatching algorithm that was proposed
by Fazel et al.~\cite{kakade18}.  In this analysis, each ``iteration'' involves
averaging a large number of one-point zero-order evaluations in order
to obtain a (low-variance) estimate of the gradient at that point,
followed by taking a step along the estimated gradient.

More precisely, for a given point $x$, consider the $k$-sample
minibatched gradient estimate
\begin{align}
  \label{EqnMinibatchGradient}
\gradest(\x) = \frac{1}{k}\sum_{i=1}^k \gradest_i(\x),
\end{align}
where each $\gradest_i$ is an i.i.d. copy of the random variable
$\gradest_r^1(\x, \shelldirection, \factorvar)$ defined in
equation~\eqref{eq:gradonepoint}. This introduces yet another
hyperparameter $k$ within the procedure, in addition to the tuple
$(\smoothingradius, \step, T)$. Note that the total number of
zero-order evaluations made by this algorithm when run for $T$
iterations is $k \cdot T$.

The following theorem holds under the same setup as in
Section~\ref{sec:conv-guar}.
\begin{theorem}
\label{thm:minibatch_gd}
Given an error tolerance $\epsilon$ in the interval $\big(0, \; \min
\{1, \frac{1}{\mu}, \rho_0^2 \} \frac{\diff{0}}{10} \big)$, suppose
that the step size and smoothing radius are chosen such that
\begin{align*}
\step \leq \min \left\{ 1, \frac{1}{8 \globalSmooth{0}},
\frac{\rho_0}{\frac{\sqrt{\mu}}{32} + \phi_0 + \lambda_0} \right\},
\quad \text{ and} \quad \smoothingradius \leq \frac{1}{8
  \globalSmooth{0}} \min\left \{ \curvature_{0} \Plconst
\sqrt{\frac{\epsilon}{240}},
\frac{1}{\globalSmooth{0}}\sqrt{\frac{\epsilon \Plconst}{30} }
\right\},
\end{align*}
and we use the minibatch size $k = \bigg(
\frac{D}{\smoothingradius}(10 \f(\x[0]) + \frac{\lambda_0}{\rho_0})
\sqrt{\log(\frac{2\lqrdims}{\delta})} \bigg)^2 \frac{1024}{\mu
  \epsilon}$.  Then running the algorithm for $T = \frac{8}{\eta
  \Plconst} \log(\frac{2}{\epsilon})$ iterations yields an output
$\x[T]$ such that
\begin{align*}
\f(\x[T]) - \f^* \leq \epsilon
\end{align*}
with probability at least $1 - T\delta$.
\end{theorem}
Note that, as before, we require a smoothing radius $\smoothingradius
\sim \sqrt{\epsilon}$, but now, the step-size can be chosen to be an
$\epsilon$-independent constant. The number of zero-order evaluations
needed to obtain an $\epsilon$-approximate solution with probability
$\tfrac{3}{4}$---ignoring parameters not dependent on $\epsilon$---is
then given by
\begin{align*}
k \cdot T \sim \epsilon^{-2} \log (1 / \epsilon) \cdot \log \log (1 /
\epsilon),
\end{align*}
where the doubly logarithmic term arises from setting $\delta \sim
T^{-1}$. Such a guarantee is thus essentially the same as that
provided by Theorem~\ref{thm:mainthm} for the high-variance zero-order
algorithm.

In this setting, it is straightforward to obtain a high probability
guarantee. Indeed, suppose that we set $\delta = T^{-1} \delta'$ for
some $\delta' \in (0, 1)$, and note that the number of zero-order
evaluations required to obtain an $\epsilon$-approximate solution with
probability $1 - \delta'$ is of the order $\epsilon^{-2} \log (1 /
\epsilon) \cdot \log \log (1 / \epsilon) \cdot \log (1 / \delta')$. As
will be clear from the proof, this is a consequence of the fact that
stability---meaning that the algorithm stays within the bounded set
$\boundedset{0}$---åcan be guaranteed with exponentially high
probability.  Such a guarantee was not possible for the high-variance
analogue of the algorithm.

As a corollary of Theorem~\ref{thm:minibatch_gd}, we have the
following guarantee on LQR control with one-point feedback.
\begin{corollary}
\label{cor:fazel}
Given an error tolerance $\epsilon$ in the interval $\big(0, \; \min
\{1, \frac{1}{\mu_{\lqr}}, \rho_{\lqr}^2 \} \frac{\diff{0}}{10}
\big)$, suppose that the step size and smoothing radius are chosen
such that
\begin{align*}
\step \leq \min \left\{ 1, \frac{1}{8 \globalSmooth{\lqr}},
\frac{\rho_{\lqr}}{\frac{\sqrt{\mu_{\lqr}}}{32} + \phi_{\lqr} +
  \lambda_{\lqr}} \right\}, \quad \text{ and} \quad \smoothingradius
\leq \frac{1}{8 \globalSmooth{\lqr}} \min\left \{ \curvature_{\lqr}
\Plconst_{\lqr} \sqrt{\frac{\epsilon}{240}},
\frac{1}{\globalSmooth{\lqr}}\sqrt{\frac{\epsilon \Plconst_{\lqr}}{30}
} \right\},
\end{align*}
and that we use a minibatch size $k = \bigg(
\frac{D}{\smoothingradius}(10 \Cinit(\K[0]) +
\frac{\lambda_{\lqr}}{\rho_{\lqr}})
\sqrt{\log(\frac{2\lqrdims}{\delta})} \bigg)^2 \frac{1024}{\mu_{\lqr}
  \epsilon}$.  Then running the algorithm for $T = \frac{8}{\eta
  \Plconst_{\lqr}} \log(\frac{2}{\epsilon})$ iterations yields an
estimate $\K[T]$ such that
\begin{align*}
\Cinit(\K[T]) - \Cinit(\Kstar) \leq \epsilon
\end{align*}
with probability exceeding $1 - T\delta$.
\end{corollary}
Thus, the algorithm of Fazel et al.~\cite{kakade18} also enjoys the same
$\ordertil{\epsilon^{-2}}$ convergence rate---measured in the number
of total zero-order evaluations---as the canonical zero-order
algorithm. At this juncture, we stress that this is a consequence of
the sharpened bounds that we establish for this problem; the analysis
of Fazel et al.---as mentioned in footnote 2---only certifies an
$\ordertil{\epsilon^{-4}}$ convergence rate.

The corollary is an immediate consequence of the theorem. We therefore
dedicate the rest of this section to a proof of
Theorem~\ref{thm:minibatch_gd}.


\subsection{Proof of Theorem~\ref{thm:minibatch_gd}}

We begin with an elementary lemma that guarantees exponential
concentration of the averaged gradient estimate $\gradest$ around its
mean.

\begin{lemma}
  \label{lem:minibatch_grad_est}
  For any $\smoothingradius \in (0, \rho_0)$, the $k$-sample minibatch
  gradient estimate~\eqref{EqnMinibatchGradient} satisfies the bound
\begin{align*}
  \enorm{\gradest(\x) - \grad \fr(\x)} \leq \frac{1}{\sqrt{k}} \cdot
  \frac{\lqrdims}{\smoothingradius}
  \left(\f(\x)+\frac{\lambda_0}{\rho_0}\right) \sqrt{\log
    \left(\frac{2\lqrdims}{\delta} \right)}
\end{align*}
with probability at least $1 - \delta$.
\end{lemma}
\begin{proof}
  This lemma is an immediate application of Corollary 7 in Jin et al.~\cite{jin19} on
  concentration for i.i.d. bounded random vectors. To verify the
  required assumptions, note that for a value of smoothing radius
  $\smoothingradius \leq \rho_0$, we have $\f(\x + \smoothingradius
  \unifdir) \leq \f(\x) + \frac{\lambda_0}{\rho_0}$ by the
  locally-Lipschitz property of the function. So each gradient
  estimate satisfies the bound
\begin{align*}
\enorm{\gradest_i(\x)} = \enorm{\frac{\lqrdims}{\smoothingradius}
  \f(\x + \smoothingradius \unifdir_i) \unifdir} \leq
\frac{\lqrdims}{\smoothingradius} (\f(\x) + \frac{\lambda_0}{\rho_0})
\end{align*}
almost surely, thus satisfying the norm sub-Gaussian condition
discussed in Jin et al.~\cite{jin19}. In addition, applying part (a) of
Lemma~\ref{lem:Thm1_lemma} yields $\mathbb{E}[\gradest_i(\x)] = \grad
\fr(\x)$.  Applying Corollary 7 of Jin et al.~\cite{jin19} then yields the claim.
\end{proof}

We are now ready to prove Theorem~\ref{thm:minibatch_gd}. First,
recall the notation $\diff{t} = \f(\x[t]) - \f^*$ and assume that the
point $\x$ satisfies $\f(\x) - \f^* \leq 10 \diff{0}$.  Suppose that
we use a minibatch of size $k = \bigg(
\frac{D}{\smoothingradius}(\f(\x[t])) + \frac{\lambda_0}{\rho_0})
\sqrt{\log(\frac{2\lqrdims}{\delta})} \bigg)^2 \frac{1024}{\mu
  \epsilon}$ to estimate the gradient.
Lemma~\ref{lem:minibatch_grad_est} then ensures that
\begin{equation}
\label{eqn:minibatching_bound}
\enorm{\gradest(\x) - \grad \fr(\x)} \leq \frac{\sqrt{\mu \epsilon}}{32}
\end{equation}
 with probability $1 - \delta$.  Conditioned on this event, we have
 the following sequence of bounds
\begin{align*}
\enorm{\eta \gradest(\x)} & = \eta \enorm{\gradest(\x) - \grad \fr(\x)
  + \grad \fr(\x) - \grad \f(\x) + \grad \f(\x)} \\ &\leq \eta
\enorm{\gradest(\x) - \grad \fr(\x)} + \eta \enorm{\grad \fr(\x) -
  \grad \f(\x)} + \eta \enorm{\grad \f(\x)} \\ &\stackrel{(i)}{\leq}
\frac{\sqrt{\mu \epsilon}}{32} + \eta \enorm{\grad \fr(\x) - \grad
  \f(\x)} + \eta \enorm{\grad \f(\x)} \\
& \stackrel{(ii)}{\leq} \eta \left(\frac{\sqrt{\mu \epsilon}}{32} +
\phi_0 \sqrt{\epsilon} + \lambda_0 \right) \\ &\stackrel{(iii)}{\leq}
\eta \left(\frac{\sqrt{\mu}}{32} + \phi_0 + \lambda_0 \right),
\end{align*}
where step (i) follows from equation~\eqref{eqn:minibatching_bound},
step (ii) follows from part (b) of Lemma~\ref{lem:Thm1_lemma} and step
(iii) follows from our assumption on the error tolerance $\epsilon
\leq 1$.  Now recall our assumption $\step \leq \rho_0
\left(\frac{\sqrt{\mu}}{32} + \phi_0 + \lambda_0 \right)^{-1}$, which
ensures that the RHS is further bounded by $\rho_0$. In effect, this
ensures that the ``size'' of the step $\eta \gradest(\x)$ is always
smaller that the radius $\rho_0$ within which our Lipschitz and
smoothness properties hold.

Since the function $\f$ is smooth with smoothness parameter
$\globalSmooth{0}$, we have
\begin{align*}\
\f(\x[t + 1]) - \f(\x[t]) & \leq \inprod{\grad \f(\x[t])}{\x[t + 1] -
  \x[t]} + \frac{\globalSmooth{0}}{2} \enorm{\x[t + 1] - \x[t]}^2 \\
& = - \inprod{ \step \grad \f(\x[t])}{ \gradest(\x[t])} +
\frac{\globalSmooth{0} \step^2}{2} \enorm{\gradest(\x[t])}^2 \\
& = - \inprod{ \step \grad \f(\x[t])}{ \gradest(\x[t]) - \grad
  \fr(\x[t])} - \inprod{\step \grad \f(\x[t])}{\grad \fr(\x[t])} +
\frac{\globalSmooth{0} \step^2}{2} \enorm{\gradest(\x[t])}^2 \\
& \leq \step \enorm{\grad \f(\x[t])} \enorm{ \gradest(\x[t]) - \grad
  \fr(\x[t])} - \inprod{\step \grad \f(\x[t])}{\grad \fr(\x[t])} +
\frac{\globalSmooth{0} \step^2}{2} \enorm{\gradest(\x[t])}^2 \\
& \stackrel{(i)}{\leq} \step \enorm{\grad \f(\x[t])} \enorm{
  \gradest(\x[t]) - \grad \fr(\x[t])} - \step \enorm{\grad
  \f(\x[t])}^2 + \step \globalSmooth{0} \smoothingradius \enorm{\grad
  \f(\x[t])} + \frac{\globalSmooth{0} \step^2}{2}
\enorm{\gradest(\x[t])}^2.
\end{align*}
Here step (i) follows from part (b) of Lemma~\ref{lem:Thm1_lemma}. Now
applying the AM-GM Inequality to the first term of the RHS, we find
that
\begin{align*}
\f(\x[t + 1]) - \f(\x[t]) & \leq \frac{\step}{2} \enorm{\grad
  \f(\x[t])}^2 + \frac{\step}{2} \enorm{ \gradest(\x[t]) - \grad
  \fr(\x[t])}^2 - \step \enorm{\grad \f(\x[t])}^2 + \step
\globalSmooth{0} \smoothingradius \enorm{\grad \f(\x[t])} +
\frac{\globalSmooth{0} \step^2}{2} \enorm{\gradest(\x[t])}^2 \\
& = - \frac{\step}{2} \enorm{\grad \f(\x[t])}^2 + \frac{\step}{2}
\enorm{ \gradest(\x[t]) - \grad \fr(\x[t])}^2 + \step \globalSmooth{0}
\smoothingradius \enorm{\grad \f(\x[t])} + \frac{\globalSmooth{0}
  \step^2}{2} \enorm{\gradest(\x[t])}^2
\end{align*}
We now turn our attention to bounding the last term on the RHS:
\begin{align*}
\frac{\globalSmooth{0} \step^2}{2} \enorm{\gradest(\x[t])}^2 &=
\frac{\phi_0 \step^2}{2} (\enorm{\gradest(\x[t]) - \grad \fr(\x[t]) +
  \grad \fr(\x[t])}^2) \\ & \leq \frac{\phi_0 \step^2}{2} (
2\enorm{\gradest(\x[t]) - \grad \fr(\x[t])}^2 + 2\enorm{\grad
  \fr(\x[t])}^2 ) \\ & = \phi_0 \step^2 \enorm{\gradest(\x[t]) - \grad
  \fr(\x[t])}^2 + \phi_0 \step^2 \enorm{\grad \fr(\x[t]) - \grad
  \f(\x[t]) + \grad \f(\x[t])}^2 \\ & \leq \phi_0 \step^2
\enorm{\gradest(\x[t]) - \grad \fr(\x[t])}^2 + 2 \phi_0 \step^2
(\enorm{\grad \fr(\x[t]) - \grad \f(\x[t])}^2 + \enorm{\grad
  \f(\x[t])}^2) \\ & \stackrel{(i)}{\leq} \phi_0 \step^2
\enorm{\gradest(\x[t]) - \grad \fr(\x[t])}^2 + 2 \phi_0 \step^2
(\phi_0^2 \smoothingradius^2 + \enorm{\grad \f(\x[t])}^2)
\end{align*}
where step (i) follows from part (b) of
Lemma~\ref{lem:Thm1_lemma}. Putting together the pieces, we have
\begin{align*}
\f(\x[t + 1]) - \f(\x[t]) & \leq \left(- \frac{\step}{2} + 2 \phi_0
\step^2 \right) \enorm{\grad \f(\x[t])}^2 + \left(\frac{\step}{2} +
\phi_0 \step^2 \right) \enorm{ \gradest(\x[t]) - \grad \fr(\x[t])}^2 +
\step \globalSmooth{0} \smoothingradius \enorm{\grad \f(\x[t])} + 2
\phi_0^3 \step^2 \smoothingradius^2.
\end{align*}

In addition, since the function is locally smooth at the point
$\x[t]$, we have
\begin{align*}
(\curvature - \curvature^2 \globalSmooth{0}/2) \enorm{ \grad \f(\x[t])
  }^2 & \leq \f(\x[t]) - \f(\x[t] - \curvature \grad \f(\x[t]) )
  \\ &\leq \f(\x[t]) - \f(\xstar),
\end{align*}
for some parameter $\curvature$ chosen small enough such that the
relation $\curvature \enorm{ \grad \f( \x[t]) } \leq \rho_{0}$ holds.
We may thus set $\curvature = \curvature_0 = \min \left\{ \frac{1}{2
  \globalSmooth{0} }, \frac{\rho_{0}}{\Lipcon_{0}} \right\}$ and
recall the notation $\diff{t} = \f(\x[t]) - \f(\xstar)$ to obtain
\begin{align*}
\CondExs{t} \bracket{ \diff{t+1} - \diff{t}} & \leq (- \frac{\step}{2}
+ 2 \phi_0 \step^2) \enorm{\grad \f(\x[t])}^2 + \step \globalSmooth{0}
\smoothingradius \frac{2}{\curvature_{0}} \diff{t}^{1/2} +
(\frac{\step}{2} + \phi_0 \step^2) \enorm{ \gradest(\x[t]) - \grad
  \fr(\x[t])}^2 + 2 \phi_0^3 \step^2 \smoothingradius^2 \\
& \stackrel{(iii)}{\leq} - \frac{\step \Plconst}{4} \diff{t} + 2
\frac{\step \globalSmooth{0} \smoothingradius }{\curvature_{0}}
\diff{t}^{1/2} + \step \enorm{ \gradest(\x[t]) - \grad \fr(\x[t])}^2 +
2 \phi_0^3 \step^2 \smoothingradius^2, \\
& \stackrel{(iv)}{\leq} - \frac{\step \Plconst}{4} \diff{t} +
\frac{\step \Plconst}{8} \diff{t} + 8 \frac{\step (\globalSmooth{0}
  \smoothingradius)^2 }{\Plconst \curvature_{0}^2} + \step \enorm{
  \gradest(\x[t]) - \grad \fr(\x[t])}^2 + 2 \phi_0^3 \step^2
\smoothingradius^2,
\end{align*}
where step (iii) follows from applying the PL inequality and using the
fact that $\step \leq \frac{1}{8 \globalSmooth{0}}$, and step (iv)
from the inequality $2ab \leq a^2 + b^2$ which holds for any pair of
scalars $(a, b)$.

Recall the assumed bounds on our parameters, namely
\begin{align*}
\step \leq \min \left\{ 1, \frac{1}{8 \globalSmooth{0}} \right\}, \quad
\text{ and} \quad \smoothingradius \leq \frac{1}{8 \globalSmooth{0}}
\min\left \{ \curvature_{0} \Plconst \sqrt{\frac{\epsilon}{240}},
\frac{1}{\globalSmooth{0}}\sqrt{\frac{\epsilon \Plconst}{30} }
\right\}.
\end{align*}
Using these bounds, we have
\begin{align*}
\diff{t+1} - \diff{t} & \leq - \frac{\step \Plconst}{8} \diff{t} +
\step \enorm{ \gradest(\x[t]) - \grad \fr(\x[t])}^2 + \step \Plconst
\frac{\epsilon}{120} + \step^2 \frac{\epsilon \Plconst}{30 \phi_0}
\\ & \leq - \frac{\step \Plconst}{8} \diff{t} + \step \enorm{
  \gradest(\x[t]) - \grad \fr(\x[t])}^2 + \step \Plconst
\frac{\epsilon}{60}.
\end{align*}
In conjunction with equation~\eqref{eqn:minibatching_bound}, we now
have the key inequality
\begin{equation}
\label{eqn:gd_recursion}
\diff{t+1} \leq \left(1 - \frac{\step \Plconst}{8}\right)
\diff{t} + \step \frac{\Plconst \epsilon}{16}.
\end{equation}

In order to complete the proof, we now demonstrate how to unroll this
recursion using strong induction. For each time step $i = 1, 2,
\ldots, T$, denote by $\mathcal{E}_i$ the event that $\diff{i} \leq 10
\diff{0}$ and $\diff{i} \leq \left(1 - \frac{\step \Plconst}{8}
\right) \diff{i-1} + \step \frac{\Plconst \epsilon}{16}$. We claim
that for each $t \in \mathbb{N}$, we have
\begin{align*}
\Pr \left\{ \cap_{ i =1}^t \Espace_i \right\} \geq 1 - \delta t.
\end{align*}
Let us establish this claim via induction.

\paragraph{Base case:} Applying Lemma~\ref{lem:minibatch_grad_est} and
equation~\eqref{eqn:gd_recursion}, we obtain with probability $1 -
\delta$ the inequality $\diff{1} \leq (1 - \frac{\step \Plconst}{8})
\diff{0} + \step \frac{\Plconst \epsilon}{16}$. Further, by our
assumption $\epsilon \leq \min \{1, \frac{1}{\mu} \}
\frac{\diff{0}}{10}$, we have $\diff{1} \leq 10 \diff{0}$, so we have
shown the base case that event $\mathcal{E}_1$ holds with probability
exceeding $1 - \delta$.

\paragraph{Induction step:} Fix an integer $t$, and assume, by the
induction hypothesis, that the event $\cap_{ i =1}^t \Espace_i$ holds
with probability exceeding $1 - \delta t$. Let us condition on this
event. In addition, applying Lemma~\ref{lem:minibatch_grad_est} and
equation~\eqref{eqn:gd_recursion} yields, with probability $1 -
\delta$, the inequality
\begin{align*}
\diff{t+1} & \leq \left(1 - \frac{\step \Plconst}{8} \right) \diff{t}
+ \step \frac{\Plconst \epsilon}{16} \\ &\leq \left(1 - \frac{\step
  \Plconst}{8} \right)^{t+1} \diff{0} + \sum_{i=1}^t (1 - \frac{\step
  \Plconst}{8})^i \step \frac{\Plconst \epsilon}{16} \\ &\leq \left(1
- \frac{\step \Plconst}{8} \right)^{t+1} \diff{0} +
\sum_{i=1}^{\infty} (1 - \frac{\step \Plconst}{8})^i \step
\frac{\Plconst \epsilon}{16} \\ &= \left(1 - \frac{\step \Plconst}{8}
\right)^{t+1} \diff{0} + \frac{\epsilon}{2}.
\end{align*}
Once again, by our assumption $\epsilon \leq \min \{1, \frac{1}{\mu}
\} \frac{\diff{0}}{10}$, we have $\diff{t + 1} \leq 10
\Delta_0$. Putting together the pieces with a union bound then implies
that the event $\cap_{ i =1}^{t + 1} \Espace_i$ holds with probability
exceeding $1 - \delta (t + 1)$, thereby establishing the induction
hypothesis.

\vspace{2mm}

Finally, at time $T$, we condition on the event $\cap_{ i
  =1}^{T} \Espace_i$, thereby obtaining the bound
\begin{align*}
\diff{T} \leq \left(1 - \frac{\step \Plconst}{8} \right)^{T} \diff{0}
+ \frac{\epsilon}{2}.
\end{align*}
We complete the proof by by substituting our choice of the tuple
$(\eta, T)$.

\bibliographystyle{alpha}
\bibliography{Bibliography_RL}

\end{document}